\documentclass[twoside,11pt]{article}

\usepackage{custom}
\usepackage{custom_global}
\usepackage{xparse,calc}
\usepackage{afterpage}

\newcommand{\pip}[1][]{\theta#1}
\newcommand{\fnbound}[1]{ {\cal F}^{\rm #1} }
\newcommand{\fnLexp}[1]{ L^{\rm{#1}} }
\newcommand{\fnL}[1]{ \hat  L^{\rm{#1}} }

\NewDocumentCommand{\ratioa}{ O{t} O{}  }{r_{#1}(\theta_{#2}) }
\NewDocumentCommand{\ratio}{ O{} O{} O{}  }{r_{#1}^{#3}({#2}) }

\newcommand{\advstd}[1][]{A^{\pi_{\rm old}}(s,a)}

\NewDocumentCommand{\adv}{ O{} O{}  }{ A_{#1}^{#2} }

\newcommand{\rbweight}[1][]{ \alpha_{#1} }

\newcommand{\FBTEXT}{{RB}}
\newcommand{\FB}{{\rm \FBTEXT}}
\def\rollback/{rollback}
\def\Rollback/{Rollback}

\def\pmethodfallback/{PPO-\FBTEXT}
\def\pmethodfallbackfull/{PPO with \Rollback/}

\def\pmethodklfull/{trust region-based PPO}
\def\Pmethodklfull/{Trust Region-based PPO}
\newcommand{\KLTEXT}{{TR}}
\newcommand{\KL}{{\rm\KLTEXT}}
\def\pmethodkl/{\KLTEXT-PPO}

\def\pmethodklsimple/{\pmethodkl/-direct}
\def\pmethodratiosimple/{PPO-direct}

\def\pmethodhybridfull/{{Truly PPO}} 
\def\Pmethodhybridfull/{{Truly PPO}} 
\def\pmethodhybrid/{Truly PPO} 
\def\pmethodhybridf{\rm truly}

\newcommand{\clip}{{\rm CLIP}}
\def\oldsub{_{\rm old}}
\def\newsub{_{\rm new}}
\def\newp{_{\rm new}}

\def\clippingmechanism/{clipping mechanism}
\def\clippingratio/{clipping range}

\def\Klbased/{Trust region-based}
\def\Ratiobased/{Ratio-based}
\def\klbased/{trust region-based}
\def\ratiobased/{ratio-based}
\def\stable/{stable}

\NewDocumentCommand{\pinew}{ O{}  }{ \pi\newsub^{\rm #1} }

\def\PG{\rm PG}
\def \LPG{ L^{\rm PG} }
\def \LM{ M } 

\label{User-Defined-End}

\usepackage[preprint]{jmlr2e}

\usepackage{cleveref}

\jmlrheading{20}{2020}{1-29}{1/20}{5/20}{}{Yuihui Wang, Hao He, Chao Wen and Xiaoyang Tan}

\ShortHeadings{Truly Proximal Policy Optimization}{Wang, He, Wen and Tan}
\firstpageno{1}

\label{User-Defined-Begin}

\begin{document}

\title{Truly Proximal Policy Optimization}

\author{\name Yuhui Wang \email y.wang@nuaa.edu.cn \\
       \name Hao He \email hugo@nuaa.edu.cn \\
       \name Chao Wen \email chaowen@nuaa.edu.cn \\
       \name Xiaoyang Tan\thanks{Corresponding author.}  \email x.tan@nuaa.edu.cn \\
       \addr College of Computer Science and Technology, Nanjing University of Aeronautics and Astronautics\\
       \addr MIIT Key Laboratory of Pattern Analysis and Machine Intelligence\\
       \addr Collaborative Innovation Center of Novel Software Technology and Industrialization\\
       Nanjing 210016, China.}

\editor{}

\maketitle

\begin{abstract}
Proximal policy optimization (PPO) is one of the most successful deep reinforcement-learning methods, achieving state-of-the-art performance across a wide range of challenging tasks. However, its optimization behavior is still far from being fully understood. In this paper, we show that PPO could neither strictly restrict the likelihood ratio as it attempts to do nor enforce a well-defined trust region constraint, which means that it may still suffer from the risk of performance instability. To address this issue, we present an enhanced PPO method, named \Pmethodhybridfull/. Two critical improvements are made in our method: 
1) it adopts a new clipping function to support a rollback behavior to restrict the difference between the new policy and the old one; 
2) the triggering condition for clipping is replaced with a trust region-based one, 
{such that optimizing the resulted surrogate objective function provides guaranteed monotonic improvement of the ultimate policy performance.}
It seems, by adhering more truly to making the algorithm proximal --- confining the policy within the trust region, the new algorithm improves the original PPO on both sample efficiency and performance.
\end{abstract}


\section{Introduction}

Deep model-free reinforcement learning has achieved great successes in recent years, notably in video games \citep{mnih2015human}, board games \citep{silver2017mastering}, robotics \citep{levine2016end, liu2018adaptive}, and challenging control tasks \citep{schulman2016high, Akshara2019Robot}. Policy gradient (PG) methods are useful model-free policy search algorithms, updating the policy with an estimator of the gradient of the expected return \citep{peters2008reinforcement, hu2019reinforcement}. {One major challenge of PG-based methods is to estimate the right step size for the policy updating, and an improper step size may result in severe policy degradation due to the fact that the input data strongly depends on the current policy \citep{kakade2002approximately, schulman2015trust}.} For this reason, the trade-off between learning stability and learning speed is an essential issue to be considered for a PG method.

The well-known trust region policy optimization (TRPO) method addressed this problem by imposing onto the objective function a trust region constraint so as to control the KL divergence between the old policy and the new one \citep{schulman2015trust}. This can be theoretically justified by showing that optimizing the policy within the trust region leads to guaranteed monotonic performance improvement. However, the complicated second-order optimization involved in TRPO makes it computationally inefficient and difficult to scale up for large scale problems when extending to complex network architectures. Proximal Policy Optimization (PPO) significantly reduces the complexity by adopting a clipping mechanism so as to avoid imposing the hard constraint completely, allowing it to use a first-order optimizer like the Gradient Descent method to optimize the objective \citep{schulman2017proximal}. As for the mechanism for dealing with the learning stability issue, in contrast with the trust region method of TRPO, PPO tries to remove the incentive for pushing the policy away from the old one when the likelihood ratio between them is out of a clipping range. PPO is proven to be very effective in dealing with a wide range of challenging tasks, while being simple to implement and tune.

However, despite its success, the actual optimization behavior of PPO is less studied, highlighting the need to study the “proximal” property of PPO. Some researchers have raised concerns about whether PPO could restrict the likelihood ratio as it attempts to do \citep{wang2019trust,ilyas2018deep}, and since there exists an obvious gap between the heuristic likelihood ratio constraint and the theoretically justified trust region constraint, it is natural to ask whether PPO enforces a trust region-like constraint as well to ensure its stability in learning?

In this paper, we formally address both the above questions and give negative answers to both of them. In particular, we found that PPO could neither strictly restrict the likelihood ratio nor enforce a trust region constraint. The former issue is mainly caused by the fact that PPO could not entirely
remove the incentive for pushing the policy away, while the latter is mainly due to the inherent difference between the two types of constraints adopted by PPO and TRPO respectively.

Inspired by the insights above, we propose an enhanced PPO method, named \Pmethodhybridfull/. In particular, we apply a negative incentive to prevent the policy from being pushed away during training, which we called a \emph{\rollback/} operation. Furthermore, we replace the triggering condition for clipping with a trust region-based one, 
such that optimizing the resulting surrogate objective function provides guaranteed monotonic improvement of the ultimate policy performance.
\pmethodhybrid/ actually combines the strengths of TRPO and PPO --- it is theoretically justified and is simple to implement with first-order optimization. Extensive results on several benchmark tasks show that the proposed methods significantly improve both the policy performance and the sample efficiency.
Source code is available at \url{https://github.com/wangyuhuix/TrulyPPO}.

\new{
A preliminary version of this work appears in \citep{wang2019truly}. This expanded version aims to provide more in-depth investigations of the characteristics of the proposed methods. 
Specifically, we propose a new objective function combining trust region-based clipping with \rollback/ operation on KL divergence. We show that it owns a better theoretical characteristic of the monotonic improvement and performs much better in practice. 
Moreover, the detailed proofs of the theorems and more derivation details are included to provide a more detailed description of the theoretical properties.
More experiments and comparison with the state-of-art methods have been added to show the effectiveness of the proposed methods. 
Last but not least, we give an elaborate summary of the relation with prior works on constraining policy and show how could our finds guide future research.
}

In what follows, we first introduce the preliminaries of proximal policy optimization in \Cref{sec_preliminaries}, then we give an analysis of the ``proximal'' property of PPO in \Cref{sec_analysis}. 
Next, we propose three variants of PPO to enhance its ability in restricting policy in \Cref{sec_method}.
We give a detailed relation with prior works in \Cref{sec_relatedwork}. 
The main experimental results are given in \Cref{sec_experiment}, and the paper is concluded in \Cref{sec_conclusion}.

\section{Preliminaries}\label{sec_preliminaries}

A Markov Decision Processes (MDP) is described by the tuple $(\mathcal{S},\mathcal{A},{\cal T},c,\rho_1,\gamma)$. $\mathcal{S}$ and $\mathcal{A}$ are the state space and action space; ${\cal T}: \mathcal{S} \times \mathcal{A} \times \mathcal{S} \rightarrow {\mathbb{R}}$ is the transition probability distribution; $c:\mathcal{S} \times \mathcal{A} \rightarrow \mathbb{R} $ is the reward function; $\rho_1$ is the distribution of the initial state $s_1$, and $\gamma \in (0,1)$ is the discount factor. 
The performance of a policy $\pi$ is defined as
$
	\eta (\pi ) = {\mathbb{E}_{s\sim{\rho ^\pi },a\sim\pi }}\left[ {c(s,a)} \right]
$
where ${\rho ^{{\pi  }}}(s) = (1-\gamma)\mathop \sum_{t = 1}^{\infty} {\gamma ^{t - 1}}{\rho_t^{{\pi }}}(s)$, $\rho_t^{\pi}$ is the density function of state at time $t$.

Policy gradients methods \citep{sutton1999policy} update the policy by the following surrogate performance objective,
\begin{equation}\label{eq_policy_gradient}
\begin{aligned}
\LPG_{\pi\oldsub}(\pi )={{\mathbb{E}_{s,a}}\left[ { \ratio[s,a][\pi][\pi\oldsub] \adv[s,a][\pi\oldsub] } \right] + \eta(\pi\oldsub)} 
\end{aligned}
\end{equation}
where $\ratio[s,a][\pi][\pi\oldsub]= {\pi (a|s)}/{\pi\oldsub(a|s)}$ is the \emph{likelihood ratio} between the new policy $\pi$ and the old policy $\pi\oldsub$,
$\adv[s,a][\pi\oldsub]\triangleq \mathbb{E}[R_t^\gamma|s_t=s,a_t=a; \pi\oldsub ]-\mathbb{E}[R_t^\gamma|s_t=s; \pi\oldsub ]$ is the advantage value function of the old policy $\pi\oldsub$.
\emph{For simplicity, we will omit writing superscript/subscript $\pi\oldsub$ explicitly, e.g., $\LPG(\pi), \ratio[s,a][\pi], A_{s,a}$.}

In practical deep RL algorithms, the policy are usually parametrized by Deep Neural Networks (DNNs).
For discrete action space tasks where $|{\cal A}|=D$, the policy is parametrized by $\pi_\theta(s_t)=f^p_\theta(s_t)$,
 where $f^p_\theta$ is the DNN outputting a vector which represents a $D$-dimensional discrete distribution.
For continuous action space tasks, it is standard to represent the policy by a Gaussian policy, i.e., $\pi_\theta(a|s_t)=\mathcal{N}(a|f^\mu_\theta(s_t),f^\Sigma_\theta(s_t))$ \citep{williams1992simple,mnih2016asynchronous}, where $f^\mu_\theta$ and $f^\Sigma_\theta$ are the DNNs which output the mean and the covariance matrix of the Gaussian distribution. 

\subsection{Trust Region Policy Optimization}
The well-known Trust Region Policy Optimization (TRPO) is derived the following performance bound:
\begin{theorem}\label{thm_lowerbound}
	Let
	
		$C=\mathop {\max }\limits_{s,a} \left| {{ \adv[s,a][] }} \right|
			 {4\gamma }  /{{{(1 - \gamma )}^2}}
		$,
		$D_{{\rm{KL}}}^{{s}}\left( {\pi\oldsub,\pi } \right) \triangleq D_{{\rm{KL}}}^{}\left( {\pi\oldsub(\cdot|s)||\pi (\cdot|s)} \right)$,
$
	{\LM}(\pi ) = {\LPG}(\pi ) - C\max_{s \in {\cal S}}D_{\rm{KL}}^{\rm{s}}\left( {\pi\oldsub,\pi } \right).
$
	{We have}
	\begin{align}
	\eta(\pi)\geq{\LM}(\pi ),
	\eta(\pi\oldsub)={\LM}(\pi\oldsub ).
	\label{eq_lower_bound}
	\end{align}
\end{theorem}
This theorem implies that maximizing ${\LM}(\pi )$ guarantee non-decreasing of the performance of the new policy $\pi$.
TRPO imposed a constraint on the KL divergence:
\begin{subequations}\label{eq_TRPO}
\begin{align}
\mathop {\max }\limits_\pi  & \LPG(\pi)
\\
\text{s.t.}& \max_{s \in {\cal S}}D_{\rm{KL}}^{\rm{s}}\left( {\pi\oldsub,\pi } \right) \le \delta  \label{eq_trust_region}
\end{align}
\end{subequations}
{Constraint \eqref{eq_trust_region}} is called the \emph{trust region-based constraint}, which is a constraint on the KL divergence between the old policy and the new one.
%

\subsection{Proximal Policy Optimization}\label{sec_PPO}

Proximal policy optimization (PPO) attempts to restrict the policy by a clipping function \citep{schulman2017proximal}.\footnote{
There are two variants of PPO:
we refer to the one with clipping function as \emph{PPO}, and refer to the one with adaptive KL penalty coefficient as \emph{PPO-penalty} \citep{schulman2017proximal}.
}
\begin{equation}\label{eq_L_PPO_expectation}
\begin{aligned}
&\fnLexp{\clip}(\pi)
=
\mathbb{E}\left[
	\min \left( \ratio[s,a][\pi] \adv[s,a][],
	{\fnbound{\clip}}\left( {\ratio[s,a][\pi],\epsilon} \right) \adv[s,a][] \right) 
\right]
\end{aligned}
\end{equation}
where $\fnbound{\clip}$ is defined as
\begin{equation}\label{eq_clip}
\begin{aligned}
	& \fnbound{\clip}( {\ratio[s,a][\pi],\epsilon} )  
	=
	\begin{cases}
		1-\epsilon & \ratio[s,a][\pi] \leq 1-\epsilon \\
		1+\epsilon & \ratio[s,a][\pi] \geq 1+\epsilon \\
		\ratio[s,a][\pi] & \text{otherwise}
	\end{cases}
\end{aligned}
\end{equation}
where $(1-\epsilon, 1+\epsilon)$ is called the \emph{\clippingratio/}, $0<\epsilon<1$ is the parameter.

Given $s_t\sim{\rho ^{\pi_{\theta\oldsub}}},a_t\sim{\pi_{\theta\oldsub}}(\cdot|s_t)$, which are sampled using the parametrized policy $\pi_{\theta\oldsub}$.
For simplicity, 
we will use subscript $t$ to denote the corresponding value for sample $(s_t,a_t)$, e.g., $\ratio[t][\pi_\pip] \triangleq \ratio[s_t,a_t][\pi_\pip]$, $A_{t} \triangleq \adv[s_t,a_t]$;
and we will also use functions of $\theta$ and the ones of $\pi$ alternatively, e.g., $\ratio[t][\pip] \triangleq \ratio[t][\pi_\pip] $, $\fnLexp{CLIP}(\theta)\triangleq \fnLexp{CLIP}(\pi_\theta)$ and $D_{\rm KL}^{s}(\theta\oldsub, \theta) \triangleq D_{\rm KL}^{s}(\pi_{\theta\oldsub}, \pi_{\theta})$.
The overall empirical objective function of data $\{s_t,a_t,A_t\}_{t=1}^T$ is $\fnL{\clip}(\theta)$. 
To provide a more intuitive form on how the clipping function works, the objective function for a single sample $(s_t,a_t)$ can be rewritten in the {following form}:
\begin{subequations}\label{eq_L_PPO_case}

 \if@twocolumn
 		\def\lenjdiucdsawd{0.55}
 \else
 	 	\def\lenjdiucdsawd{0.3}	 	
 \fi
\begin{align}[left=	{\fnLexp{\clip}_{t}(\theta)=\empheqlbrace}]
	  	&{(1-\epsilon)}\adv[t] &  
	  	  	\parbox[c]{\lenjdiucdsawd\linewidth}{
	  				  	$\ratioa \leq 1-\epsilon$ \text{and} $ \adv[t]<0 $
	  		} \label{eq_PPO_case_1}
	  	\\
	  	&{(1+\epsilon)}\adv[t] &   
	  	  	\parbox[c]{\lenjdiucdsawd\linewidth}{
	  				  	$\ratioa \geq 1+\epsilon$ \text{and} $ \adv[t]>0 $
	  		} \label{eq_PPO_case_2}
	  	\\
	  	&\ratioa\adv[t] &  \text{otherwise \quad\quad\quad\quad} \notag
\end{align}
\end{subequations}
The case \eqref{eq_PPO_case_1} and \eqref{eq_PPO_case_2} are called the \emph{clipping condition}.
As the equation implies, once $\ratioa$ is out of the clipping range (with a certain condition of $\adv[t]$), the gradient of $\fnLexp{\clip}_t(\theta)$ w.r.t. $\theta$ will be zero.
As a result, $\ratio[t][\theta]$ {could} stop moving outward and the policy could be restricted.


\begin{figure}[!t]
	 \if@twocolumn
 		\def\lenjdiucchdhjs{0.5}
	 \else
 	 	\def\lenjdiucchdhjs{0.4}	 	
	 \fi
	\centering
	\centerline{
	\begin{subfigure}{\lenjdiucchdhjs\linewidth}
	{\includegraphics[width=1.0\linewidth]{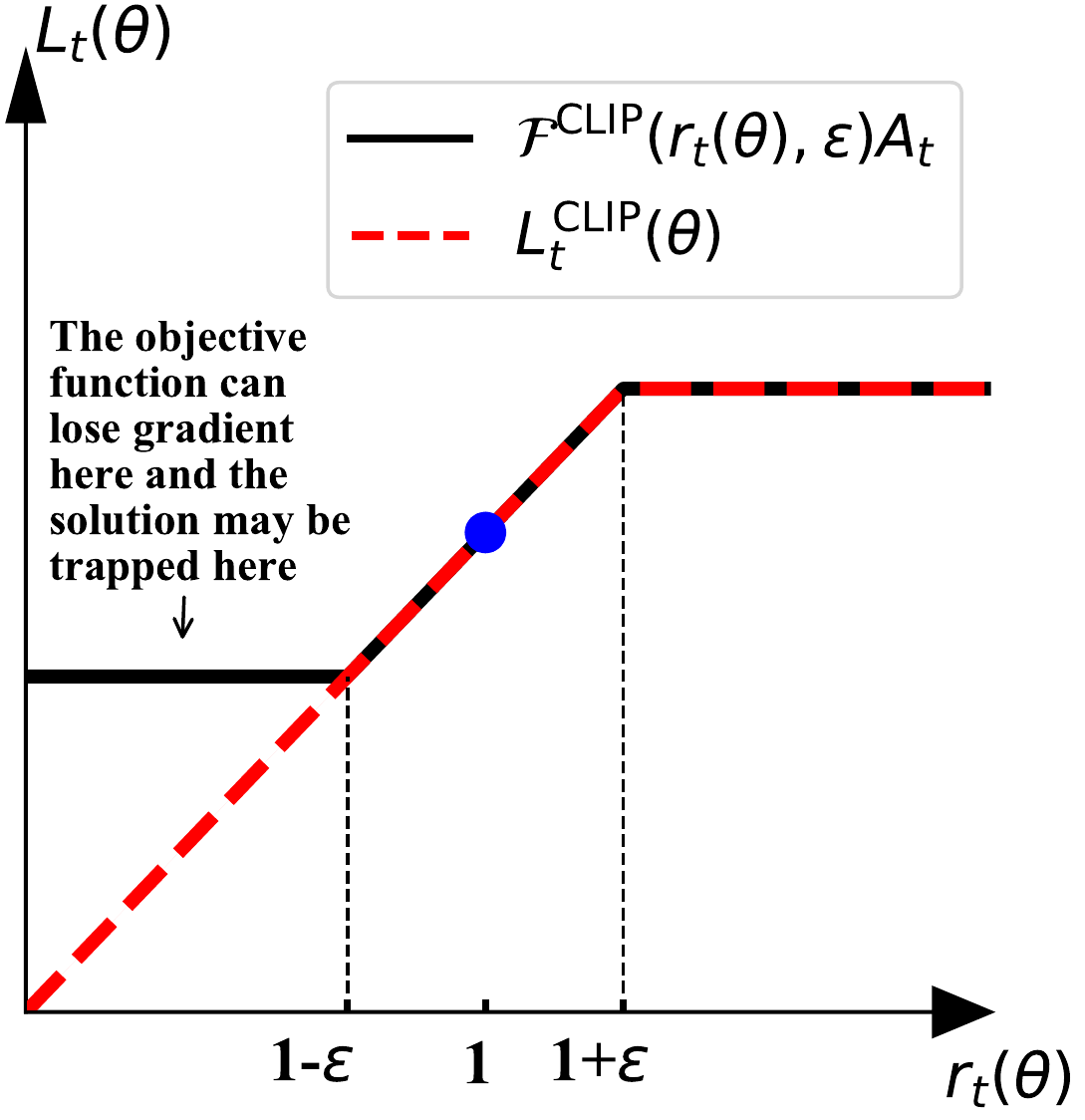} }
	\setlength{\abovecaptionskip}{0mm}
	\caption{$A_t>0$}
	\end{subfigure}
	\begin{subfigure}{\lenjdiucchdhjs\linewidth}
	{\includegraphics[width=1.0\linewidth]{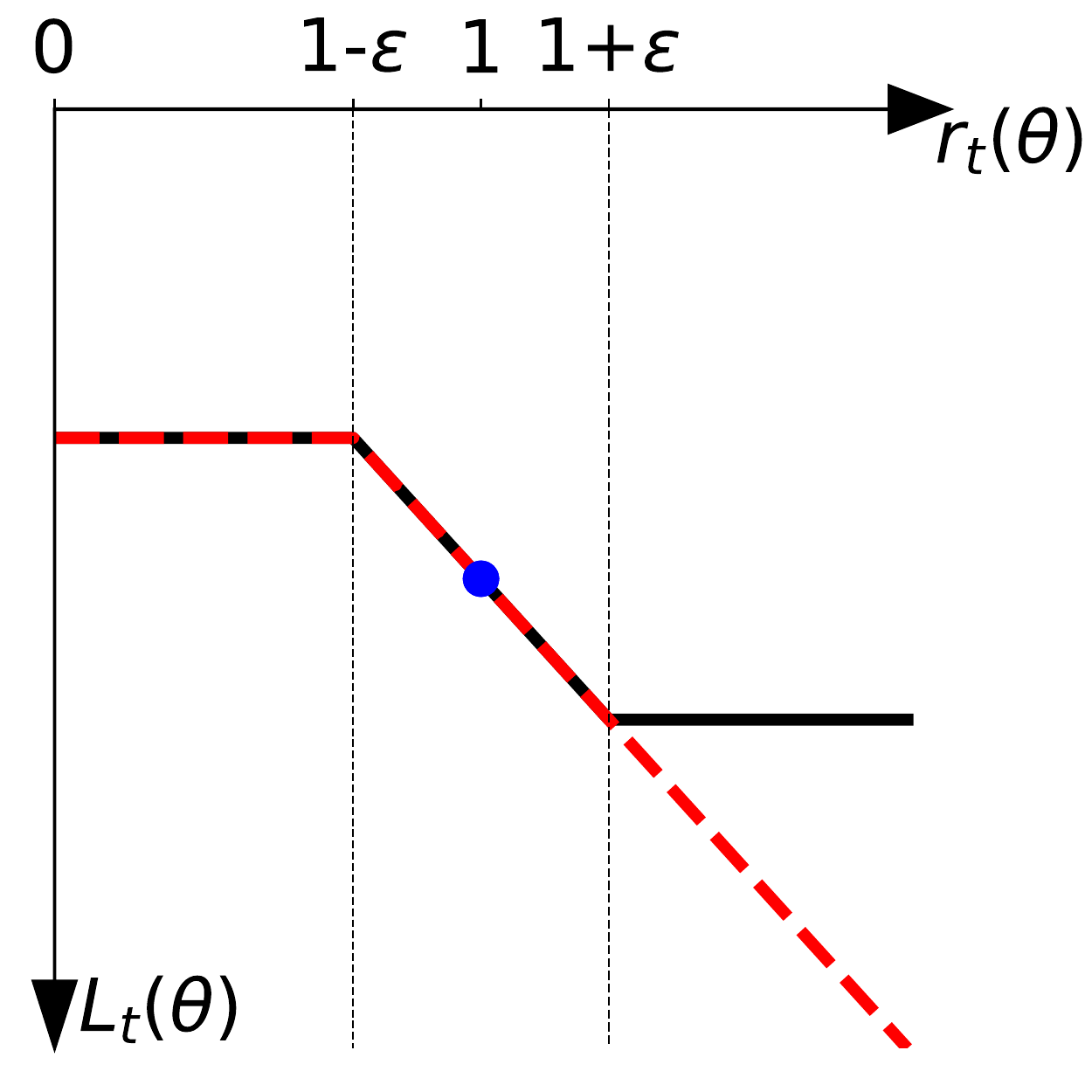} }
	\setlength{\abovecaptionskip}{0mm}
	\caption{$A_t<0$}
	\end{subfigure}
	}
	\iffastcompile
		\caption{
		}\label{fig_min_operation}
	\else
		\caption{
			Plots showing the surrogate function as functions of the likelihood ratio $\ratio[t][\theta]$ for positive advantages (left) and negative advantages (right).
			For the one without the minimum operation (black solid line), i.e., $\fnbound{\clip}( {\ratioa,\epsilon} ) A_t$, it can lose gradient for improving $\ratio[t][\theta]$ once $|\ratio[t][\theta]-1|>\epsilon$, even the $\ratio[t][\theta]$ is not improved ($\ratio[t][\theta]<1-\epsilon$ when $A_t<0$ or $\ratio[t][\theta]>1+\epsilon$ when $A_t<0$).
			While the one with the minimum operation (red dashed line), i.e., $\fnLexp{CLIP}_{t}(\theta )$, does not suffer from this issue.
		}\label{fig_min_operation}
	\fi

\end{figure}

{The minimum between the clipped and unclipped objective in \cref{eq_L_PPO_expectation} is designed to make the final objective $\fnLexp{CLIP}(\theta)$ to be a lower bound on the unclipped objective \citep{schulman2017proximal}.
It should also be noted that such operation is important for optimization, {which is not referred to by prior literature}.
As implied in \cref{eq_clip}, the clipped objective without minimum operation, i.e., $\fnbound{CLIP}( { \ratioa,\epsilon} ) \adv[t]$, would lose gradient for improving the ratio $\ratio[t][\theta]$ once the ratio goes past the clipping range, even the objective value is worse than the initial one, i.e., $\ratioa \adv[t] < \ratioa[t][old] \adv[t] $.
The minimum operation actually provides a remedy for this issue.
To see this, \cref{eq_L_PPO_case} is rewritten as
\begin{subequations}\label{eq_L_PPO_case_general}
\if@twocolumn
	\small
	\def\lenjdiucdiuwyrchdhjs{0.32}
\else
 	\def\lenjdiucdiuwyrchdhjs{0.5}	 	
\fi

\begin{align}[left=	{\fnLexp{CLIP}_{t}(\theta )=\empheqlbrace}]
	& \text{const} &  
  	\parbox[c]{\lenjdiucdiuwyrchdhjs\linewidth}{
				$|\ratioa-1|\geq \epsilon$
			  	\text{ and } 
  				$ \ratioa \adv[t] \geq \ratioa[t][\oldsub] \adv[t]$  
	} \label{eq_L_PPO_case_general_condition}
  	\\
  	& \ratioa\adv[t] &  \text{otherwise \quad\quad\quad\quad} \notag
\end{align}
\end{subequations}
Condition \eqref{eq_PPO_case_1} combined with \eqref{eq_PPO_case_2} is equivalent to condition \eqref{eq_L_PPO_case_general_condition}.
As can be seen, the ratio $\ratio[t][\theta]$ is clipped only if the objective value is improved $ \ratioa \adv[t] \geq \ratioa[t][\oldsub] \adv[t]$ .
\Cref{fig_min_operation} depicts the mechanism of the minimum operation.
We also experimented with the direct-clipping method, i.e., $\fnbound{\clip}( {\ratioa,\epsilon} ) A_t$, and found it performs extremely bad in practice. 
See \Cref{sec_experiment_min} for more detail.
}

\section{{Analysis of the ``Proximal'' Property of PPO}}\label{sec_analysis}

PPO attempts to restrict the policy by clipping the likelihood ratio between the new policy and the old one.
{Recently, researchers have raised concerns about whether this \clippingmechanism/ can really restrict the policy} \citep{wang2019trust,ilyas2018deep}.
We investigate the following questions of PPO.
The first one is that whether PPO could bound the likelihood ratio as it attempts to do.
The second one is that whether PPO could enforce a well-defined trust region constraint, which is primarily concerned since that it is a theoretical indicator on the performance guarantee (see \eqrefnop[eq_lower_bound]) \citep{schulman2015trust}.
We give an elaborate analysis of PPO to answer these two questions.

\begin{question}\label{que_ratio}
	Could PPO bound the likelihood ratio within the clipping range as {it attempts to do}?
\end{question}


In general, PPO could generate an effect of preventing the likelihood ratio from exceeding the clipping range too much, but it could not strictly bound the likelihood ratio.

As we have discussed in Sec \ref{sec_PPO}, the gradient of $\fnLexp{\clip}_{t}(\theta)$ w.r.t. $\theta$ will be zero only if $\ratioa$ is out of the clipping range.
{As a result, the incentive, deriving from $\fnLexp{\clip}_{t}(\theta)$, for driving $\ratioa$ to go farther beyond the clipping range is removed.}%

However, in practice the likelihood ratios are known to be not bounded within the clipping range \citep{ilyas2018deep}. 
The likelihood ratios are larger than 4. on almost all the tasks, which are much larger than the upper clipping range 1.2 ($\epsilon=0.2$, see our empirical results in \Cref{sec_experiment}).
One main factor for this problem is that
{the clipping mechanism could not entirely remove incentive deriving from the overall objective $\fnL{\clip}(\theta)$, which possibly push these out-of-the-range $\ratioa$ to go farther beyond the clipping range.}
We formally describe this claim as follows.

\begin{theorem}
{Given $\theta_0$ that $r_t(\theta_0)$} satisfies the clipping condition (either condition \ref{eq_PPO_case_1} or \ref{eq_PPO_case_2}).
Let $\nabla \fnL{\clip}(\theta_0 )$ denote the gradient of $\fnL{\clip}$ at $\theta_0$, %
and similarly $\nabla r_t(\theta_0)$.
Let $\theta_1=\theta_0+\beta \nabla\fnL{\clip}(\theta_0 )$, where $\beta$ is the step size.
If 
\begin{equation}
\langle \nabla \fnL{\clip}(\theta_0 ), \nabla r_t(\theta_0)  \rangle A_t  > 0 \label{eq_PPO_outward_condition}
\end{equation}
then there exists some $\bar \beta>0$ such that for any $\beta \in (0, \bar \beta)$, we have 
\begin{equation}
 \left|r_t({\theta_1}) -1 \right| > \left|r_t({\theta_0}) -1 \right| > \epsilon .
 \label{eq_PPO_outward}
\end{equation}
\end{theorem}
\begin{proof}
Consider $\phi(\beta)=r_t( \theta_0+\beta \nabla\fnL{\clip}(\theta_0 ) )$.

By chain rule, we have 
\[
	\phi'(0)= \langle \nabla \fnL{\clip}(\theta_0 ), \nabla r_t(\theta_0)  \rangle 
\]

For the case where $r_t(\theta_0) \geq 1+\epsilon$ and $ \adv[t]>0 $, we have $\phi'(0)>0$.
Hence, there exists $\bar \beta >0 $ such that for any $\beta \in (0, \bar \beta)$
\[
	\phi(\beta) > \phi(0)
\]
Thus, we have 
\[
	r_t(\theta_1) > r_t(\theta_0) \geq 1+\epsilon
\]
We obtain
\[
	|r_t(\theta_1) -1 | > |r_t(\theta_0) -1|
\]

Similarly, for the case where $r_t(\theta_0) \leq 1-\epsilon$ and $ \adv[t]<0 $, we also have $|r_t(\theta_1) -1 | > |r_t(\theta_0) -1|$.
\end{proof}

{As this theorem implies, even the likelihood ratio $r_t(\theta_0)$ is already out of the clipping range, it could be driven to go farther beyond the range (see \eqrefnop[eq_PPO_outward]).}
The condition \eqref{eq_PPO_outward_condition} requires the gradient of the overall objective $ \fnL{\clip}(\theta_0 )$ to be similar in direction to that of $r_t(\theta_0) A_t$.
This condition possibly happens due to the similar gradients of different samples or optimization tricks.
In addition, the Momentum optimization methods preserve the gradients attained before, which could possibly make this situation happen.
Such condition occurs quite often in practice.
We made statistics over 1 million samples on benchmark tasks in \Cref{sec_experiment}, and the condition occurs at a percentage from 25\% to 45\% across different tasks.


\begin{question}\label{que_KL}
	Could PPO enforce a trust region constraint?
\end{question}
PPO does not explicitly attempt to impose a trust region constraint, i.e., the KL divergence between the old policy and the new one.
Nevertheless, our previous work revealed that a different scale of the clipping range can affect the scale of the KL divergence \citep{wang2019trust}.
Under state-action $(s_t,a_t)$, if the likelihood ratio $\ratioa$ is not bounded, then \emph{neither} could the corresponding KL divergence $D_{\rm KL}^{s_t}(\pip{\oldsub},\pip)$ be bounded. 
Thus, together with the previous conclusion in Question \ref{que_ratio}, we can know that PPO could not bound KL divergence.
{In fact, even the likelihood ratio $\ratioa$ is bounded, the corresponding KL divergence $D_{\rm KL}^{s_t}(\pip{\oldsub}, \pip)$ is not necessarily bounded. Formally, we have the following theorem.
}

\begin{theorem}
{Assume} that 
for discrete action space tasks where $|{\cal A}|\geq 3$, the policy is parametrized by $\pi_\theta(s_t)=p_t \in {\mathbb R}^{+|{\cal A}|} $, where $ \sum_d ^{|{\cal A}|} p^{(d)}_t=1$;
for continuous action space tasks, the policy is parametrized by $\pi_\theta(a|s_t)=\mathcal{N}(a|\mu_t,\Sigma_t)$.
Let $\Theta=\{ \theta| 1-\epsilon \leq \ratioa \leq 1+\epsilon \}$.
We have $\max_{\theta \in \Theta} D_{\rm KL}^{s_t}(\pip{\oldsub}, \pip)=+\infty$ for both discrete and continuous action space tasks.
\end{theorem}

\begin{proof}
The problem $\max_{\theta \in \Theta} D_{\rm KL}^{s_t}(\pip{\oldsub}, \pip)$ is formalized as 
\begin{equation}
	\begin{aligned}
		\max_\theta & D_{\rm KL}^{s_t}(\pip{\oldsub}, \pip)
		\\
		s.t. &  1-\epsilon \leq  r_t(\theta) \leq 1+\epsilon
	\end{aligned}
\end{equation}
We first prove the discrete action space case, where the problem can be transformed into the following form,
\begin{equation}\label{eq_max_kl_origin}
	\begin{aligned}
		\max_p & \sum_d p\oldsub^{(d)} \log \frac{  p\oldsub^{(d)} }{ p^{(d)} }
		\\
		s.t. &  1-\epsilon \leq \frac{ p^{(a_t)} }{  p\oldsub^{(a_t)} }\leq 1+\epsilon
		\\
		& \sum_d { p^{(d)} } =1
	\end{aligned}
\end{equation}
where $p\oldsub = f^p_{\theta\oldsub}(s_t)$.
We could construct a $p_{\rm new}$ satisfies 1) $p_{\rm new}^{(d')}=0$ for a $d' \neq a_t$ where $p_{\rm old}^{(d')}>0$; 2) $1-\epsilon\leq\frac{ p^{(a_t)}_{\rm new} }{  p\oldsub^{(a_t)} } \leq 1+\epsilon$.
Thus we have \[\sum_d p\oldsub^{(d)} \log \frac{  p\oldsub^{(d)} }{ p_{\rm new}^{(d)} }=+\infty\]
Then we prove the continuous action space case where $dim({\cal A})=1$. The problem \eqref{eq_max_kl_origin} can be transformed into the following form, 
\begin{equation}
	\begin{aligned}
		\max_{\mu, \sigma} & F(\mu,\sigma)= 
		  \frac{1}{2} {\left[ { - 2\log \frac{ \sigma }{\sigma{\oldsub}} + \frac{ \sigma }{\sigma{\oldsub}} + { (\mu-\mu\oldsub)^2}\sigma\oldsub^{-1} - 1} \right]}
		\\
		s.t. &  1-\epsilon \leq \frac{{\cal N}(a_t| \mu, \sigma)}{{\cal N}(a_t| \mu\oldsub, \sigma\oldsub)} \leq 1+\epsilon
	\end{aligned}
\end{equation}
where $\mu\oldsub=f^\mu_{\theta\oldsub}(s_t)$, $\sigma\oldsub=f^\Sigma_{\theta\oldsub}(s_t)$,
\[
	{{\cal N}(a| \mu, \sigma)} = \frac{1}{\sqrt{2\pi}\sigma}\exp{ \left( -\frac{(\mu-a)^2}{2\sigma^2} \right) }
\]
As can be seen, $\lim_{\sigma \rightarrow 0} F(\mu,\sigma) = +\infty $, we just need to prove that given any $\sigma_{\rm new}<\sigma\oldsub$, there exists $\mu_{\rm new}$ such that 
\[{{\cal N}(a_t| \mu_{\rm new}, \sigma_{\rm new})}={{\cal N}(a_t| \mu\oldsub, \sigma\oldsub)}\]
In fact, if $\sigma_{\rm new}<\sigma\oldsub$, then $\max_a {{\cal N}(a| \mu_{\rm new}, \sigma_{\rm new})} > \max_a {{\cal N}(a| \mu\oldsub, \sigma\oldsub)} $ for any $\mu_{\rm new}$.
Thus given any $\sigma_{\rm new}<\sigma\oldsub$, there always exists $\mu_{\rm new}$ such that ${{\cal N}(a_t| \mu_{\rm new}, \sigma_{\rm new})} =  {{\cal N}(a_t| \mu\oldsub, \sigma\oldsub)}$.

Similarly, for the case where $dim({\cal A})>1$, we also have $\max_{\theta \in \Theta} D_{\rm KL}^{s_t}(\pip{\oldsub}, \pip)=+\infty$.
\end{proof}

\begin{figure*}[!b]

	 \if@twocolumn
 		\def\lenjdiuceueurchdhjs{0.35}
	 \else
 	 	\def\lenjdiuceueurchdhjs{0.4}	 	
	 \fi
	\centerline{
	\begin{subfigure}{\lenjdiuceueurchdhjs\linewidth}
		\includegraphics[width=1.0\linewidth]{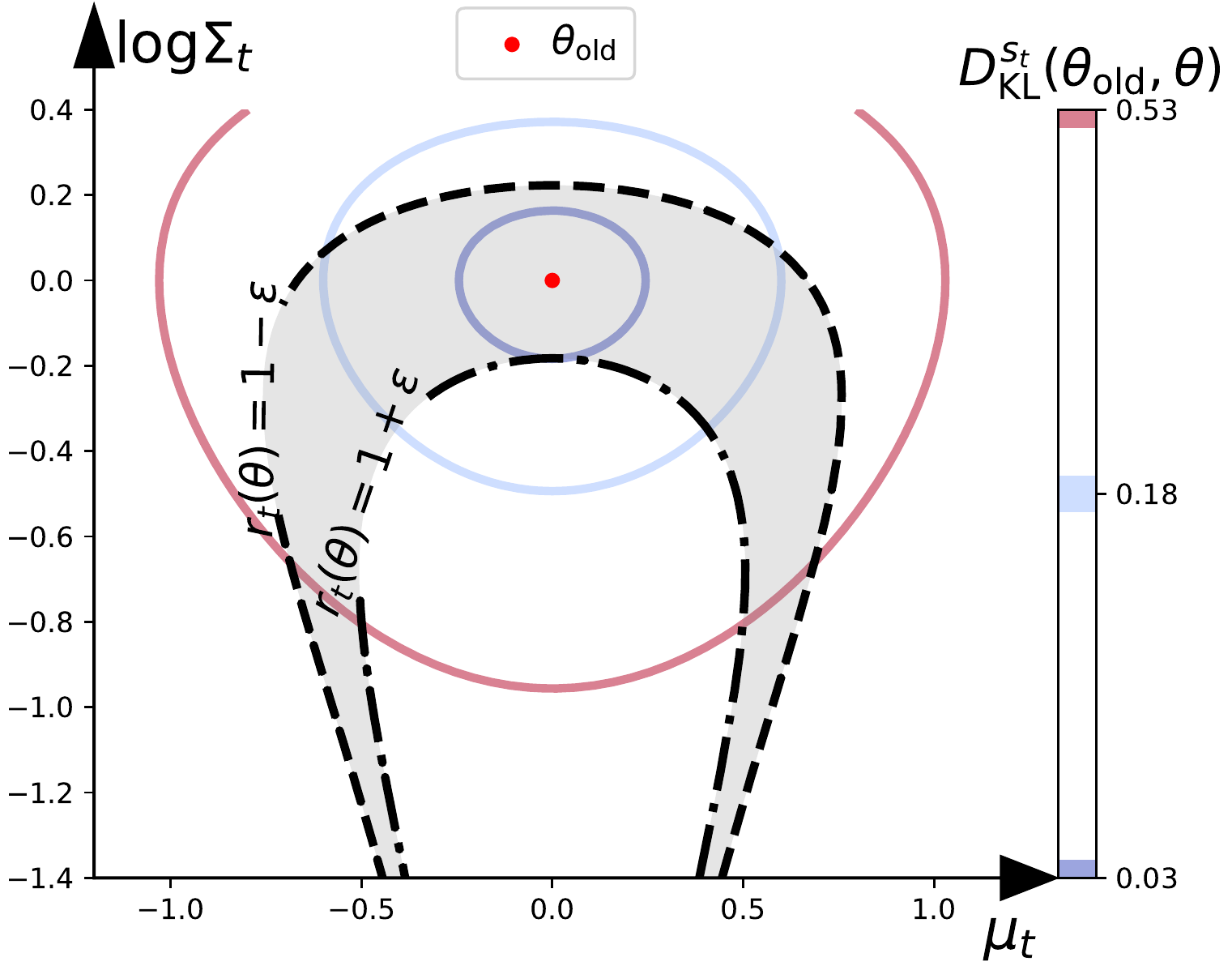}
		\caption{\small Case in Continuous Action Space}\label{fig_levelset_gaussian}
	\end{subfigure}
	\hspace*{0.5in}
	\begin{subfigure}{\lenjdiuceueurchdhjs\linewidth}
		\includegraphics[width=1.0\linewidth]{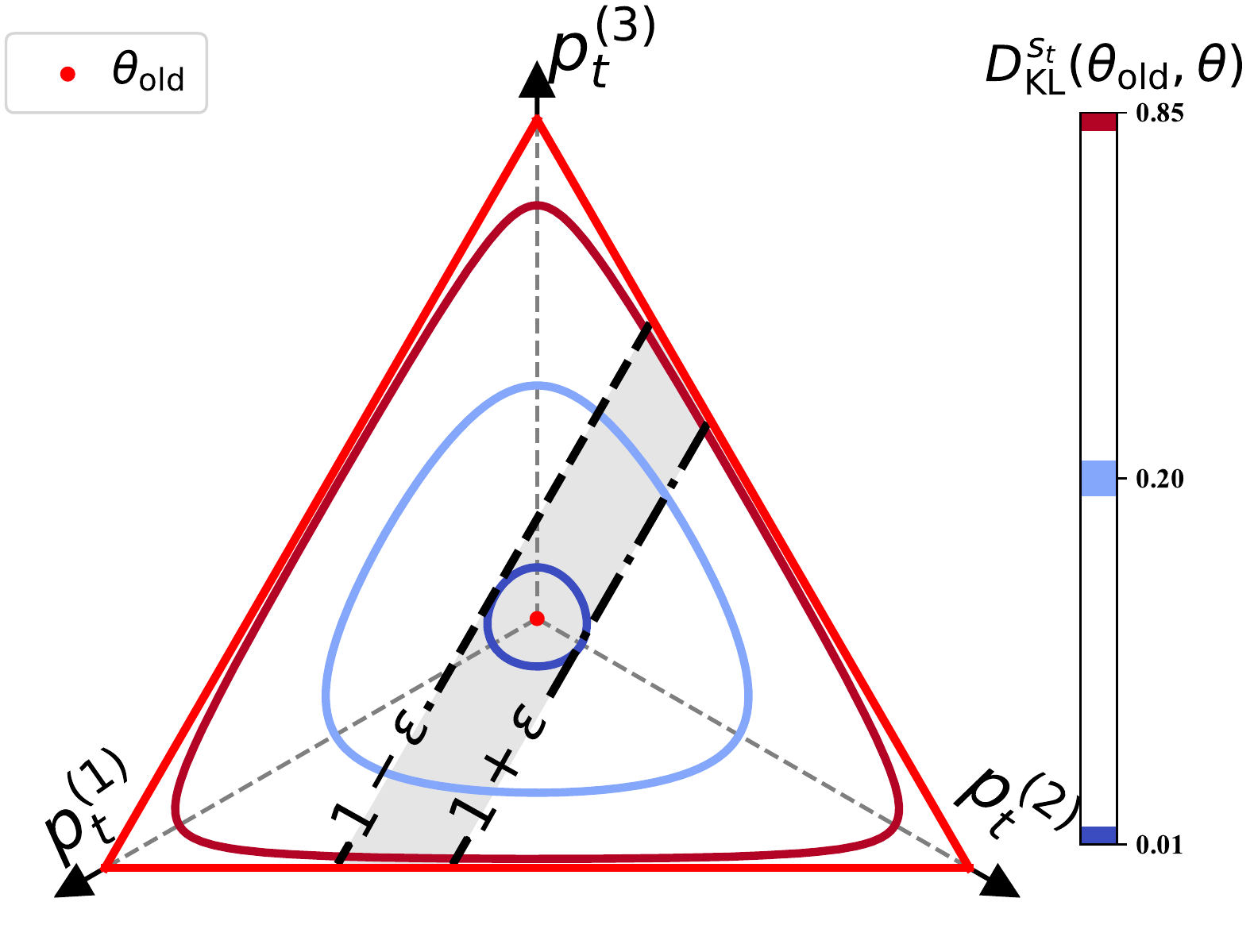}
		\caption{\small Case in Discrete Action Space}\label{fig_levelset_discrete}
	\end{subfigure}
	}
	\iffastcompile
		\caption{
		}\label{fig_levelset}
	\else
		\caption{
		The {grey} area shows the sublevel sets of $\ratioa$, i.e., $\Theta=\{ \theta| 1-\epsilon \leq \ratioa \leq 1+\epsilon \}$.
		The solid lines are the level sets of the KL divergence, i.e., $\{\theta|D_{\rm KL}^{s_t}(\pip{\oldsub},\pip)=\delta\}$.
		(a) A case of continuous action space, where $dim({\cal A})=1$.
		The action distribution under state $s_t$ is $\pi_\theta(s_t)=\mathcal{N}(\mu_t,\Sigma_t)$.
		(b) A case of discrete action space, where $|{\cal A}|=3$.
		The action distribution under state $s_t$ is $\pi_\theta(s_t)=(p_t^{(1)},p_t^{(2)},p_t^{(3)})$.
		{Note that the level sets are plotted on the hyperplane $\sum_{d=1}^3{p_t^{(d)}}=1$ and the figure is showed from the view of elevation$=45^\circ$ and azimuth$=45^\circ$. }
		}\label{fig_levelset}
	\fi
\end{figure*}

To attain an intuition on how this theorem holds, we plot the sublevel sets of $\ratioa$ and the level sets of $D_{\rm KL}^{s_t}(\pip{\oldsub},\pip)$ for the continuous and discrete action space tasks respectively.
As \Cref{fig_levelset} illustrates, the KL divergences (solid lines) within the sublevel sets of likelihood ratio (grey area) could go to infinity.

It can be concluded that there is an obvious gap between bounding the likelihood ratio and bounding the KL divergence. 
Approaches which manage to bound the likelihood ratio could not necessarily bound KL divergence theoretically.

\section{Method}\label{sec_method}

In the previous section, we have shown that PPO could neither strictly restrict the likelihood ratio nor enforce a trust-region constraint.
We address these problems
in the scheme of PPO with a general form for sample $(s,a)$,
\begin{equation}\label{eq_L_general}
\begin{aligned}
 \fnLexp{}_{s,a}(\pi ) 
= 
	\min \left( \ratio[s,a][\pi] \adv[s,a],
		{\fnbound{}}\left( {\ratio[s,a][\pi],\cdot} \right)\adv[s,a] \right) 
\end{aligned}
\end{equation}
where ${\fnbound{}}$ is a clipping function which attempts to restrict the policy $\pi$, ``$\cdot$'' in ${\fnbound{}}$ means any hyperparameters of it.
For example, in PPO, ${\fnbound{}}$ is a \ratiobased/ clipping function $\fnbound{\clip}(\ratio[s,a][\pi],\epsilon)$ (see \cref{eq_clip}). 
We modify this function to promote the ability in bounding the likelihood ratio and the KL divergence.
We now detail how to achieve this goal in the following sections. 
For simplicity, we will use functions of $\theta$ and use subscript $t$ to denote the function for sample $(s_t,a_t)$, 
e.g., $\ratio[t][\theta] \triangleq \ratio[s_t,a_t][\pi_\theta] $ and $\fnLexp{}_{t}(\theta) \triangleq \fnLexp{}_{s_t,a_t}(\pi_\theta) $.


\subsection{ \pmethodfallbackfull/ (\pmethodfallback/) }\label{sec_ppo_rb}

As discussed in Question \ref{que_ratio}, PPO can not strictly confine the likelihood ratio within the clipping range:
\emph{the likelihood ratio $\ratio[t][\theta]$ could be driving to go farther beyond the clipping range $(1-\epsilon, 1+\epsilon)$, as the incentive for moving $\ratio[t][\theta]$ could derive from the overall objective $\fnL{CLIP}(\theta)$, which can not be removed by the clipping function.}
We address this issue by introducing a \emph{\rollback/ operation} once the likelihood ratio exceeds, which is defined as
\begin{equation}
\begin{aligned}
	& \fnbound{\FB}( {\ratio[s,a][\pi],\epsilon, \rbweight} )  
	=&
	\begin{cases}
		{-\rbweight \ratio[s,a][\pi]} {+(1+\rbweight)(1-\epsilon)}
		& \ratio[s,a][\pi] \leq 1-\epsilon \\
		{-\rbweight \ratio[s,a][\pi]} {+(1+\rbweight)(1+\epsilon)}
		& \ratio[s,a][\pi] \geq 1+\epsilon \\
		\ratio[s,a][\pi] & \text{otherwise}
	\end{cases}
\end{aligned}
\end{equation}
where $\rbweight>0$ is a hyperparameter to decide the force of the \rollback/. 
\Cref{fig_fallback} plots $\fnLexp{\FB}_{s,a}(\pi)$ and $\fnLexp{\clip}_{s,a}(\pi)$ as functions of the likelihood ratio $\ratio[s,a][\pi]$. As the figure depicted, when $\ratio[s,a][\pi]$ is over the clipping range, the slope of $\fnLexp{\FB}_{s,a}(\pi)$ is reversed, while that of $\fnLexp{\clip}_{s,a}(\pi)$ is zero.

We now show how the \rollback/ operation can improve the ability in confining the likelihood ratio.
Let $\fnLexp{\FB}_t(\pip )$ denote the corresponding objective function for sample $(s_t,a_t)$; and let $\fnL{\FB}(\pip )$ denote the overall empirical objective.
The \rollback/ function $\fnbound{\FB}\left( {\ratio[t][\theta],\epsilon,\rbweight} \right)$ generates a negative incentive when $r_t(\theta)$ is outside of the clipping range. 
Thus it could somewhat neutralize the incentive deriving from the overall objective $\fnL{\FB}(\theta)$.
The \rollback/ operation could more forcefully prevent the likelihood ratio from being pushed away compared to the original clipping function. Formally, we have the following theorem.

\begin{figure}[!t]
	\if@twocolumn
	 		\def\lenjdiucsddeechdhjs{0.5}
	 \else
	 	 	\def\lenjdiucsddeechdhjs{0.4}	 	
	 \fi
	\centering
	\iffastcompile
		\caption{
		}\label{fig_fallback}
	\else
		\centerline{
		\begin{subfigure}{\lenjdiucsddeechdhjs\linewidth}
		{\includegraphics[width=1.0\linewidth]{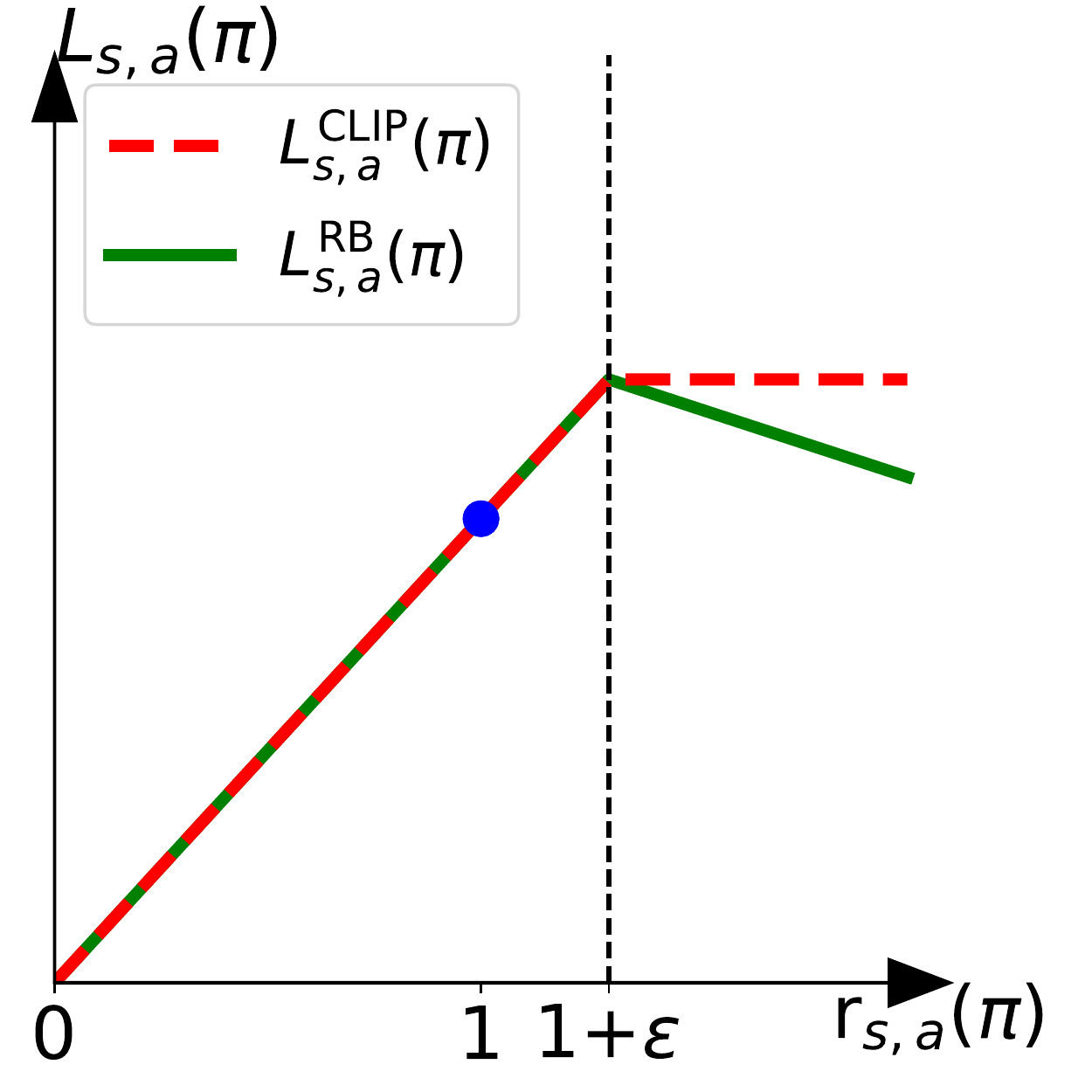} \label{fig_fallback_positive}}
		\setlength{\abovecaptionskip}{0mm}
		\caption{$A_{s,a}>0$}
		\end{subfigure}
		\begin{subfigure}{\lenjdiucsddeechdhjs\linewidth}
		{\includegraphics[width=1.0\linewidth]{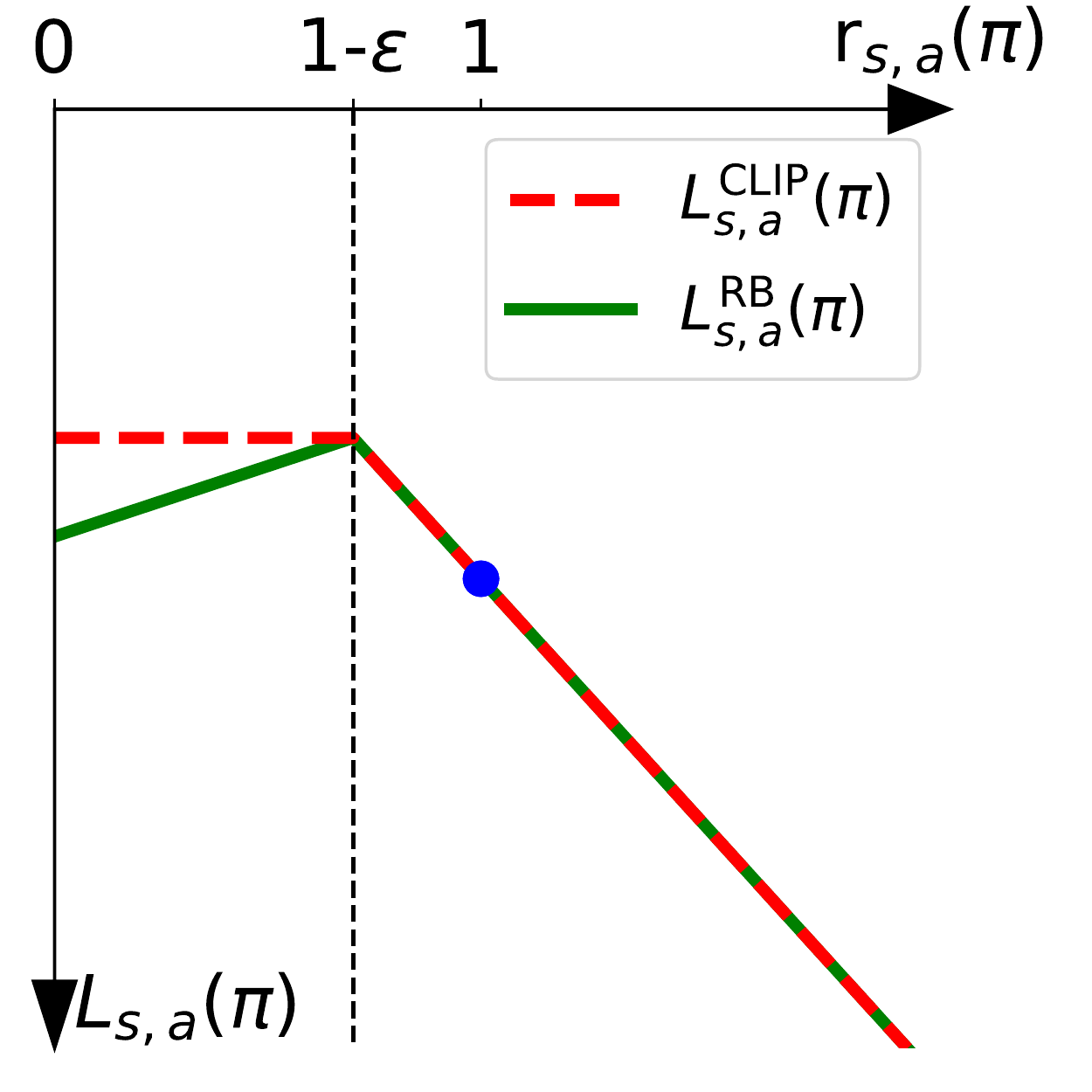} \label{fig_fallback_negative}}
		\setlength{\abovecaptionskip}{0mm}
		\caption{$A_{s,a}<0$}
		\end{subfigure}
		}
		\caption{
			Plots showing $\fnLexp{\FB}_{s,a}(\pi)$ and $\fnLexp{\clip}_{s,a}(\pi)$ as functions of the likelihood ratio $\ratio[s,a][\pi]$, for positive advantages (left) and negative advantages (right). 
			The blue circle on each plot shows the starting point for the optimization, i.e., $\ratio[s,a][\pi]=1$.
			When $\ratio[s,a][\pi]$ crosses the clipping range, the slope of $\fnLexp{\FB}_{s,a}$ is reversed, while that of $\fnLexp{\clip}_{s,a}$ is flattened.
		}\label{fig_fallback}
	\fi
\end{figure}

\begin{theorem}
Given parameter $\theta_0$, 
let $\theta_1^{\clip}=\theta_0+\beta \nabla\fnL{\clip}(\theta_0 )$, 
$\theta_1^{\FB}=\theta_0+\beta \nabla\fnL{\FB}(\theta_0 )$.
The set of indexes of the samples which satisfy the clipping condition is denoted as 
$\Omega=\{ t | 1\leq t \leq T, |\ratioa[t][0]-1|\geq \epsilon \text{ and } \ratioa[t][0] \adv[t] \geq \ratioa[t][\oldsub] \adv[t] \}$. 
If $t \in \Omega$ and $r_t(\theta_0)$ satisfies $\sum_{t' \in \Omega} { \langle \nabla r_{t}({ \theta_0 }), \nabla r_{t'}({ \theta_0 }) \rangle \adv[t] \adv[t']  }>0$, then there exists some $\bar \beta>0$ such that for any $\beta \in (0, \bar \beta)$, we have 
\begin{equation}
  \left|r_t({\theta_1^{\FB}}) -1 \right| < \left|r_t({\theta_1^{\clip}}) -1 \right|. \label{eq_RB_better_CLIP}
\end{equation}
\end{theorem}
\begin{proof}
Consider $\phi(\beta)=r_t( \theta_0+\beta \nabla\fnL{\FB}(\theta_0 ) )-r_t( \theta_0+\beta \nabla\fnL{\clip}(\theta_0 ) )$, 

By chain rule, we have 
\begin{equation}
	\begin{aligned}
		\phi'(0)
				& = \nabla r_t^\top(\theta_0)( \nabla\fnL{\FB}(\theta_0 ) - \nabla\fnL{\clip}(\theta_0 ) ) \\
				& = - \alpha \sum_{t' \in \Omega} { \langle \nabla r_{t}({ \theta_0 }), \nabla r_{t'}({ \theta_0 }) \rangle  \adv[t']  }  \\
	\end{aligned}
\end{equation}

For the case where $r_t(\theta_0) \geq 1+\epsilon$ and $ \adv[t]>0 $, we have $\phi'(0)<0$.

Hence, there exists $\bar \beta >0 $ such that for any $\beta \in (0, \bar \beta)$
\[
	\phi(\beta) < \phi(0)
\]
Thus, we have 
\[
	r_t(\theta_1^{\FB}) < r_t(\theta_1^{\FB})
\]
We obtain
\[
  \left|r_t({\theta_1^{\FB}}) -1 \right| < \left|r_t({\theta_1^{\clip}}) -1 \right|.
\]
Similarly, for the case where $r_t(\theta_0) \leq 1-\epsilon$ and $ \adv[t]<0 $, we also have $\left|r_t({\theta_1^{\FB}}) -1 \right| < \left|r_t({\theta_1^{\clip}}) -1 \right|$.
\end{proof}

This theorem implies that the rollback function can improve its ability in preventing the out-of-the-range ratios from going farther beyond the range.
{Ideally, if $\rbweight$ is sufficiently large, then the new policy are guaranteed to be confined within the clipping range.}

\begin{theorem}
Let $\pi\newp=\mathop{argmax}_{\pi} \fnLexp{ \FB }(\pi) $.
If $\rbweight \rightarrow +\infty $, then for any $(s,a)$ we have 
$$ \left|r_{s,a}(\pi{\newp}) -1\right| \leq \epsilon. $$
\end{theorem}
\begin{proof}
Similar to \cref{eq_L_PPO_case}, $\fnLexp{\FB}_{s,a}(\pi)$ can be rewritten as
\begin{subequations}
\if@twocolumn
	\def\lenjdiudddcchdhjs{0.3}
\else
 	\def\lenjdiudddcchdhjs{0.4}	 	
\fi

\begin{align}[left={\fnLexp{\FB}_{s,a}(\pi)=\empheqlbrace}]
	  	&{\quad \left({-\rbweight \ratio[s,a][\pi]} {+(1+\rbweight)(1-\epsilon)}  \right)}\adv[s,a] &  %
	  	  	\parbox[l]{\lenjdiudddcchdhjs\linewidth}{
	  				  	$\ratio[s,a][\pi] \leq 1-\epsilon$ \text{and} $ \adv[s,a]<0 $
	  		}  \notag
	  	\\
	  	& \quad {\left({-\rbweight \ratio[s,a][\pi]} {+(1+\rbweight)(1+\epsilon)}\right)}\adv[s,a] &   
	  	  	\parbox[l]{\lenjdiudddcchdhjs\linewidth}{
	  				  	$\ratio[s,a][\pi] \geq 1+\epsilon$ \text{and} $ \adv[s,a]>0 $
	  		} \notag
	  	\\
	  	& \quad \ratio[s,a][\pi]\adv[s,a] &  \text{otherwise \quad\quad\quad\quad} \notag
\end{align}
\end{subequations}
We prove the converse-positive of this theorem.
Assume that given an optimal policy $\pi'$, there exists $(s',a')$ which satisfies $ \left|r_{s',a'}(\pi{'}) -1\right| > \epsilon$. 
We consider the following cases:

\begin{itemize}
\item If $A_{s',a'}>0$ and $r_{s',a'}(\pi') > 1+ \epsilon$, then $\fnLexp{ \FB }_{s',a'}(\pi') = -\infty < \fnLexp{ \FB }_{s',a'}(\pi\oldsub)=A_{s',a'}$. 
\item If $A_{s',a'}>0$ and $r_{s',a'}(\pi') < 1- \epsilon$, then $\fnLexp{ \FB }_{s',a'}(\pi') < (1-\epsilon) A_{s',a'} < \fnLexp{ \FB }_{s',a'}(\pi\oldsub)=A_{s',a'}$. 
\end{itemize}

Similarly, if $A_{s',a'}<0$, we also have $\fnLexp{ \FB }_{s',a'}(\pi') < \fnLexp{ \FB }_{s',a'}(\pi\oldsub)$.

Finally, we can construct a policy 
$
{\pi''}(\cdot|s)=
\begin{cases}
\pi\oldsub(\cdot|s) & \text{if }\exists a \text{ such that } \left|r_{s',a}(\pi') -1\right| > \epsilon \\
\pi'(\cdot|s) & \text{otherwise}
\end{cases}
$, for which we have $ \fnLexp{ \FB }(\pi') < \fnLexp{ \FB }(\pi'') $.
This means that $\pi'$ is not an optimal solution of $\fnLexp{RB}$.
\end{proof}

\subsection{\Pmethodklfull/ (\pmethodkl/)}\label{sec_method_kl}
As discussed in Question \ref{que_KL}, there is a gap between the \ratiobased/ constraint and the trust region-based one: bounding the likelihood ratio is not sufficient to bound the KL divergence.
However, bounding the KL divergence is what we primarily concern about, since it is a theoretical indicator on the performance guarantee (see \Cref{thm_lowerbound}).
Therefore, new mechanism incorporating the KL divergence should be taken into account. 

The original clipping function of PPO employs the likelihood ratio as the \ttt{element} of the trigger condition for clipping.
Inspired by this thinking, we substitute the \ratiobased/ triggering condition with a trust region-based one.
Formally, the likelihood ratio is clipped when the policy $\pi$ is out of the trust region, 
\begin{equation}\label{eq_clipping_KL}
\begin{aligned}
	\fnbound{\KL}( {\ratio[s,a][\pi],\delta} ) 
	=
	\begin{cases}
		\ratio[s,a][\pi\oldsub] & D_{\rm KL}^{s}(\pi{\oldsub}, \pi) \geq \delta \\
		\ratio[s,a][\pi] & \text{otherwise}
	\end{cases}
\end{aligned}
\end{equation}
where $\delta$ is the parameter, $\ratio[s,a][\pi\oldsub]=1$ is a constant.
The incentive for updating policy is removed when the parametrized policy $\pi_\theta$ is out of the trust region, i.e., $D_{\rm KL}^{s_t}(\pip{\oldsub}, \pip) \geq \delta$.
Although the clipped value $\ratio[t][\theta\oldsub]$ may make the surrogate objective discontinuous, this discontinuity does not affect the optimization of the parameter $\theta$ at all, since the value of the constant does not affect the gradient.

\new{
In general, \pmethodkl/ can combine both the strengths of PPO and TRPO: it is simple to implement (requiring the first-order optimization) while it is somewhat theoretically justified (by the trust region constraint). 
Like PPO, \pmethodkl/ uses the clipping technique to restrict the policy.
The difference is that they use different policy metrics: \pmethodkl/ uses the KL divergence while PPO employs the likelihood ratio.
On the other hand, \pmethodkl/ attempts to restrict the policy within the trust region as TRPO does. 
What makes \pmethodkl/ distinctive is that the trust region-based constraint is used to decide whether to clip the likelihood ratio $\ratio[s,a][\pi]$ or not, without leading to optimizing the objective with a difficult constraint.
As a result, it allows using the first-order optimizer like Gradient Descent and significantly reduce the optimization complexity.
In other words, it can avoid the complex computation of higher-order optimization which are usually inaccurate (e.g. the second-order optimization of TRPO), resulting in a more stable process of optimization and leading to better solutions.
}

\subsection{{Combining \pmethodkl/ with \Rollback/ (\pmethodhybrid/)}}\label{sec_trulyPPO}
The \klbased/ clipping still possibly suffers from the unbounded KL divergence problem, since we do not enforce any negative incentive when the policy is out of the trust region.
Thus we integrate the \klbased/ clipping with the \rollback/ operation on KL divergence.
We do not use the \rollback/ operation on the likelihood ratio (like \pmethodfallback/) as our goal is to restrict the KL divergence.\footnote{In the preliminary version\citep{wang2019truly}, we heuristically combine \pmethodkl/ with the \rollback/ on the likelihood ratio,
\begin{subequations}
\begin{align}[left=	{{\fnLexp{TR-RB}_{s,a}(\pi )}=\empheqlbrace}]
	& - \rbweight \ratio[s,a][\pi] A_{s,a}
	 &  
  	\parbox[c]{.6\linewidth}{
				$D_{\rm KL}^{s}(\pi{\oldsub}, \pi) \geq \delta$
			  	\text{ and}
  				$ {\ratio[s,a][\pi] A_{s,a} \geq \ratio[s,a][\pi\oldsub] A_{s,a}} $ 
	} \notag
  	\\
  	& \quad \ratio[s,a][\pi] A_{s,a}  &  \text{\centering otherwise \hspace{1.5in} }  \notag
\end{align}
\end{subequations}
In this paper, we use the KL divergence-based one which owns a better theoretical property (by the monotonic improvement) and performs better in practice.
}
To make the formalism more intuitive, instead of using the clipping function $\fnbound{}$, we use the ``case'' form to formulate the objective (similar to \eqrefnop[eq_L_PPO_case_general]),

\begin{subequations}
\begin{align}[left=	{{\fnLexp{\pmethodhybridf}_{s,a}(\pi )}=\ratio[s,a][\pi] A_{s,a}-\empheqlbrace}]
	& \rbweight D_{\rm KL}^{s}(\pi{\oldsub}, \pi) 
	 &  
  	\parbox[c]{.4\linewidth}{
				$D_{\rm KL}^{s}(\pi{\oldsub}, \pi) \geq \delta$
			  	\text{ and} \\
  				$ {\ratio[s,a][\pi] A_{s,a} \geq \ratio[s,a][\pi\oldsub] A_{s,a}} $ 
	}
	\label{eq_L_TR_PPO_RB_case_general_condition}
  	\\
  	& \quad \delta  &  \text{otherwise \footnotemark  \quad\quad\quad\quad}  \label{eq_L_TR_PPO_RB_case_general_condition_otherwise}
\end{align}
\end{subequations}
\footnotetext{{The constant term $\delta$ is designed to make the function continuous.}}
As the equation implies, the objective generates a negative incentive on the KL divergence when $\pi_\theta$ is out of the trust region and the objective is improved.
The improvement condition ${\ratio[s,a][\pi] A_{s,a} \geq \ratio[s,a][\pi\oldsub] A_{s,a}}$ is the same as the one in \cref{eq_L_PPO_case_general} (as we have discussed in \Cref{sec_PPO}).

The ``rollback'' operation on the KL divergence can also be regarded as a penalty (regularization) term, 
which is also proposed as a variant of PPO \citep{schulman2017proximal}, 
$$
\fnLexp{penalty}_{s,a}(\pi) = \ratio[s,a][\pi]\adv[s,a] - \alpha D_{\rm KL}^s (\pi\oldsub, \pi)
$$
\new{The penalty-based methods are usually notorious by the difficulty of adjusting the trade-off coefficient.
And PPO-penalty addresses this issue by adaptively adjusting the rollback coefficient $\alpha$ to achieve a target value of the KL divergence.
However, the penalty-based PPO does not perform well as the clipping-based one, as it is difficult to find an effective coefficient-adjusting strategy across different tasks.
Our method introduces the ``clipping'' strategy to assist in restricting policy, i.e., the penalty is enforced only when the policy is out of the trust region. 
As for when the policy is inside the trust region, the objective function is not affected by the penalty term. 
Such a mechanism could relieve the difficulty on adjusting the trade-off coefficient, and it will not alter the theoretical property of monotonic improvement (as we will show below).
In practice, we found \pmethodhybrid/ to be more robust to the coefficient and achieve better performance across different tasks.
The clipping technique may be served as an effective method to enforce the restriction, which enjoys low optimization complexity and seems to be more robust.
}


To analyse the monotonic improvement property, 
we use the maximum KL divergence instead, i.e.,

\begin{subequations}\label{eq_TR_RB_maxKL}
\if@twocolumn
	\def\lenjdiucweetrrttchdhjs{0.47}
	\small
	$ {\quad\quad \fnLexp{\pmethodhybridf}_{s,a}(\pi )}=\ratio[s,a][\pi] A_{s,a}-$
\else
 	\def\lenjdiucweetrrttchdhjs{0.47}	 	
\fi

\begin{align}[left=	{{ \fnLexp{\pmethodhybridf}_{s,a}(\pi )}=\ratio[s,a][\pi] A_{s,a}-\empheqlbrace}]
	& \rbweight \max_{s'\in {\cal S}} D_{\rm KL}^{s'}(\pi{\oldsub}, \pi) 
	 &  
  	\parbox[c]{.47\linewidth}{
  		\scriptsize
				$ {\max_{s'\in {\cal S}}D_{\rm KL}^{s'}(\pi{\oldsub}, \pi) \geq \delta}$
			  	\text{ and} \\
  				$\exists a', {\ratio[s,a'][\pi] A_{s,a'} \geq \ratio[s,a'][\pi\oldsub] A_{s,a'}} $ 
	} \label{eq_condition_maxklout}
  	\\
  	& \quad \delta  &  \text{otherwise  \quad\quad\quad\quad}  
\end{align}
\end{subequations}
in which the maximum KL divergence is also used in TRPO for theoretical analysis. 
Such objective function also possesses the theoretical property of the guaranteed monotonic improvement.
Let $\pinew[\pmethodhybridf]=\mathop{argmax}_{\pi} \fnLexp{ \pmethodhybridf }(\pi) $ and $\pinew[TRPO] = \mathop{argmax}_{\pi} \LM(\pi)$ denote the optimal solution of \pmethodhybrid/ and TRPO respectively.
We have the following theorem.
\begin{theorem}
If $\alpha=C\triangleq\mathop {\max }\limits_{s,a} \left| {{ \adv[s,a][] }} \right| {4\gamma }  /{{{(1 - \gamma )}^2}}$ and $\delta\leq \max_{s \in {\cal S}} D_{\rm KL}^{s}(\pi\oldsub, \pinew[TRPO]) $, then  $\eta(\pi\newsub^{\rm \pmethodhybridf}) \geq \eta(\pi\oldsub)$.
\end{theorem}
\begin{proof}
First, we prove two properties of $\pinew[TRPO]$.

\begin{itemize}
\item 
Note that $\LM( \pi ) = E_{s,a}\left[ r_{s,a}(\pi)A_{s,a} \right]  - \alpha \max_{s' \in {\cal S}} D_{\rm KL}^{s'}(\pi\oldsub, \pi).$
As $\pinew[TRPO]$ is the optimal solution of $\LM(\pi)$,
we have
\begin{equation}
\begin{aligned}
\footnotesize
\mathbb{E}_{a} \left[ \ratio[s,a][\pinew[TRPO]] A_{s,a} \right] \geq \mathbb{E}_{a} \left[ \ratio[s,a][\pi\oldsub ] A_{s,a} \right] \text{ for any }s
\end{aligned}
\label{eq_proof_hjsdhusd}
\end{equation}
{Suffice it to prove the counter-positive of \cref{eq_proof_hjsdhusd}.}
Assume $\pi'$ is an optimal solution of $M(\pi)$ and there exists some $s'$ such that
$\mathbb{E}_{a} \left[ \ratio[s',a][\pi' ] A_{s',a} \right] < \mathbb{E}_{a} \left[ \ratio[s',a][\pi\oldsub ] A_{s',a} \right]$,
then we can construct a new policy 
$$
\small
{\pi''}(\cdot|s)=
\begin{cases}
\pi\oldsub(\cdot|s) & \text{if } \mathbb{E}_{a} \left[ \ratio[s,a][\pi' ] A_{s,a} \right] < \mathbb{E}_{a} \left[ \ratio[s,a][\pi\oldsub ] A_{s,a} \right] \\
\pi'(\cdot|s) & \text{otherwise}
\end{cases}
$$
We have $M(\pi') < M(\pi'')$, which contradicts that $\pi'$ is an optimal policy.

\item 

Besides, by \cref{eq_proof_hjsdhusd}, we can also obtain that for any $s$ there exists at least one $a'$ such that ${\ratio[s,a'][\pi] A_{s,a'} \geq \ratio[s,a'][\pi\oldsub] A_{s,a'}}$.
Therefore, by condition \eqref{eq_condition_maxklout}, we have 
\begin{equation}
\begin{aligned}
&L^{\pmethodhybridf}( \pinew[TRPO] )+\eta(\pi\oldsub) \\
= & \LPG(\pinew[TRPO]) - \alpha {\max_{s\in {\cal S}}D_{\rm KL}^{s}(\pi{\oldsub}, \pinew[TRPO])}.
\end{aligned}
\end{equation}
\end{itemize}


Then, we prove that $\pinew[TRPO]$ is the optimal solution of $L^\pmethodhybridf$.
There are three cases.
\begin{itemize}
\item 
For $\pi'$ which satisfies 
${\max_{s\in {\cal S}}D_{\rm KL}^{s}(\pi{\oldsub}, \pi') \geq \delta}$ and there exists some $a'$ such that ${\ratio[s,a'][\pi'] A_{s,a'} \geq \ratio[s,a'][\pi\oldsub] A_{s,a'}}$ for any $s$,
we have
\begin{align}
&L^\pmethodhybridf(\pi') + \eta(\pi\oldsub) \\
=& L^{\PG}(\pi') - \alpha {\max_{s\in {\cal S}}D_{\rm KL}^{s}(\pi{\oldsub}, \pi')}  \\
\leq & L^{\PG}(\pinew[TRPO]) - \alpha {\max_{s\in {\cal S}}D_{\rm KL}^{s}(\pi{\oldsub}, \pinew[TRPO])} \label{eq_zhjyruw} \\
= &L^\pmethodhybridf(\pinew[TRPO]) + \eta(\pi\oldsub) \label{eq_}
\end{align}

\item 

For $\pi'$ which satisfies ${\max_{s\in {\cal S}}D_{\rm KL}^{s}(\pi{\oldsub}, \pi') < \delta}$,
we have 
\begin{align}
& L^\pmethodhybridf(\pi') + \eta(\pi\oldsub) \\
=& L^\PG(\pi') - \alpha \delta  \\
< & L^\PG(\pi') - \alpha {\max_{s\in {\cal S}}D_{\rm KL}^{s}(\pi{\oldsub}, \pi')} \\
\leq & L^\PG(\pinew[TRPO]) - \alpha {\max_{s\in {\cal S}}D_{\rm KL}^{s}(\pi{\oldsub}, \pinew[TRPO])} \\
= & L^\pmethodhybridf(\pinew[TRPO]) + \eta(\pi\oldsub)
\end{align}
\item 
We now prove the case of $\pi'$ which satisfies ${\max_{s\in {\cal S}}D_{\rm KL}^{s}(\pi{\oldsub}, \pi') \geq \delta}$ and 
there exists some $s'$ such that 
${\ratio[s',a][\pi'] A_{s',a} < \ratio[s',a][\pi\oldsub] A_{s',a}} \text{ for any }a. $
We have
\begin{align}
& \mathbb{E}_{a} \left[ L_{s',a}^{\pmethodhybridf}(\pi') \right] \\
= & \mathbb{E}_{a} \left[ r_{s',a}^{}(\pi') \right] - \alpha \delta \\
< & \mathbb{E}_{a} \left[ r_{s',a}^{}(\pi\oldsub) \right] - \alpha \delta \\
\leq & \mathbb{E}_{a} \left[ r_{s',a}^{}(\pinew[TRPO]) \right] - \alpha \max_{s\in {\cal S}} D_{\rm KL}^s(\pinew[TRPO]) \\
=& \mathbb E_{a} \left[ L_{s',a}^{\pmethodhybridf}(\pinew[TRPO]) \right]
\end{align}

We can construct a new policy 
$$
\small
{\pi''}(\cdot|s)=
\begin{cases}
\pinew[TRPO](\cdot|s) & \text{if } s \in \{s'\} \\
\pi'(\cdot|s) & \text{otherwise}
\end{cases}
$$ for which we have
\begin{align}
& L_{}^{\pmethodhybridf}(\pi') + \eta(\pi\oldsub) \\
= & \mathbb{E}_{s,a} \left[ L_{s,a}^{\pmethodhybridf}(\pi') \right] + \eta(\pi\oldsub) \\
< & \mathbb{E}_{s,a} \left[ L_{s,a}^{\pmethodhybridf}(\pi'') \right] + \eta(\pi\oldsub) \\
= & \LPG(\pi'') - \alpha \max_{s\in {\cal S}} D_{\rm KL}^s(\pi'') \\
\leq & \LM(\pinew[TRPO]) \\
= & L_{}^{\pmethodhybridf}(\pinew[TRPO]) + \eta(\pi\oldsub) 
\end{align}


\end{itemize}

Finally, by \Cref{thm_lowerbound}, we have $\eta(\pinew[\pmethodhybridf]) = \eta(\pinew[TRPO]) \geq M(\pinew[TRPO]) \geq M(\pi\oldsub) = \eta(\pi\oldsub)  $.

\end{proof}


\section{Related Work}\label{sec_relatedwork}
Many researchers have extensively studied different approaches to enforce the constraint on policy updating.
Policy gradient-based methods \citep{sutton1999policy,kakade2001natural} update the parameter of the policy by several steps, which could be considered as a constraint in the parameter space.
\citeauthor{kakade2002approximately} firstly stated that improving the policy within a region in policy space leads to a better policy.
Followed by their work, the well-known trust region policy optimization (TRPO) incorporates a KL divergence constraint on policy, and \citeauthor{wu2017scalable} proposed an enhanced method which uses Kronecker-Factored trust regions (ACKTR).
Proximal Policy Optimization (PPO) \citep{schulman2017proximal} uses a clipping mechanism to enforce the constraint, which allows using the first-order optimization.

Several studies focus on investigating the clipping mechanism of PPO. 
In our previous work, we show that the \ratiobased/ clipping with a constant clipping range may fail when the policy is initialized from a bad one.
To address this problem, we proposed a method that adaptively adjusts the clipping ranges guided by the trust region criterion. 
In this paper, we also propose a method based on the trust-region criterion, but we use it as a triggering condition for clipping, which is much simpler to implement.
{\citeauthor{ilyas2018deep} performed a fine-grained examination and found that the PPO's performance depends heavily on optimization tricks but not the core clipping mechanism \citep{ilyas2018deep}.
However, as we found, although the clipping mechanism could not strictly restrict the policy, it does exert an essential effect in restricting the policy and maintain stability.}
We provide a detailed discussion with empirical results in \Cref{sec_experiment}.

Our methods are mostly related to several prior methods of constraining policy.
We give detail on the relation and difference between these methods.
We make comparisons from the following perspectives:
\emph{Restriction approach} --- the form of the objective function designing to restrict the policy, which could decide the \emph{optimization} complexity;
\emph{Boundness} --- whether the method could theoretically bound the policy by the corresponding metric;
\emph{Policy metric} --- the metric of the policy difference that the algorithm attempts to restrict.
\Cref{tab_relations} summarizes the properties of the algorithms.

\begin{table}[b!]

\centering
\setlength{\tabcolsep}{2pt}

\begin{tabular}{c|c|c|c|c}

\toprule
Algorithm    & \scriptsize{Restriction Approach}  & Optimization & Boundness & Policy Metric \\
\midrule
TRPO      & constraint & second-order   		&    \checkmark       & KL     				\\
PPO       & clipping   & first-order 		&      $\times$    & likelihood ratio  	 \\
PPO-RB    & clipping   & first-order       &      \checkmark     & likelihood ratio  	\\
TR-PPO    & clipping   & first-order      &      $\times$     & KL     				\\
\pmethodhybrid/ & clipping   & first-order       &     \checkmark    & KL     				\\
PPO-penalty & penalty	& first-order	& \checkmark & KL				\\
\bottomrule
\end{tabular}
\caption{
Properties of the algorithms.
}\label{tab_relations}
\end{table}

\subsection{Restriction Approach: constraint vs. clipping vs. penalty }

TRPO restricts the policy by explicitly imposing a \emph{constraint}, whereas the variants of PPO make restrictions by the \emph{clipping} mechanism. 
Schulman et al. also proposed a \emph{penalty}-based method, PPO-penalty, which imposes a penalty on the KL divergence and adaptively adjusts the penalty coefficient \citep{schulman2017proximal}.

The different restriction approaches could lead to different optimization complexity. 
To solve the constrained problem, TRPO involves second-order optimization and employs conjugate gradient optimization. 
Such computation could be inaccurate in high-dimensional tasks, leading to incorrect policy gradient.
While the clipping-based and penalty-based methods can use stochastic gradient descent (SGD) to train directly, which are much easier to implement and require relatively less computation.

These different restriction approaches could also result in different solutions.
The constraint-based method attempts to find an optimal solution within the policy constraint. 
{And the penalty-based methods can also be interpreted in this way by Lagrangian dual.}
While the clipping-based methods attempt to find a sub-optimal solution within the policy constraint, as these methods stop updating when the policy violates the constraint. 
However, in practice, the clipping-based methods usually perform much better than the other two ones. 
{One explanation is that the quality of the solution of the constraint-based methods heavily depends on the optimization, which is usually inaccurate, especially for DNNs.}
While for the penalty-based methods, it is hard to determine the coefficient of the penalty.


\subsection{Boundness}
The restriction approach and the optimization method could also affect the boundness on the policy.
Both constraint-based and penalty-based methods are guaranteed to bound the policy theoretically.
As for the original clipping-based methods, as we have discussed in \cref{sec_analysis}, it suffers from the unbounded problem.
However, by incorporating the rollback operation, such an issue is relieved, and the policy is also guaranteed to be bounded theoretically.

\subsection{Policy metric: KL divergence vs. likelihood ratio}\label{sec_related_policymetric}
These algorithms restrict the policy by a different metric. 
PPO and its variant with rollback operation (PPO-\FB) employ the ratio-based metric, 
while the other algorithms use the metric of KL divergence.
The divergence-based metric, $D_{\rm KL}^s (\pi\oldsub,\pi)=\mathbb{E}_{a} \left[ \log \frac{\pi\oldsub(a|s)}{\pi(a|s)} \right] \leq \delta$, imposes a summation constraint over the action space; while the ratio-based one, $1-\epsilon\leq \frac{\pi(a|s)}{\pi\oldsub(a|s)}\leq1+\epsilon$, is an element-wise one on each action point. 
The KL divergence metric is more theoretically justified according to the trust-region theorem. 
In fact, as we have discussed in Section \ref{sec_analysis}, bounding the likelihood ratio \ttt{at} one action point does not necessarily lead to bounded KL divergence.

In our previous work, we showed that different metrics of the policy difference could result in different algorithmic behavior \citep{wang2019trust}.
We found that the ratio-based constraint is prone to make the policy be trapped in a bad local optimum when the policy is initialized from a bad solution.
To show this, the constraint $1-\epsilon \leq \pi(a|s)/\pi\oldsub(a|s) \leq 1+\epsilon$ can be rewritten as $ -\pi\oldsub(a|s) \epsilon \leq \pi(a|s)-\pi\oldsub(a|s) \leq  \pi\oldsub(a|s) \epsilon $, which can reflect the allowable change of the likelihood $\pi(a|s)$.
We can observe that such constraints impose relatively strict restrictions on actions which are not preferred by the old policy (i.e., $\pi\oldsub(a_t|s_t)$ is small). 
Such bias may continuously weaken the likelihood of choosing the optimal action when the initial likelihood $\pi\oldsub(a_{\rm optimal}|s)$ is small.
We refer interested readers to \citet{wang2019trust} for further details.
While the \klbased/ one is averaged over the action space and thus has no such bias. We found it usually could learn better in practice.






%

\section{Experiment}\label{sec_experiment}
We conducted experiments to investigate whether the proposed methods could improve ability in restricting the policy and accordingly benefit the learning.
We will first describe the experimental setup.
Then the effect on restricting policy and improving performance will be presented.
Finally, comparison with several baseline methods and the state-of-the-art methods will be demonstrated.

\subsection{Experimental Setup}
To measure the behavior and the performance of the algorithm, we evaluate the likelihood ratio, the KL divergence, and the episode reward during the training process.
The likelihood ratio and the KL divergence are measured between the new policy and the old one at each epoch.
We refer one epoch as: 1) sample state-actions from a behavior policy $\pi_{\pip\oldsub}$; 2) optimize the policy $\pi_{\pip}$ with the surrogate function and obtain a new policy $\pi_{\theta\newsub}$.



We evaluate the following algorithms.
(a) \emph{PPO}: the original PPO algorithm. We used $\epsilon=0.2$, which is recommended by \citet{schulman2017proximal}. 
We also tested PPO with $\epsilon=0.6$, denoted as \emph{PPO-0.6}.
(b) \emph{\pmethodfallback/}: PPO with the extra \rollback/ trick. The {\rollback/ coefficient is set to be $\rbweight=0.3$} for all tasks (except for the Humanoid task we use $\rbweight=0.02$).
(c) \emph{\pmethodkl/}:  \pmethodklfull/. The trust-region coefficient is set to be $\delta=0.35$ for all tasks (except for the Humanoid task we use $\delta=0.05$).
(d) \emph{\pmethodhybrid/}: the coefficients are set to be $\delta=0.03$ and $\rbweight=5$ (except for the Humanoid task we use $\delta=0.05$). 
(e) \emph{PPO-penalty}: a a variant of PPO which adaptively adjust the coefficient of the KL divergence penalty \citep{schulman2017proximal}.
(f) \emph{TRPO}: restricting the policy by enforcing a hard constraint.
(g) \emph{A2C}: a classic policy gradient method. 
(h) \emph{SAC}: Soft Actor-Critic \citep{haarnoja2018soft}, a state-of-the-art off-policy RL algorithm.
(i) \emph{TD3}: Twin Delayed Deep Deterministic policy gradient \citep{fujimoto18td3}, a state-of-the-art off-policy RL algorithm which is competitive with SAC.
All our proposed methods and PPO adopts the same implementations given by \citet{baselines}. This ensures that the differences are due to the algorithm changes instead of the implementations. 
The implementation details of the proposed methods are given in Appendix \ref{sec_implementation_detail}.
For SAC and TD3, we adopt the implementations published by the original users \citep{haarnoja2018soft,fujimoto18td3}.

%

The algorithms are evaluated on continuous control benchmark tasks implemented in OpenAI Gym \citep{Brockman2016OpenAI} simulated by MuJoCo \citep{Todorov2012MuJoCo}  and Arcade Learning Environment \citep{bellemare2013arcade}.
For Mujoco, each algorithm was run with \nrandomseed/ random seeds, 1 million timesteps (except for the Humanoid task was 20 million timesteps); the trained policies are evaluated after sampling every 2048 timesteps data. 
For Atari, the algorithms was run 4 random seeds, 10 million timesteps; we report the episode rewards of the policy during the training process.



\subsection{The Effect on Policy Restriction}

\begin{question}
Does PPO suffer from the issue in bounding the likelihood ratio and KL divergence as we have analysed?
\end{question}

In general, PPO could not strictly bound the likelihood ratio within the predefined clipping range.
As shown in \Cref{fig_clipfrac}, a {reasonable} proportion of the likelihood ratios of PPO are out of the clipping range on all tasks.
Notably on Hopper, Reacher, and Walker2d, over 30\% of the likelihood ratios exceed.
Moreover, as can be seen in \Cref{fig_ratiomax}, the maximum likelihood ratios of PPO achieve more than 4 on all tasks (the upper clipping range is 1.2).
In addition, the maximum KL divergence also grows as timestep increases (see \Cref{fig_klmax}).

{Nevertheless, the clipping technique of PPO still exerts an important effect on restricting the policy.}
To show this, we tested two variants of PPO: one uses $\epsilon=0.6$, denoted as \emph{PPO-0.6}; another one entirely removes the clipping mechanism, which collapses to the vanilla \emph{A2C} algorithm. 
As expected, the maximum likelihood ratios and KL divergences of these two variants are {significantly larger than} those of PPO.
Moreover, the performance of these two methods fail on all the tasks and fluctuate dramatically during the training process (see \Cref{fig_rew_validate}).

\if@twocolumn
	\def\widthproperty{0.43}
\else
 	\def\widthproperty{0.245}
\fi

\begin{figure}[!b]
	\captionsetup[subfigure]{labelformat=empty}
    \centering
   	\centerline{
   		\includegraphics[width=\widthproperty\linewidth]{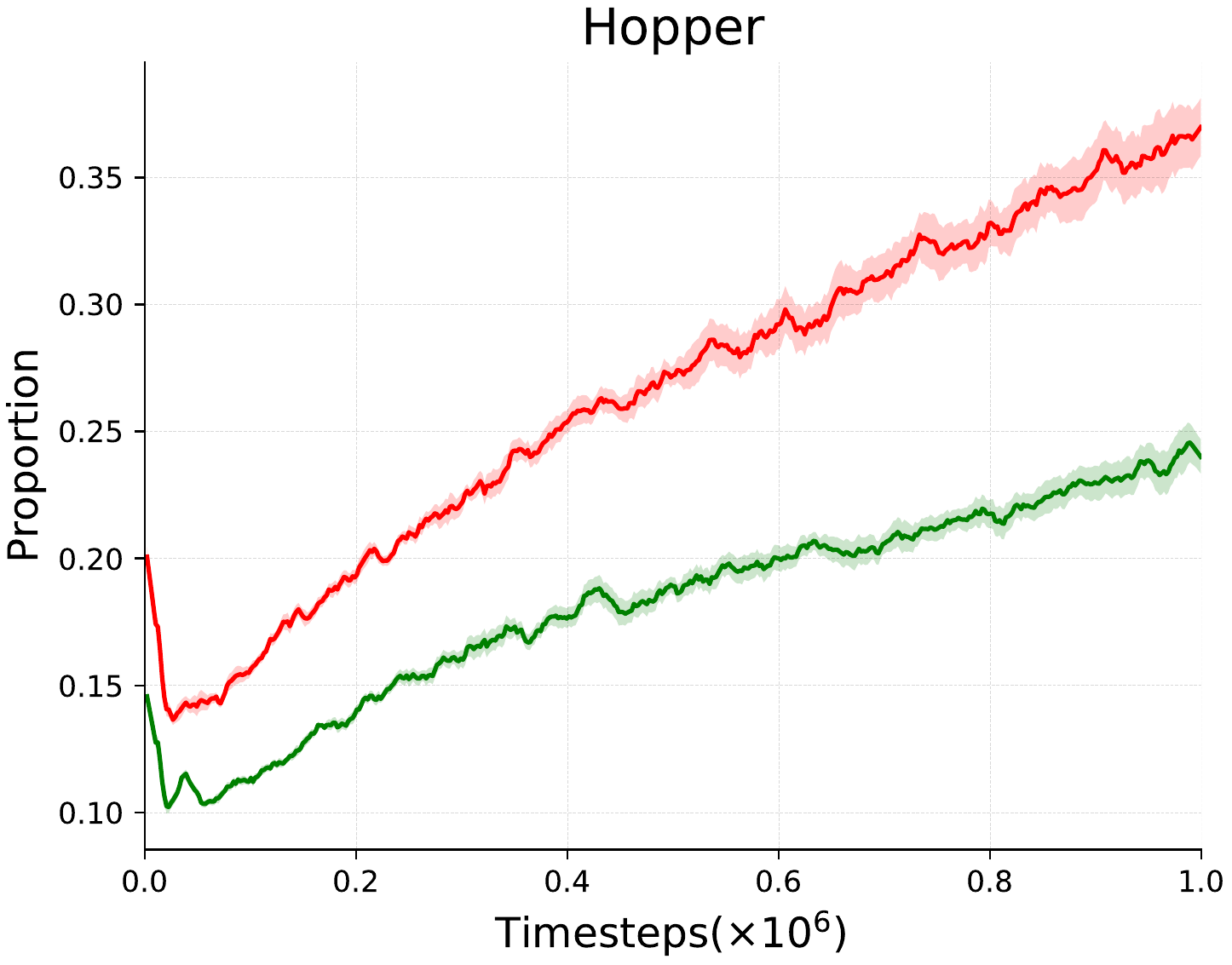}
   		\includegraphics[width=\widthproperty\linewidth]{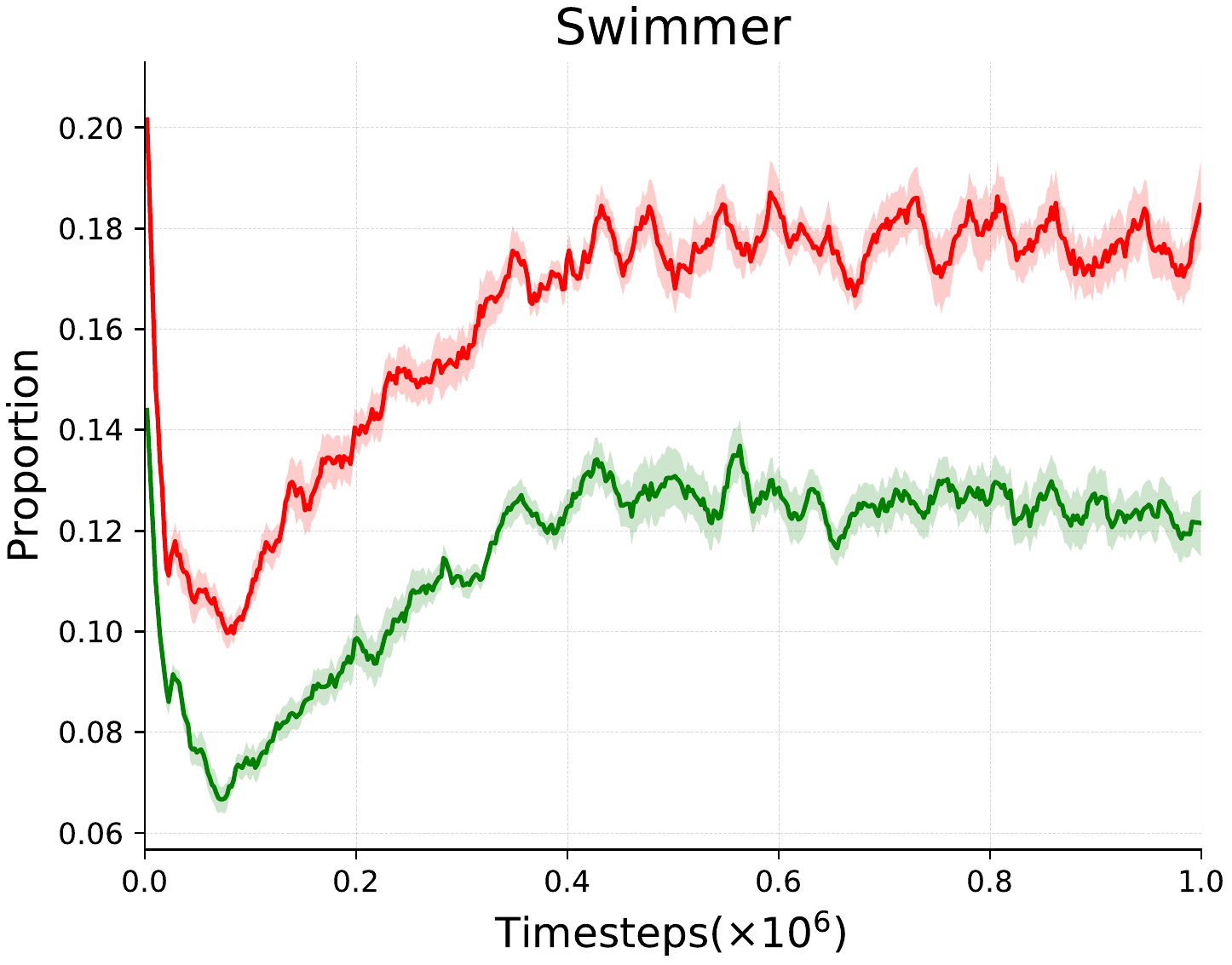}
   		\includegraphics[width=\widthproperty\linewidth]{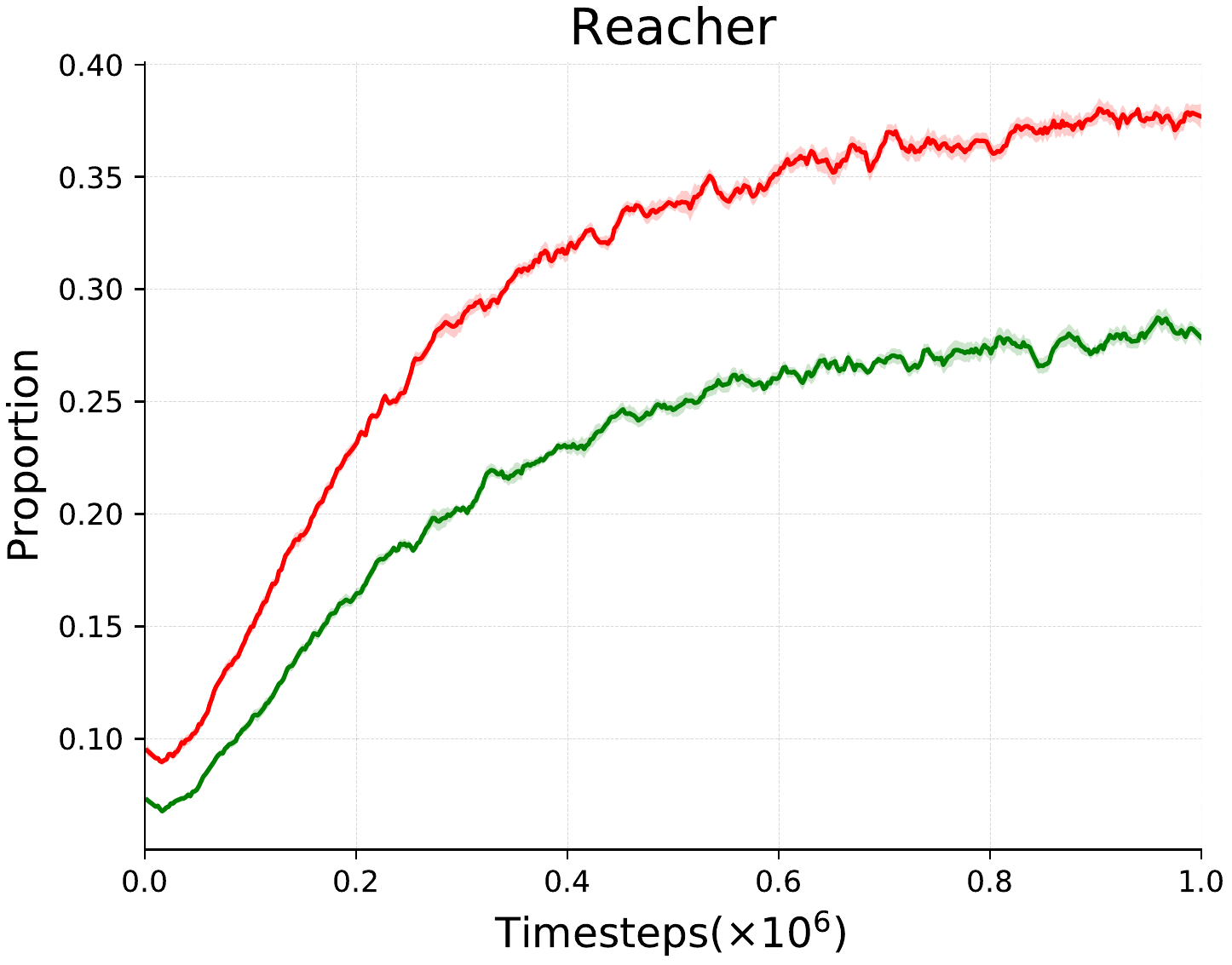}
   		\includegraphics[width=\widthproperty\linewidth]{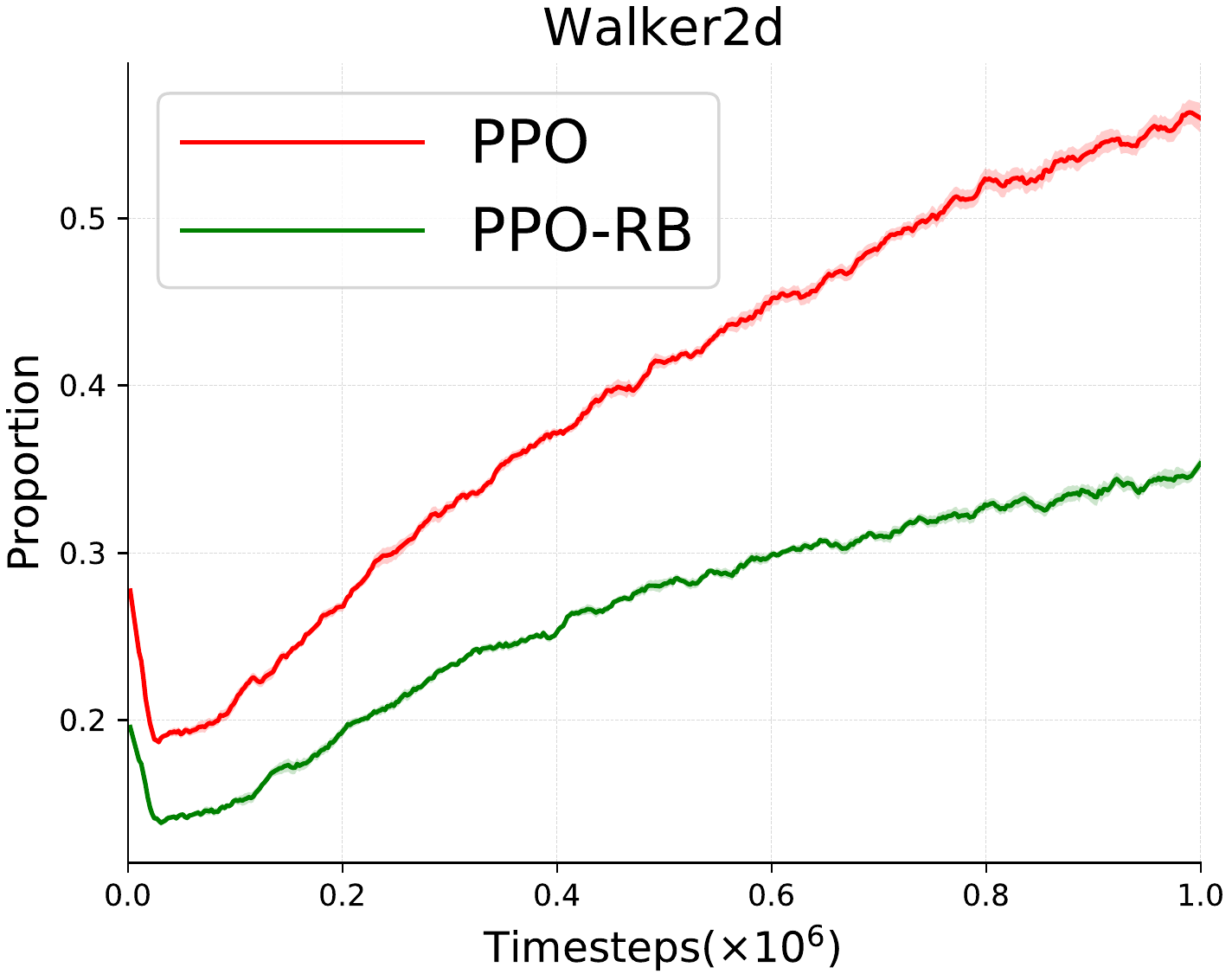}
   	}
	\iffastcompile
			\caption{
			}\label{fig_clipfrac}
	\else
		   \caption{
		   The proportions of the likelihood ratios which are out of the clipping range, i.e., $|r_{t}(\theta\newp)-1| \geq \epsilon$, where $\theta\newsub$ is the parameter at the end of each training epoch.
		   The proportions are calculated over all sampled state-actions at that epoch.
		   }\label{fig_clipfrac}
	\fi
\end{figure}

In summary, it could be concluded that although the core clipping mechanism of PPO could not strictly restrict the likelihood ratio within the predefined clipping range, it could somewhat take effect on restricting the policy and benefit the policy learning. 
This conclusion is partly different from that of \citet{ilyas2018deep}. 
They drew a conclusion that ``PPO's performance depends heavily on optimization tricks but not the core clipping mechanism''.
They got this conclusion by examining a variant of PPO which implements only the core clipping mechanism and removes additional {optimization tricks} (e.g., clipped value loss, reward scaling). 
This variant also fails in restricting policy and learning.
However, as can be seen in our results, arbitrarily enlarging the clipping range or removing the core clipping mechanism can also result in failure. 
These results confirm the critical and indispensable efficacy of the core clipping mechanism.



\begin{figure}[!t]
    \centering
   	\centerline{
   		\includegraphics[width=\widthproperty\linewidth]{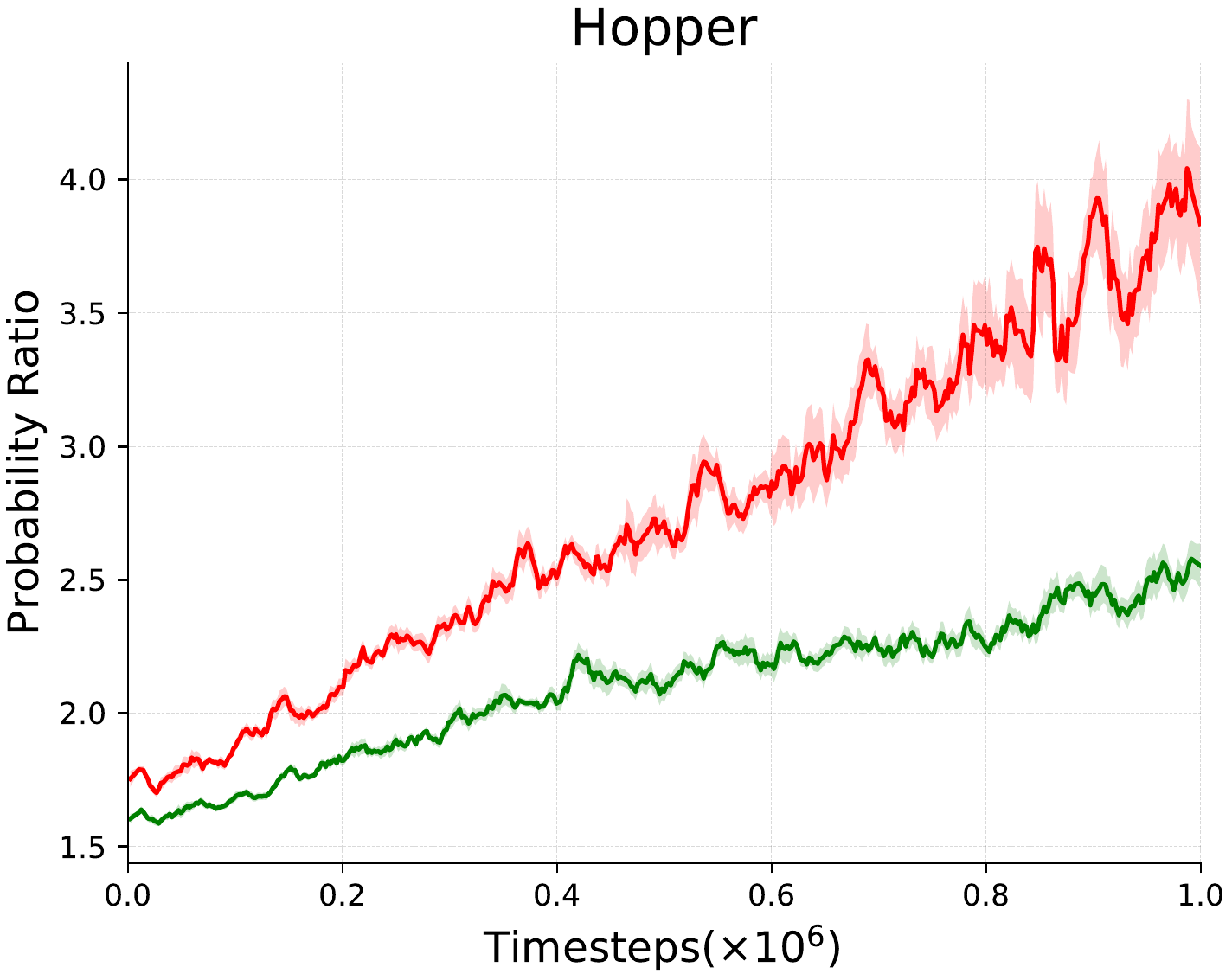}
   		\includegraphics[width=\widthproperty\linewidth]{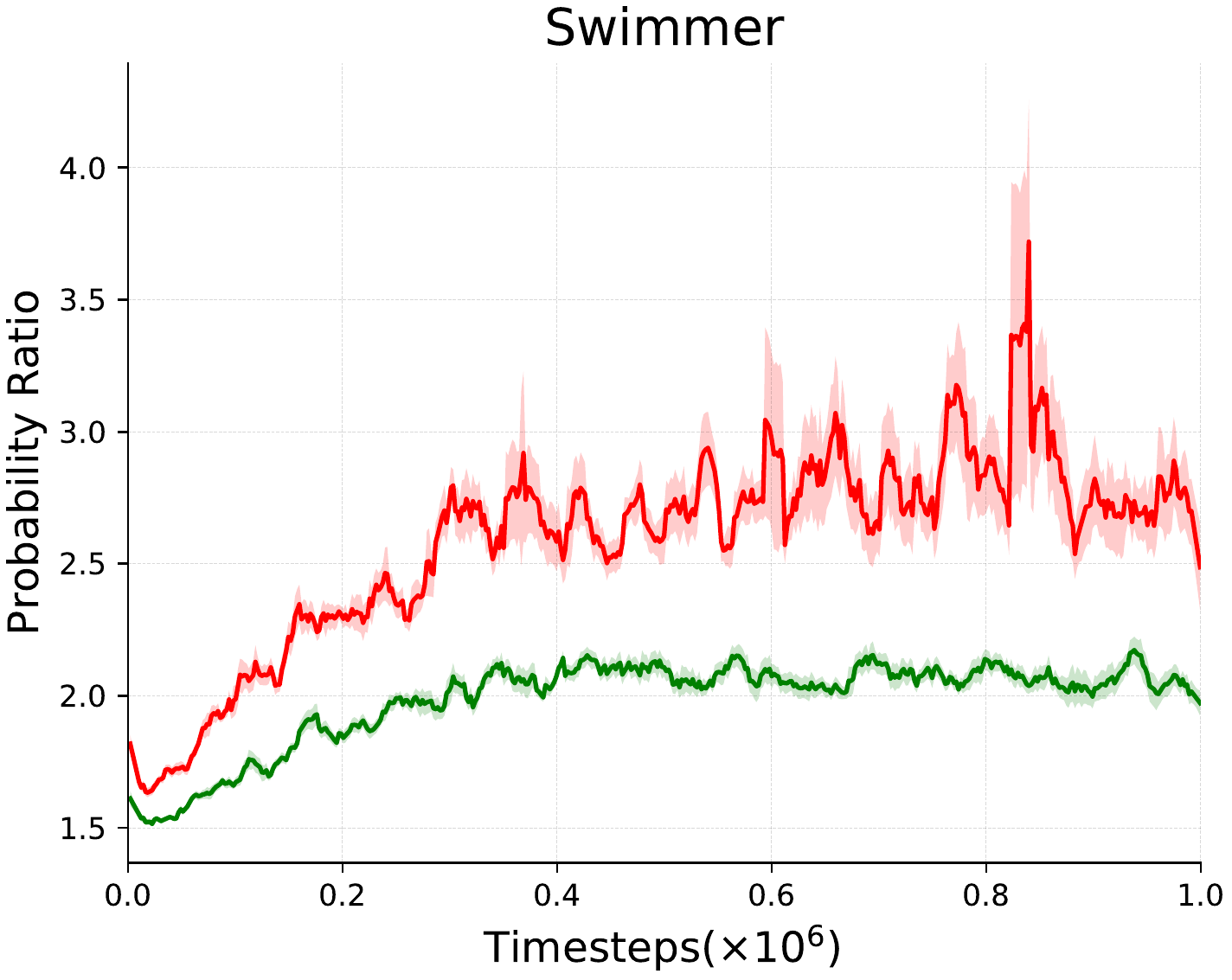}
   		\includegraphics[width=\widthproperty\linewidth]{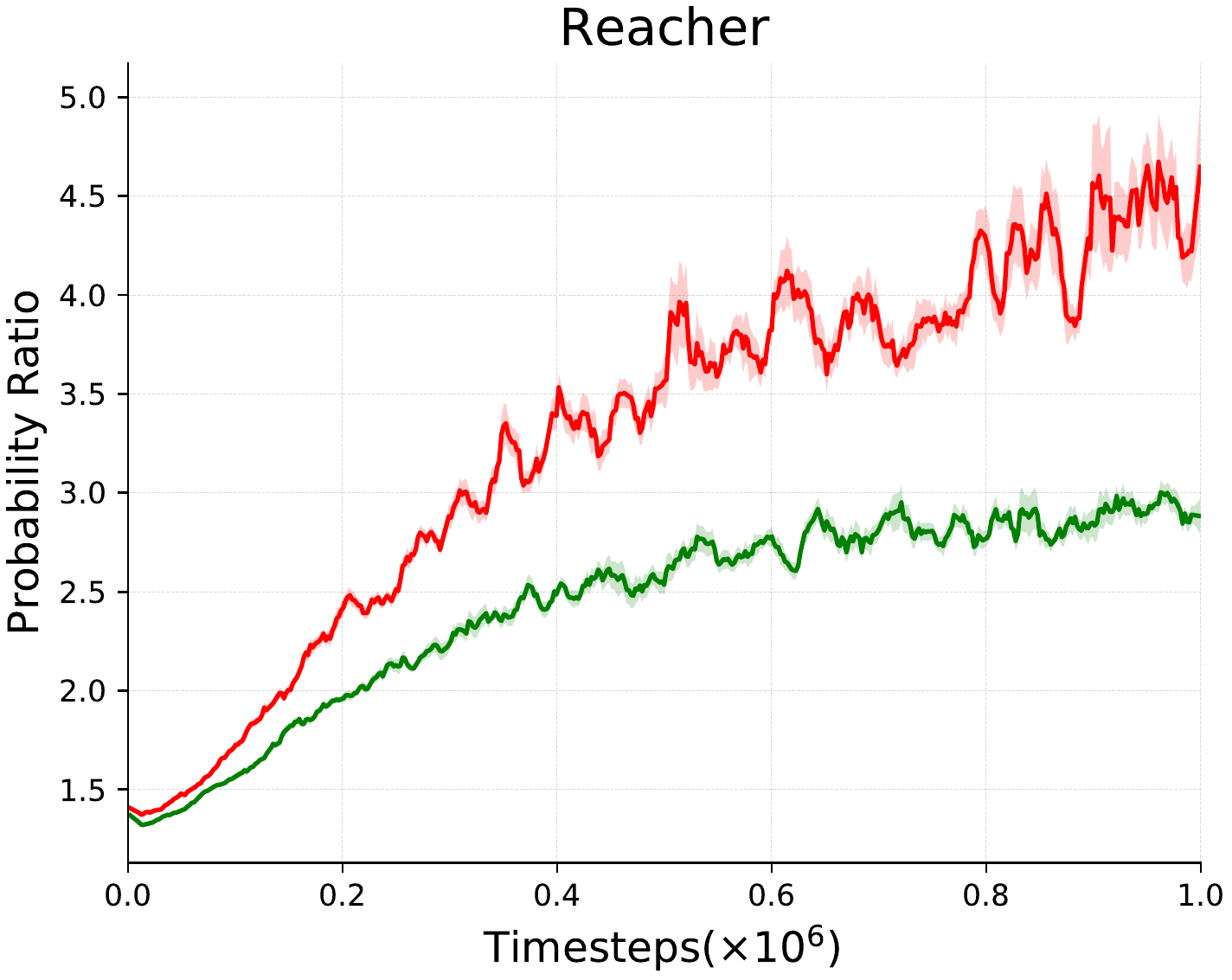}
   		\includegraphics[width=\widthproperty\linewidth]{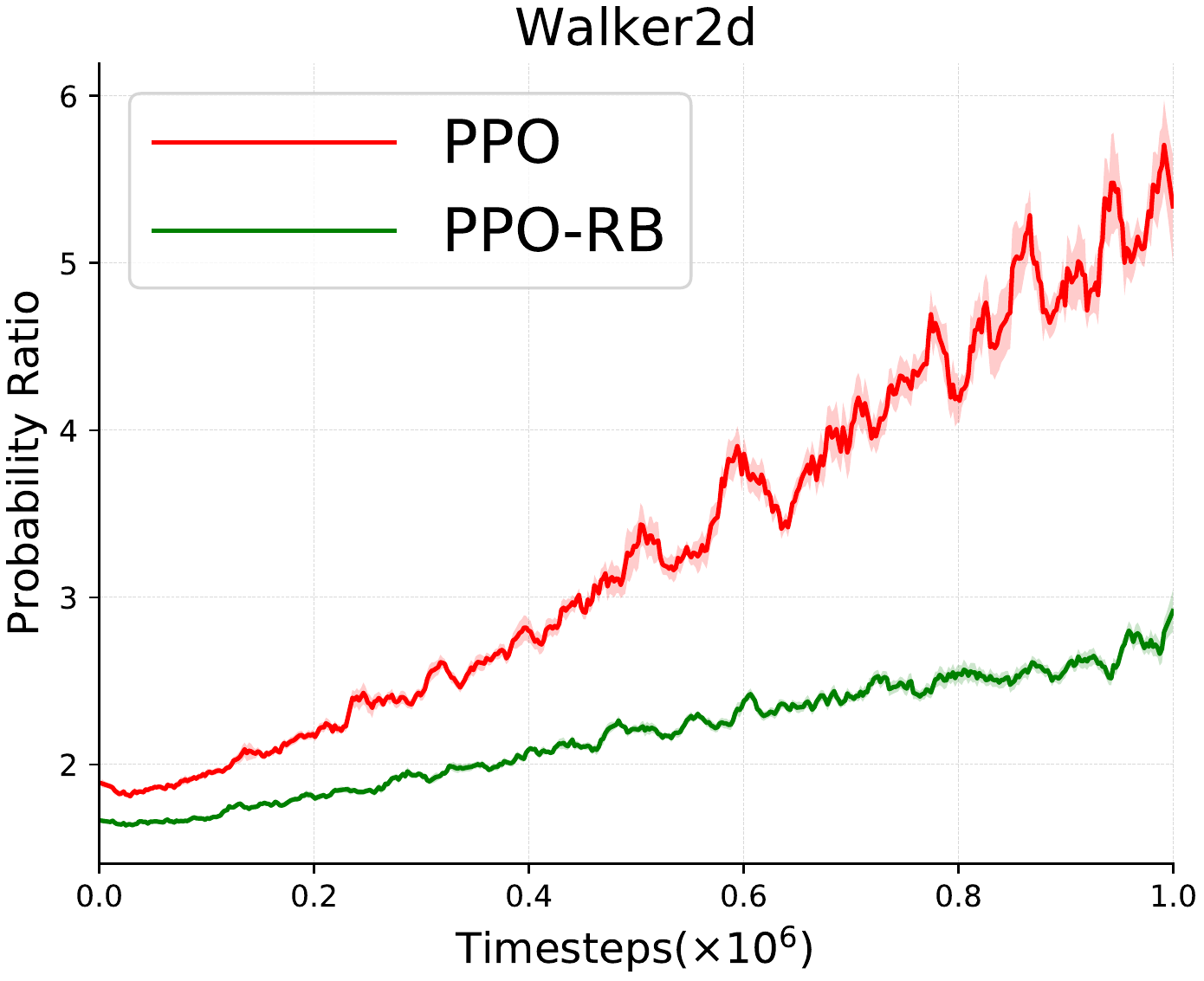}
   	}
	\iffastcompile
			\caption{
			}\label{fig_ratiomax}
	\else
    \caption{
    The maximum ratio over all sampled sates of each update during the training process, i.e., $\max_t r_t(\theta\newp)$.
    }\label{fig_ratiomax}
   	\fi
\end{figure}

\begin{figure}[!t]
    \centering
  	\centerline{
  		\includegraphics[width=\widthproperty\linewidth]{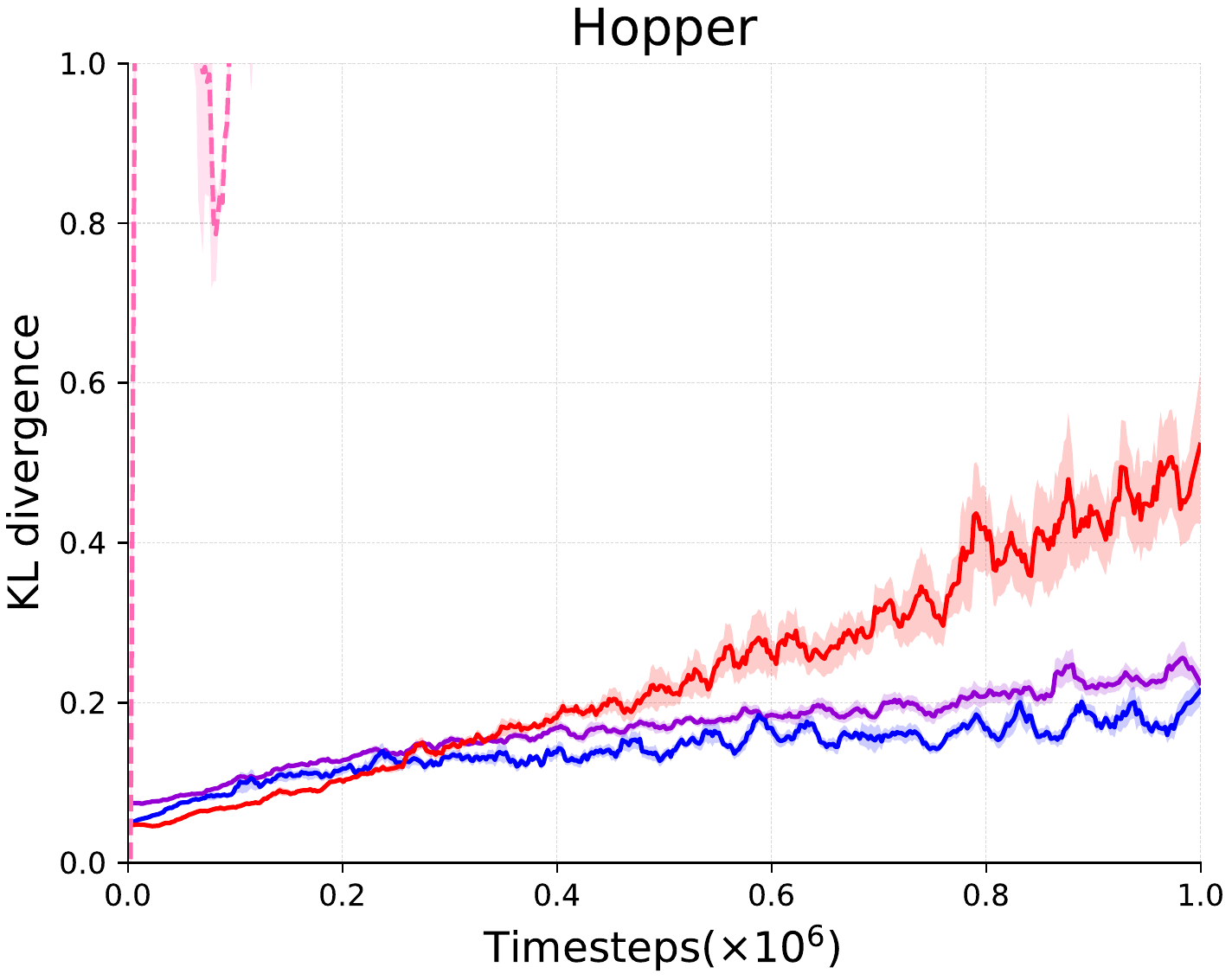}
  		\includegraphics[width=\widthproperty\linewidth]{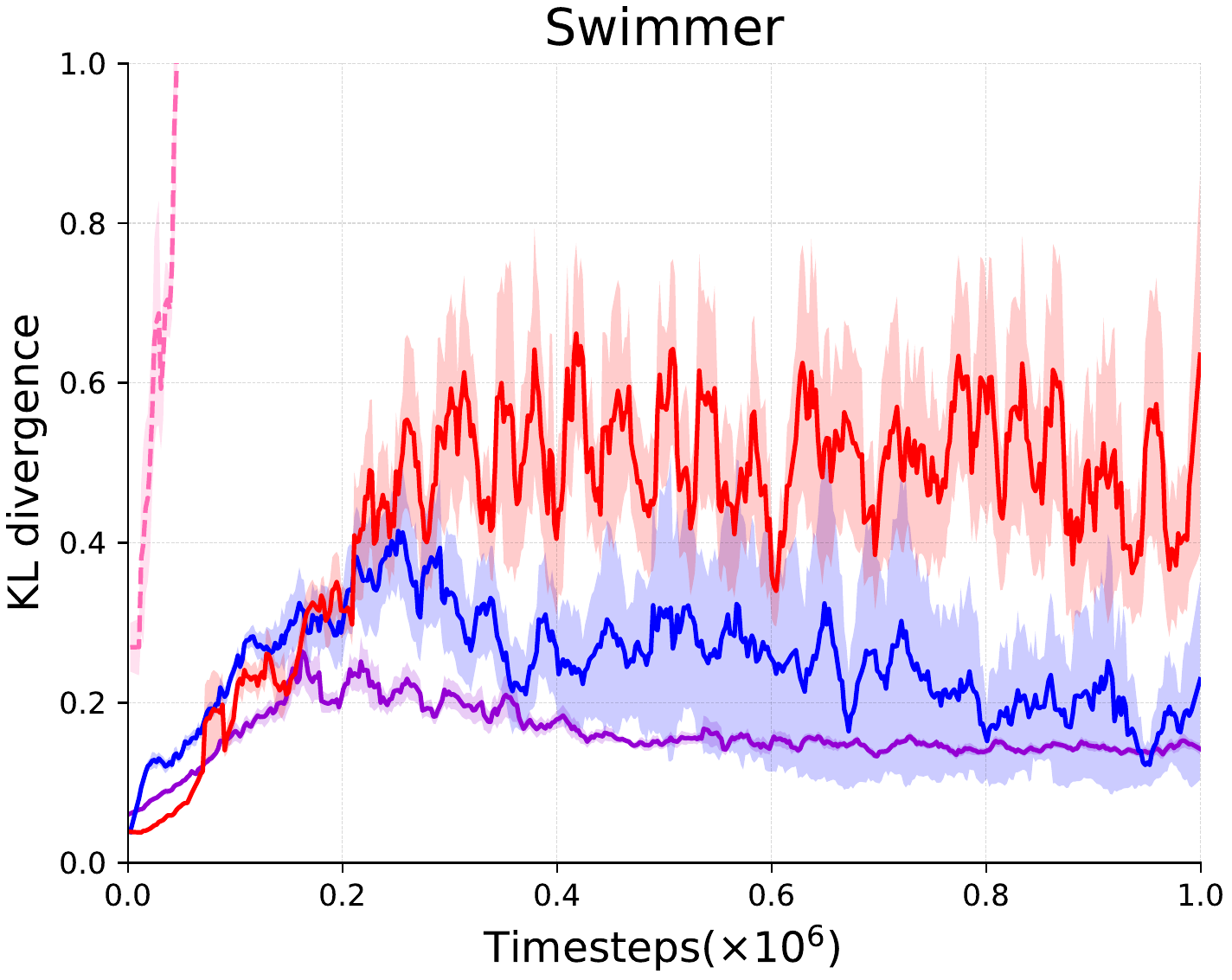}
  		\includegraphics[width=\widthproperty\linewidth]{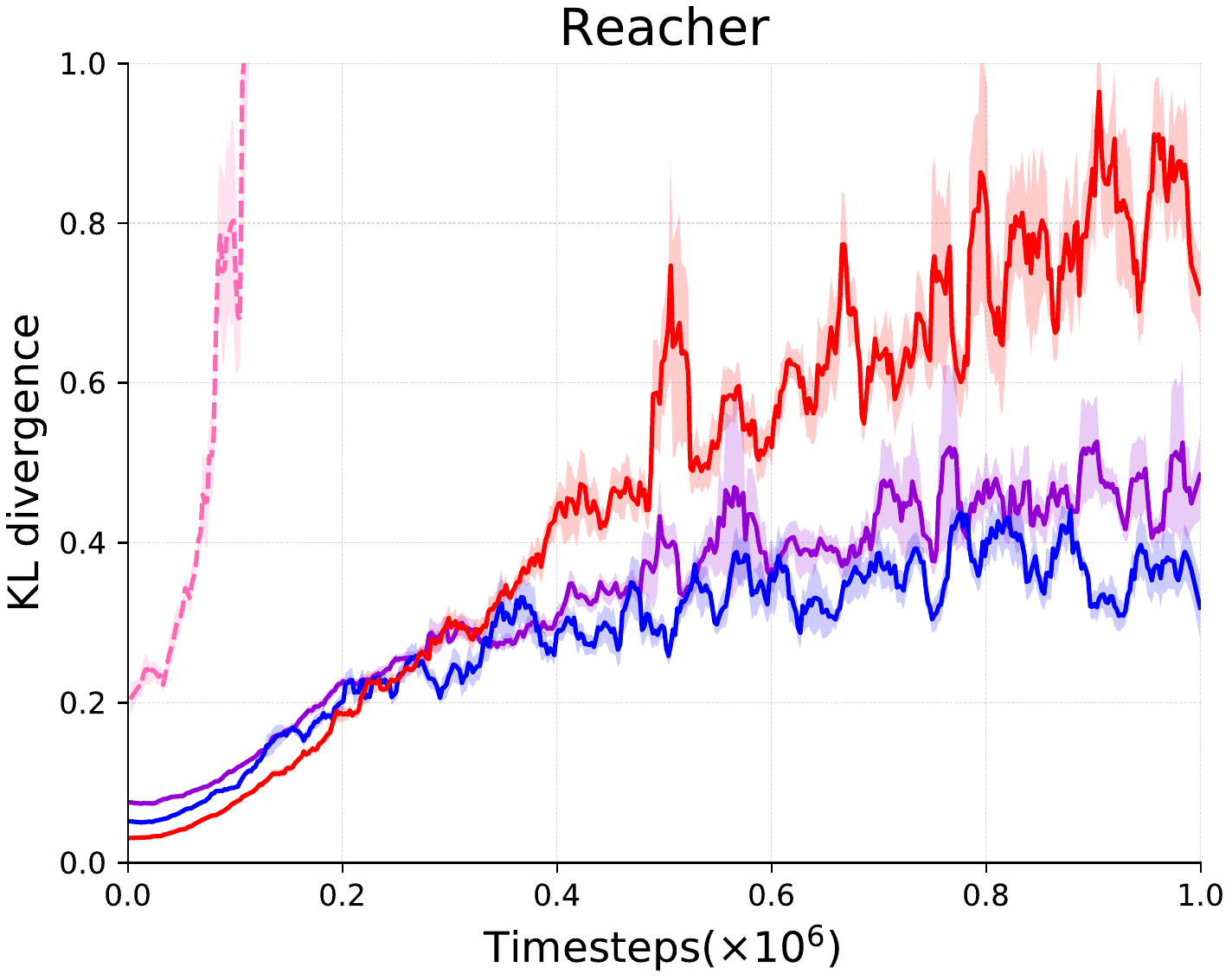}
  		\includegraphics[width=\widthproperty\linewidth]{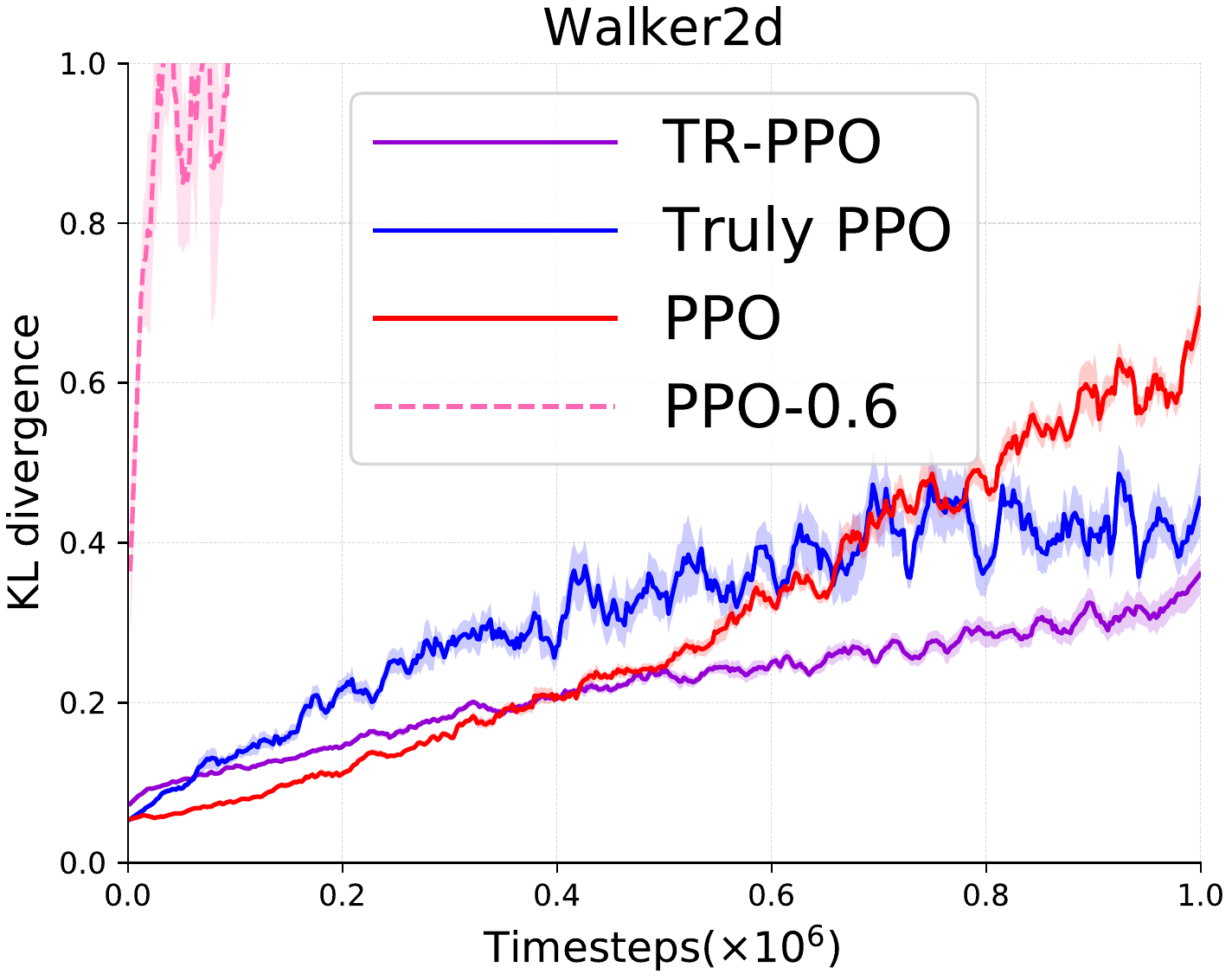}
  	}
	\iffastcompile
		\caption{
		}\label{fig_klmax}
	\else
		\caption{
		  The maximum KL divergence over all sampled states of each update during the training process, i.e., $\max_t D_{\rm KL}^{s_t}(\theta\oldsub, \theta\newp)$.
		  The curves of PPO-$0.6$ are clipped as the values can reach 40 which is too large on these tasks.
		}\label{fig_klmax}
	\fi
\end{figure}

\begin{question}
Could the \rollback/ operation and the trust region-based clipping improve its ability in bounding the likelihood ratio or the KL divergence?
\end{question}

In general, our new methods could take a significant effect in restricting the policy compared to PPO.
As can be seen in \Cref{fig_clipfrac}, the proportions of out-of-range likelihood ratios of \pmethodfallback/ are much less than those of the original PPO during the training process.
Besides, the likelihood ratios of \pmethodfallback/ are much smaller than those of PPO (see \Cref{fig_ratiomax}).
For the trust region-based clipping methods (\pmethodkl/ and \pmethodhybrid/), the KL divergences are also smaller than {those of PPO} (see \Cref{fig_klmax}).
{The maximum KL divergences of \pmethodhybrid/ are slightly larger than that of \pmethodkl/ on some tasks. This is because \pmethodhybrid/ retains the term of the likelihood ratio even when the policy is out of the trust region, which could push the KL divergences to increase. However, maintaining the likelihood term could benefit the policy performance of interest, as we will show below.}

\subsection{The Effect on Policy Performance}

\Cref{tab_reward_hit} lists learning speed and final rewards on Mujoco tasks.
\Cref{fig_rew} and \Cref{fig_rew_atari} plots episode rewards during the training process on Mujoco and Atari tasks, respectively.

In general, our \pmethodhybrid/ method, combining both the rollback operation and the trust region-based clipping, significantly outperform the original PPO on hard tasks characterized by high dimension (e.g., Walker2d with $|{\cal A}|=6$), both in terms of learning speed and final rewards; and \pmethodhybrid/ is comparable to PPO on the easier tasks with low dimension (e.g., Reacher with $|{\cal A}|=2$).
{Notably, on Mujoco tasks like Walker2d and Hopper, \pmethodhybrid/ requires almost 60\% and 50\% timesteps of PPO to hit the thresholds; and it achieves about 15\% and 24\% higher final rewards than PPO does on these tasks.}
{On Atari tasks {like} Breakout, \pmethodhybrid/ achieves almost twice the final rewards of PPO.}
We now investigate the effect of the two newly proposed techniques independently.

%

\begin{question}
\ttt{Could the \rollback/ operation benefit policy learning?}
\end{question}

We first consider two groups of comparisons: 
(1) {PPO vs. \pmethodfallback/;
(2) \pmethodkl/ vs. \pmethodhybrid/}.
The only difference within each group is the existence of the rollback operation.

The results show that the methods with rollback operation outperform the ones without that operation on most of the tasks.
For example, \pmethodhybrid/ achieves fairly better performance than \pmethodkl/ on 5 of 6 Mujoco tasks (see \Cref{fig_rew}) and 5 of 6 Atari tasks (see \Cref{fig_rew_atari}).
\pmethodfallback/ also performs much better than PPO on 4 of 6 Mujoco tasks. 
However, the improvements of \pmethodfallback/ in comparison of PPO are not significant on Atari tasks. 
{One reason is that the ratio-based constraint may lead to bad local optimum when the policy is initialized from a bad solution (as we have discussed in \Cref{sec_related_policymetric}).
While enhancing the ability of restriction may aggravate the issue, especially in the discrete action space tasks.}

In summary, the compared methods possess different ability to confine the policy and performs differently in practice.
According to the restriction ability in ascending order, the methods can be generally sorted as: without clipping (A2C), with loose clipping (PPO-$0.6$), with proper clipping (PPO, \pmethodkl/), with \rollback/ operation (\pmethodfallback/, \pmethodhybrid/).
As we have seen, the performance generally increases as the restriction ability increases. 
These results confirm the {necessity} of restricting the policy difference with the old policy.
Such improvements {may} be considered as a \ttt{justification} of the ``trust region'' theorem --- making the policy less greedy to the evaluated value of another policy (old policy) result in a better policy.

	\begin{figure*}[!t]
	    \centering
		\centerline{
			\includegraphics[width=0.33\linewidth]{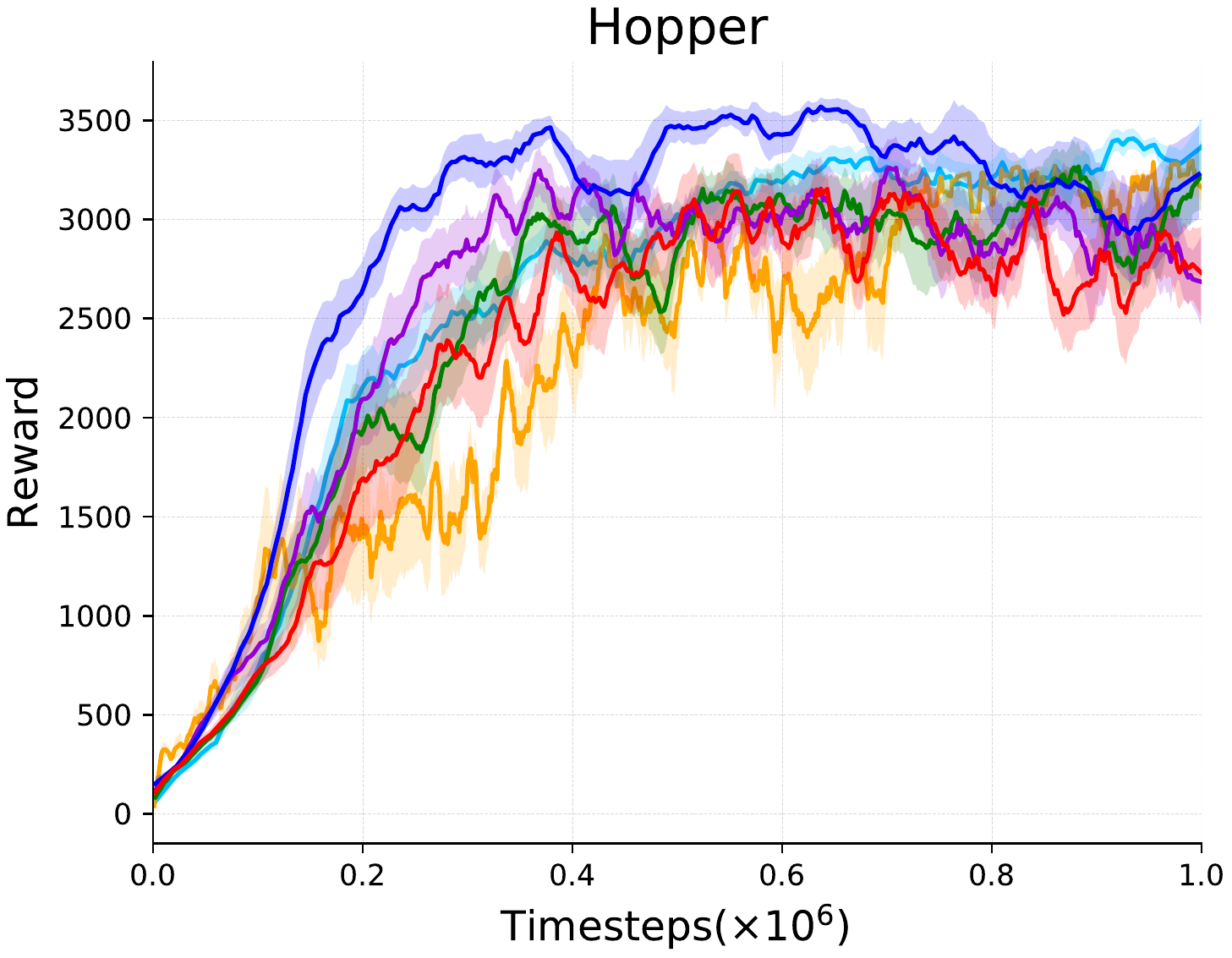}
			\includegraphics[width=0.33\linewidth]{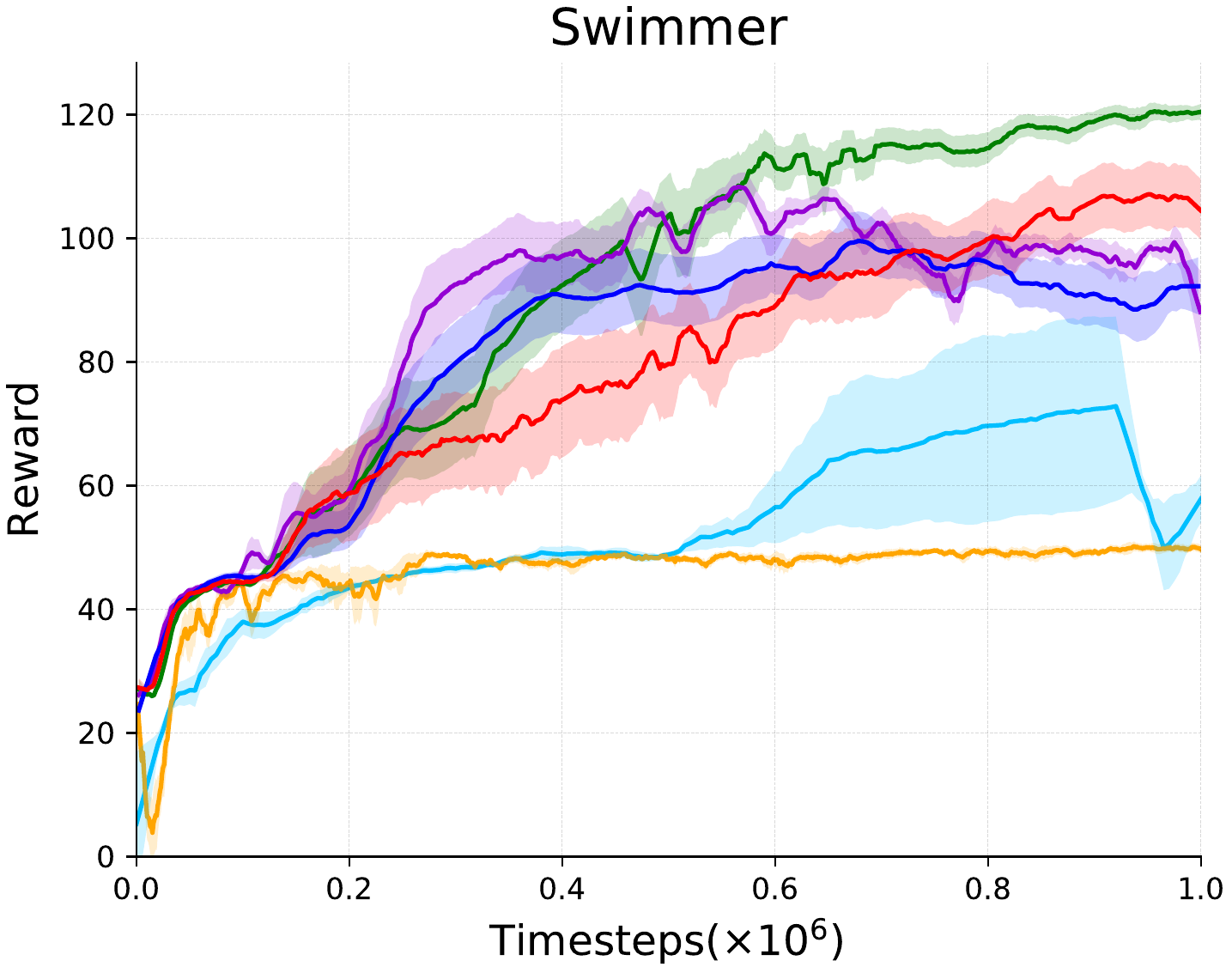}
			\includegraphics[width=0.33\linewidth]{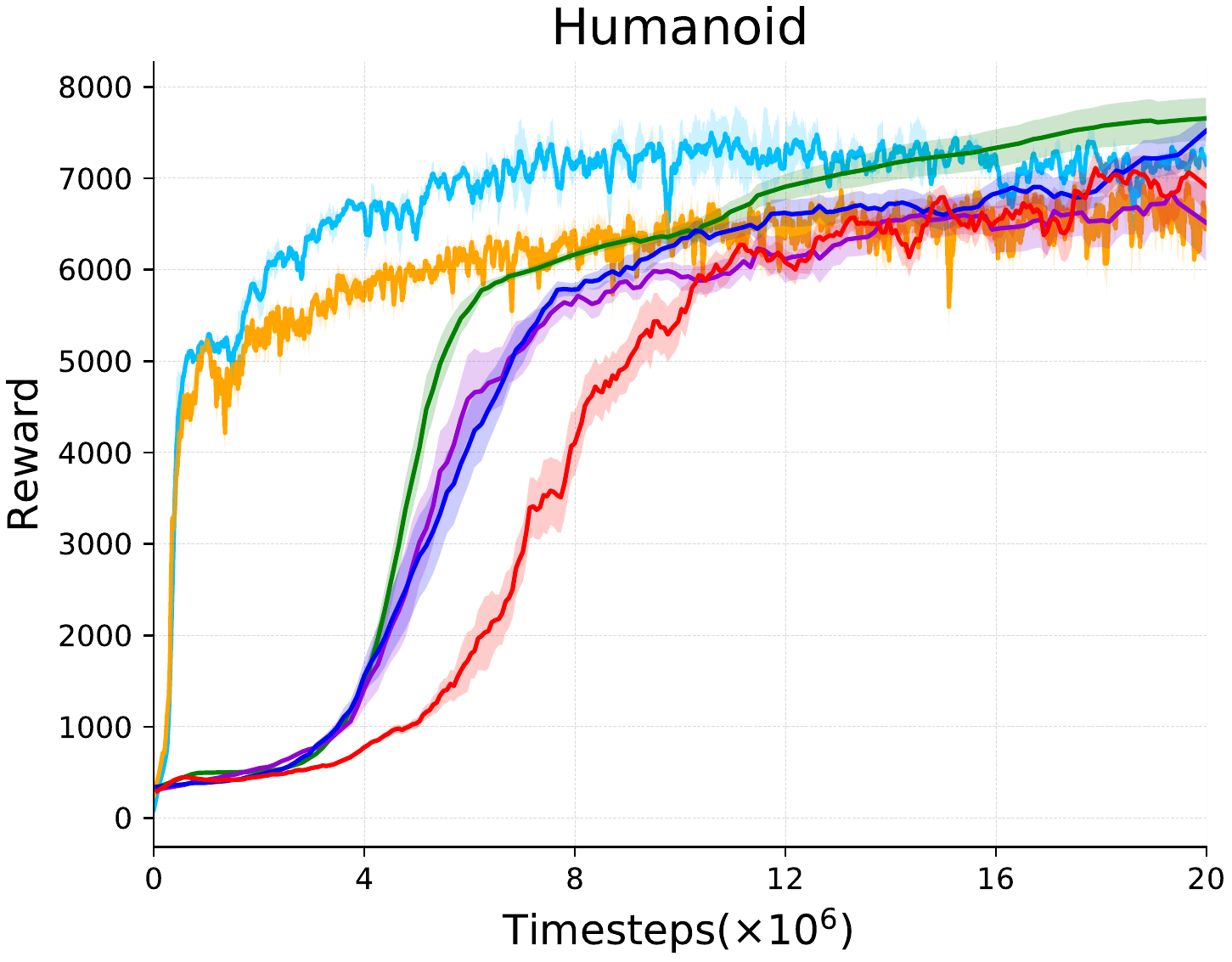}
		}
		\centerline{
			\includegraphics[width=0.33\linewidth]{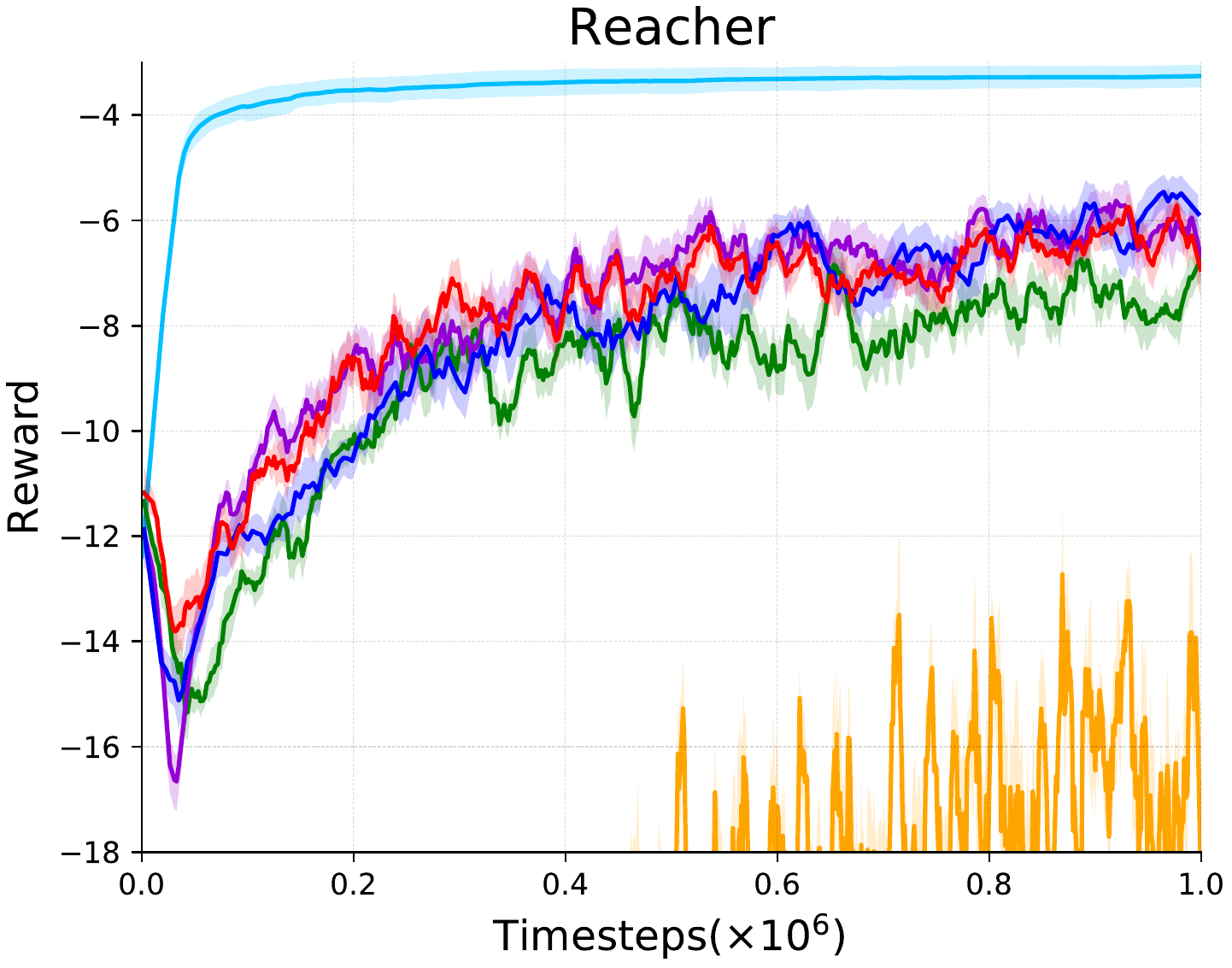}
			\includegraphics[width=0.33\linewidth]{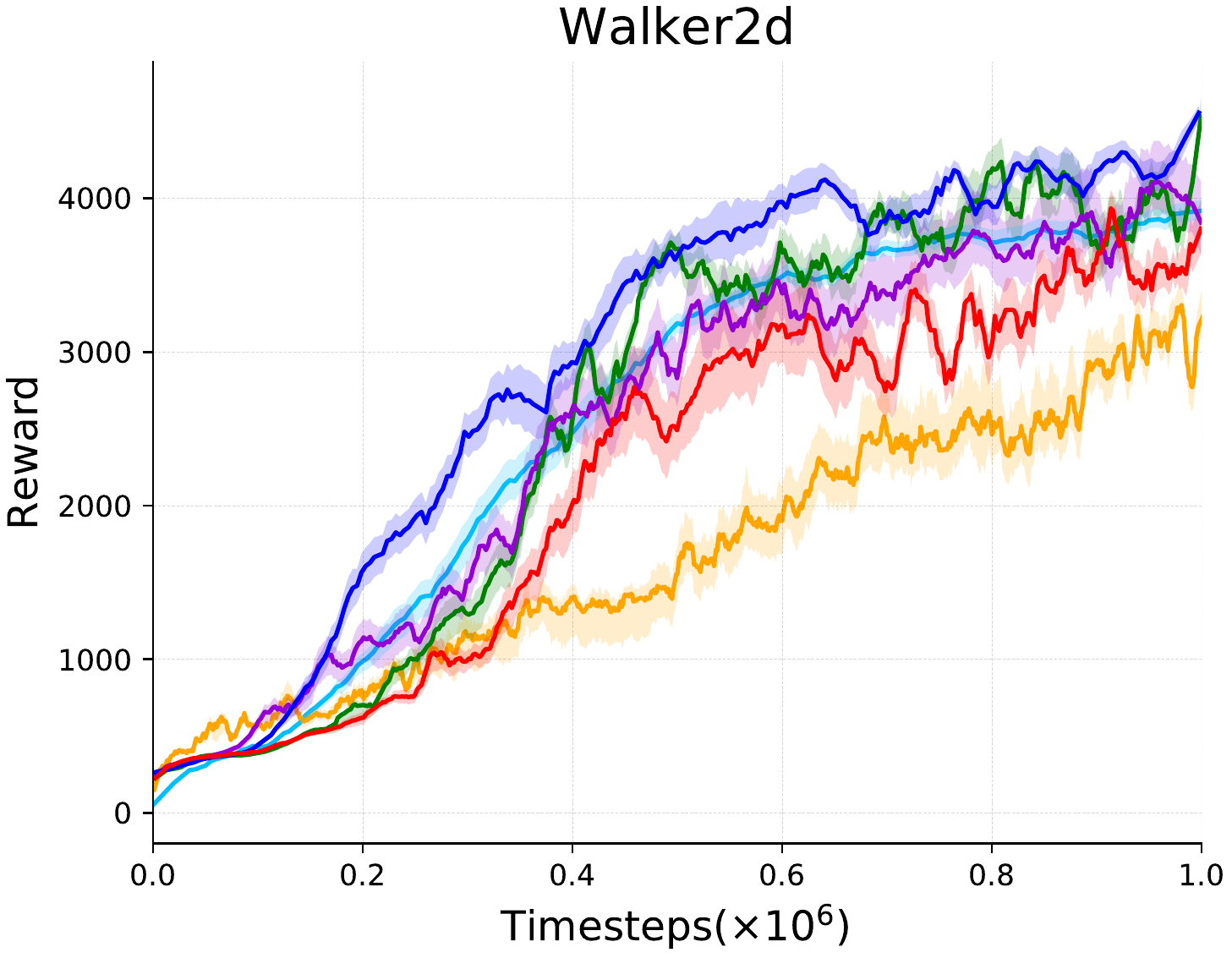}
			\includegraphics[width=0.33\linewidth]{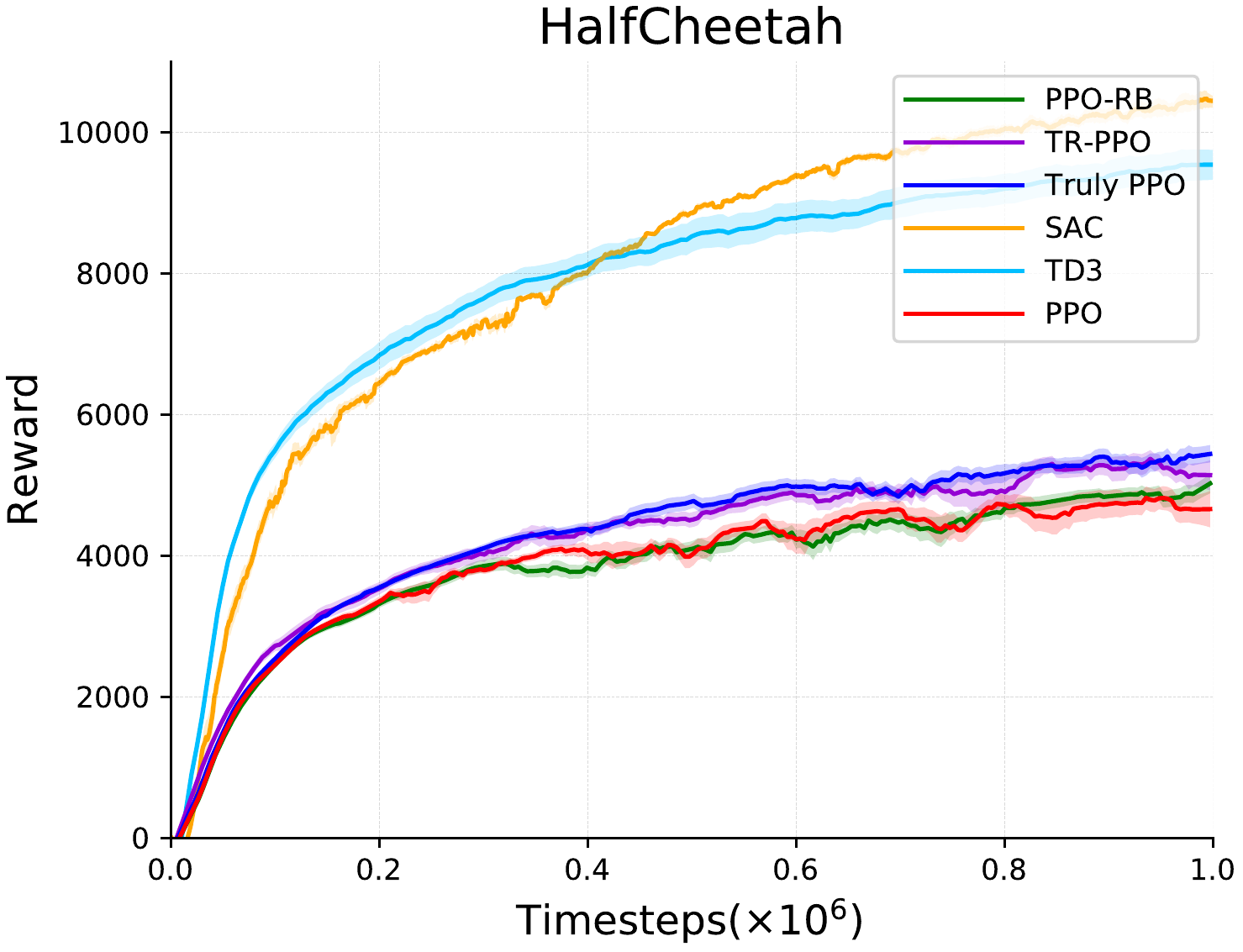}
		}
		\iffastcompile
			\caption{
			}\label{fig_rew}
		\else
		    \caption{
		        Episode rewards of the policy during the training process averaged over \nrandomseed/ random seeds on Mujoco Tasks, comparing our methods with PPO and the state-of-the-art methods.
		        The shaded area depicts the mean $\pm$ half the deviation.
		    }\label{fig_rew}
		\fi
	\end{figure*}
	
	\afterpage{
		\begin{table*}[!tp]
			\newlength{\widthefficiency}
			\setlength{\widthefficiency}{0.6in}
			\setlength{\tabcolsep}{0pt} 
			\centering
			\scriptsize
				\begin{tabular}{@{}
					+m{0.7in}<{\centering}Ym{0.6in}<{\centering}
					Ym{\widthefficiency}<{\centering}Ym{\widthefficiency}<{\centering}Ym{\widthefficiency}<{\centering}Ym{\widthefficiency}<{\centering}
           Ym{\widthefficiency}<{\centering}Ym{\widthefficiency}<{\centering}
           Ym{\widthefficiency}<{\centering}Ym{\widthefficiency}<{\centering}
					@{}}
					\toprule
				\rowstyle{}& \multicolumn{9}{c}{(a) Timesteps to hit threshold ($\times 10^3$) }    \\
					\noalign{\smallskip}
           \rowstyle{\footnotesize} & Threshold  & \pmethodfallback/   & \pmethodkl/   & \pmethodhybrid/ & PPO & SAC & TD3 & TRPO & PPO-penalty 
                     \\
					\midrule 
		        Humanoid &  5000 &  5059 &          5373 &          5920 &  7241 &  \textbf{343} &          465 &  6498 &  13096 \\
		         Reacher &    -5 &   203 &           157 &           183 &   178 &           265 &  \textbf{35} &    70 &    301 \\
		         Swimmer &    90 &   374 &  \textbf{276} &           411 &   564 &             /\footnotemark &            / &   331 &    507 \\
		     HalfCheetah &  3000 &   152 &           128 &           133 &   148 &            53 &  \textbf{41} &     / &    220 \\
		          Hopper &  3000 &   240 &           198 &  \textbf{166} &   267 &           209 &          211 &   185 &    188 \\
		        Walker2d &  3000 &   337 &           362 &  \textbf{244} &   454 &           610 &          380 &   350 &    393 \\
					\bottomrule \\
					\toprule 
					\rowstyle{} & \multicolumn{9}{c} {(b) Averaged rewards}   \\
           			\rowstyle{\footnotesize} &   & \pmethodfallback/   & \pmethodkl/   & \pmethodhybrid/ & PPO & SAC & TD3 & TRPO & PPO-penalty \\
						\midrule
							Humanoid &   &  \textbf{7344.4} &  6511.3 &           6856.7 &  6620.9 &           6535.9 &         7182.1 &  4254.4 &  3612.3 \\
					  Reacher &   &             -7.8 &    -6.4 &             -6.4 &    -6.7 &            -17.2 &  \textbf{-3.3} &      -6 &    -6.8 \\
					  Swimmer &   &   \textbf{116.1} &    98.6 &             94.1 &   100.1 &               49 &           65.4 &   107.2 &    94.1 \\
											 HalfCheetah &   &           4617.8 &  5047.3 &           5158.7 &  4600.2 &  \textbf{9987.1} &         9191.5 &  1840.3 &  4868.3 \\
				   Hopper &   &             3014 &  2963.4 &  \textbf{3263.7} &  2848.9 &           3020.7 &         3256.1 &  2757.2 &  3018.7 \\
						 Walker2d &   &           3849.9 &  3635.4 &  \textbf{4084.7} &  3276.2 &             2570 &         3721.1 &  3431.7 &    3524 \\
					\bottomrule
				\end{tabular}
				\normalsize
				\caption{
				a) Timesteps to hit thresholds within 1 million timesteps {(except Humanoid with 20 million)}. b) Averaged rewards over last 40\% episodes during training process. 
				}\label{tab_reward_hit}
		\end{table*}
		\footnotetext{The method did not reach the reward threshold within the required timesteps on all the seeds.}
	}






\begin{question}
How well do the likelihood \ratiobased/ methods perform compared to the \klbased/ ones?
\end{question}

We then consider two groups of comparisons: 
(1) {PPO vs. \pmethodkl/;
(2) \pmethodfallback/ vs. \pmethodhybrid/}.
The only difference within each group is whether the method uses constraint by ratio-based metric or KL divergence-based (trust region-based) one.


The results imply that the \klbased/ methods perform more sample efficient than the \ratiobased/ ones, and they usually obtain a better policy on most of the tasks.
As listed in \Cref{tab_reward_hit} (a), \pmethodkl/ learns faster than PPO does on almost all the 6 Mujoco tasks. 
For example, \pmethodkl/ requires almost half of the episodes than PPO does on Swimmer, and the reductions are more than 30\% on Hopper and Walker2d.
Similar results can be seen in the comparison of \pmethodhybrid/ with \pmethodfallback/.
Besides, \pmethodkl/ and \pmethodhybrid/ achieve much higher reward than PPO and \pmethodfallback/ do on 4 of the 6 Mujoco tasks (see \Cref{tab_reward_hit} (b) and \Cref{fig_rew}) and 5 of 6 Atari tasks (see \Cref{fig_rew_atari}). 

As we have discussed in \Cref{sec_related_policymetric}, the constraints with different metrics have different preferences on the actions, leading to the unusual algorithmic behavior.
For the ratio-based constraint, the larger likelihood of the action is allowed to update more, making the policy be less random and explore less.
While the trust region-based constraint does not has such bias and usually could explore more.
To show this, we plot the entropies during the training process in \Cref{fig_entropy}.
The entropies of \klbased/ methods tend to be remarkably larger than those of \ratiobased/ on all tasks.
These results confirm the effect on learning behavior with different policy metrics.


		\begin{figure*}[!t]
		 \centering
		 	\def\widthjjduywhe{0.33}
			\centerline{
				\includegraphics[width=\widthjjduywhe\linewidth]{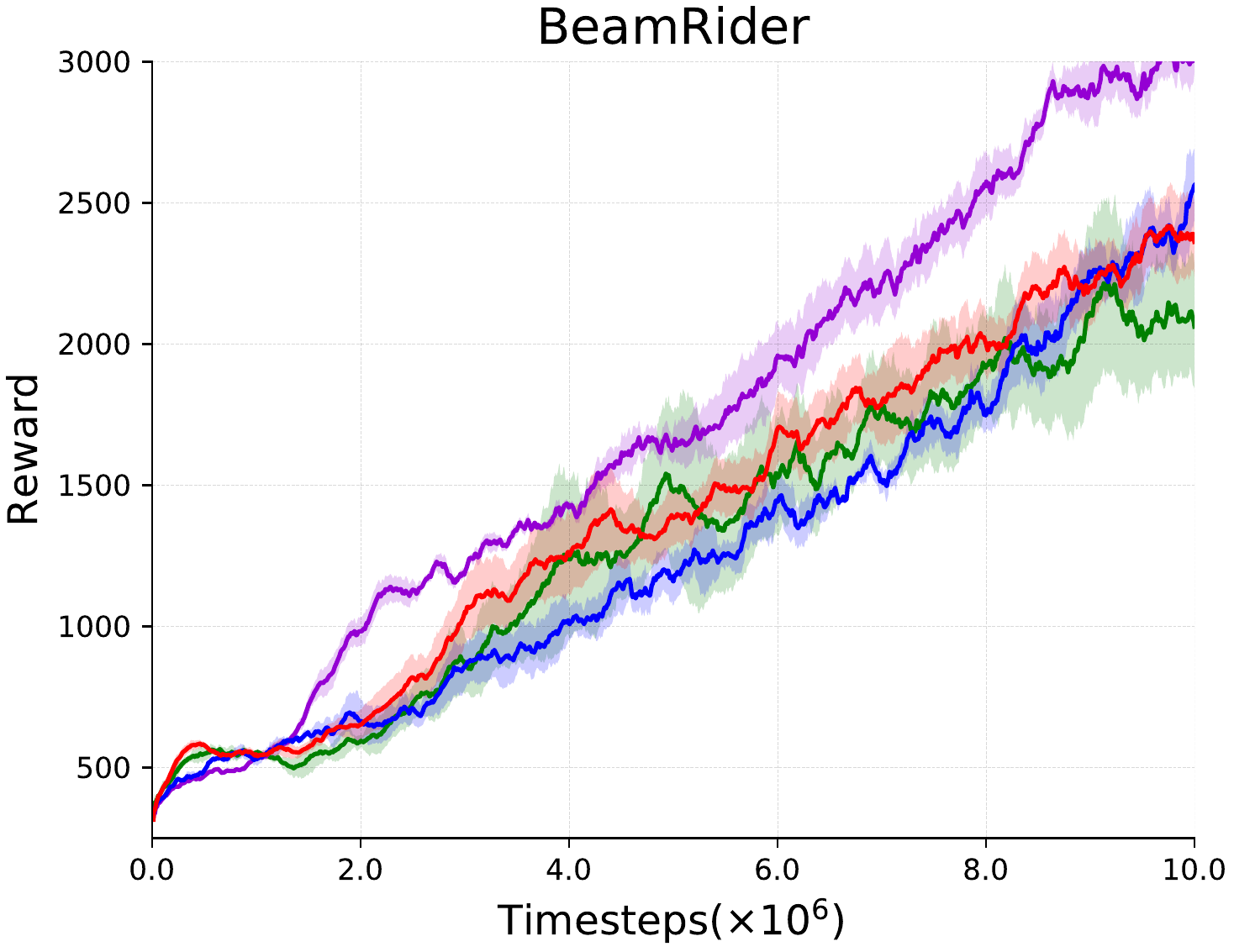}
				\includegraphics[width=\widthjjduywhe\linewidth]{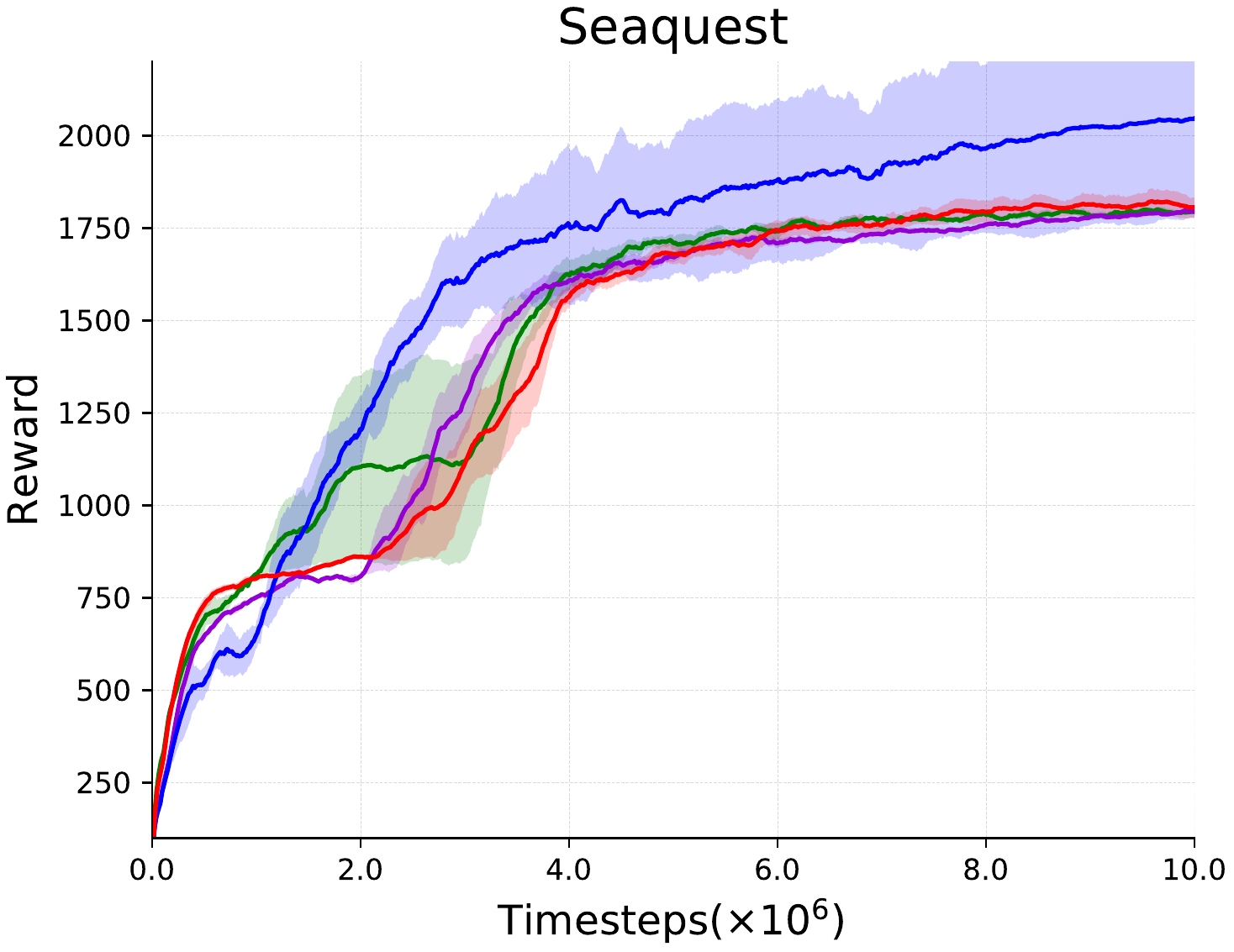}
				\includegraphics[width=\widthjjduywhe\linewidth]{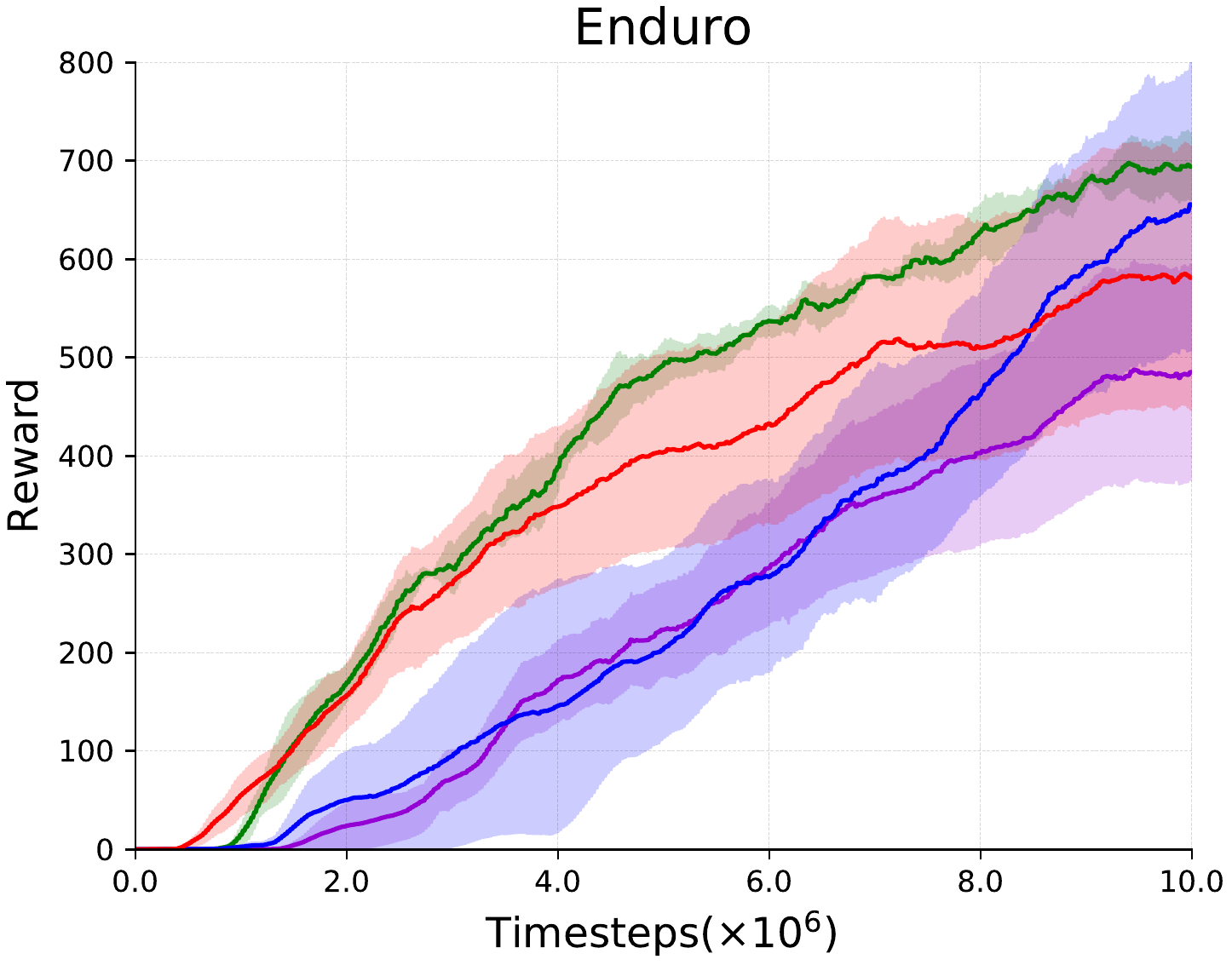}
			}
			\centerline{
				\includegraphics[width=\widthjjduywhe\linewidth]{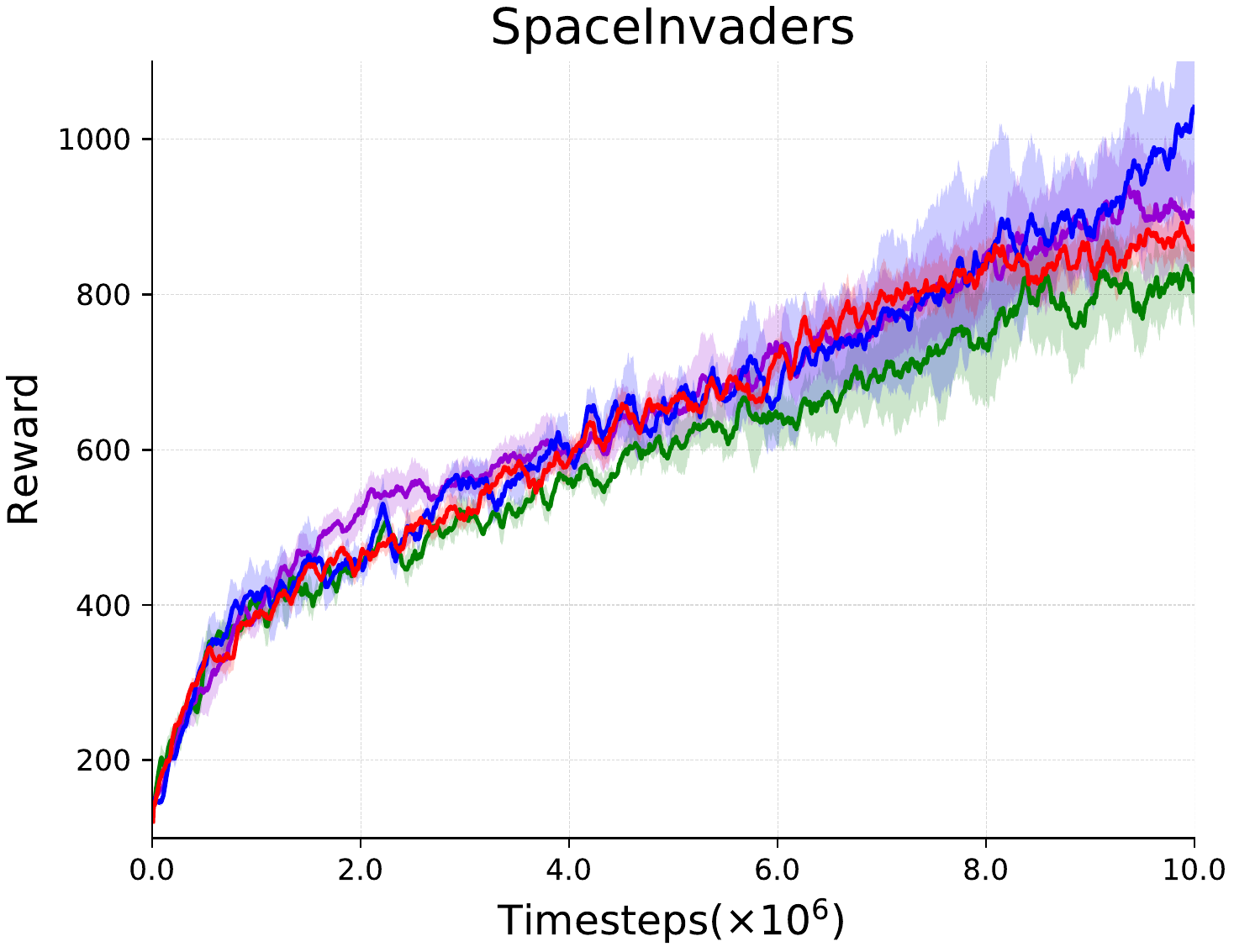}
				\includegraphics[width=\widthjjduywhe\linewidth]{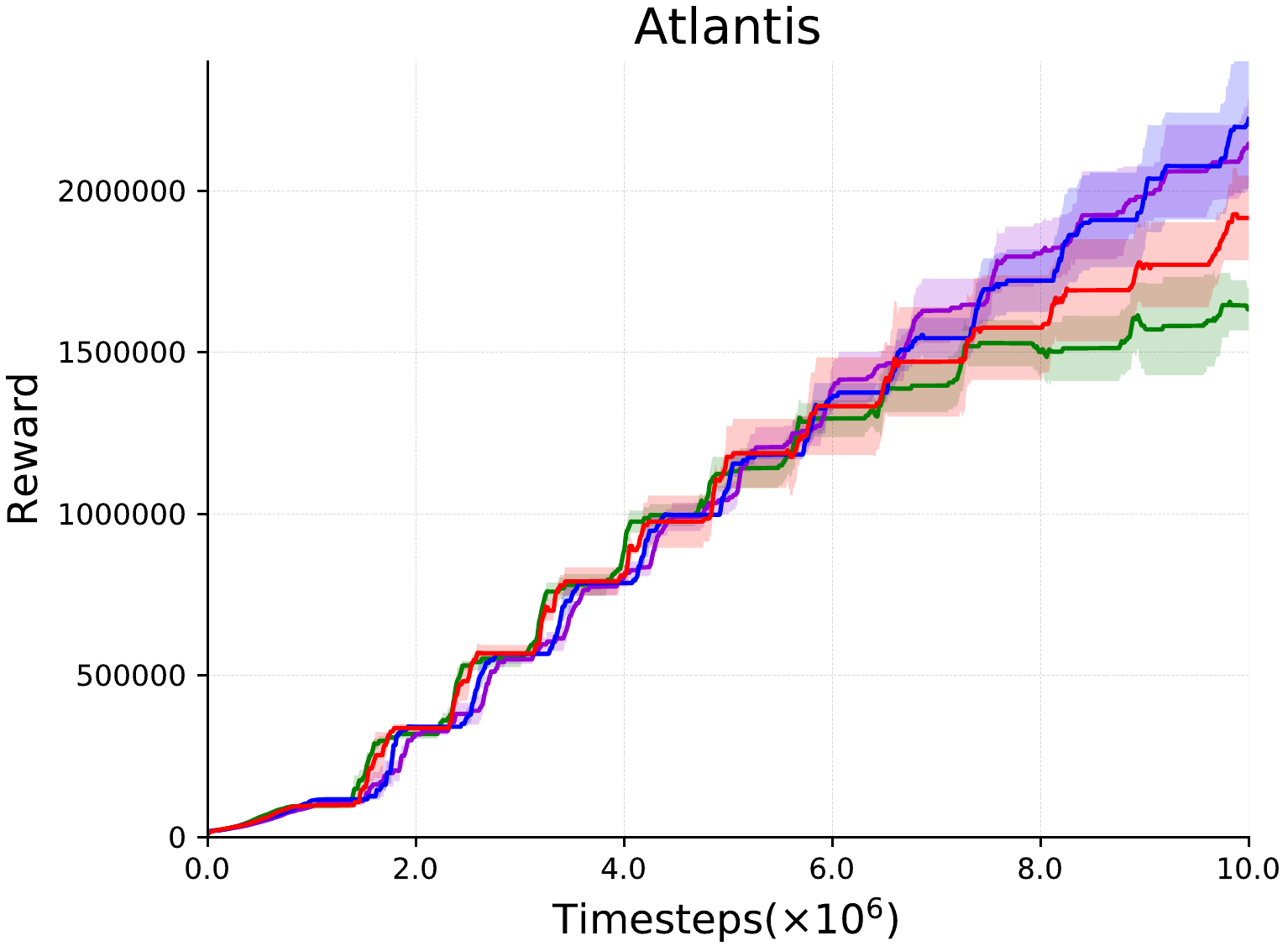}
				\includegraphics[width=\widthjjduywhe\linewidth]{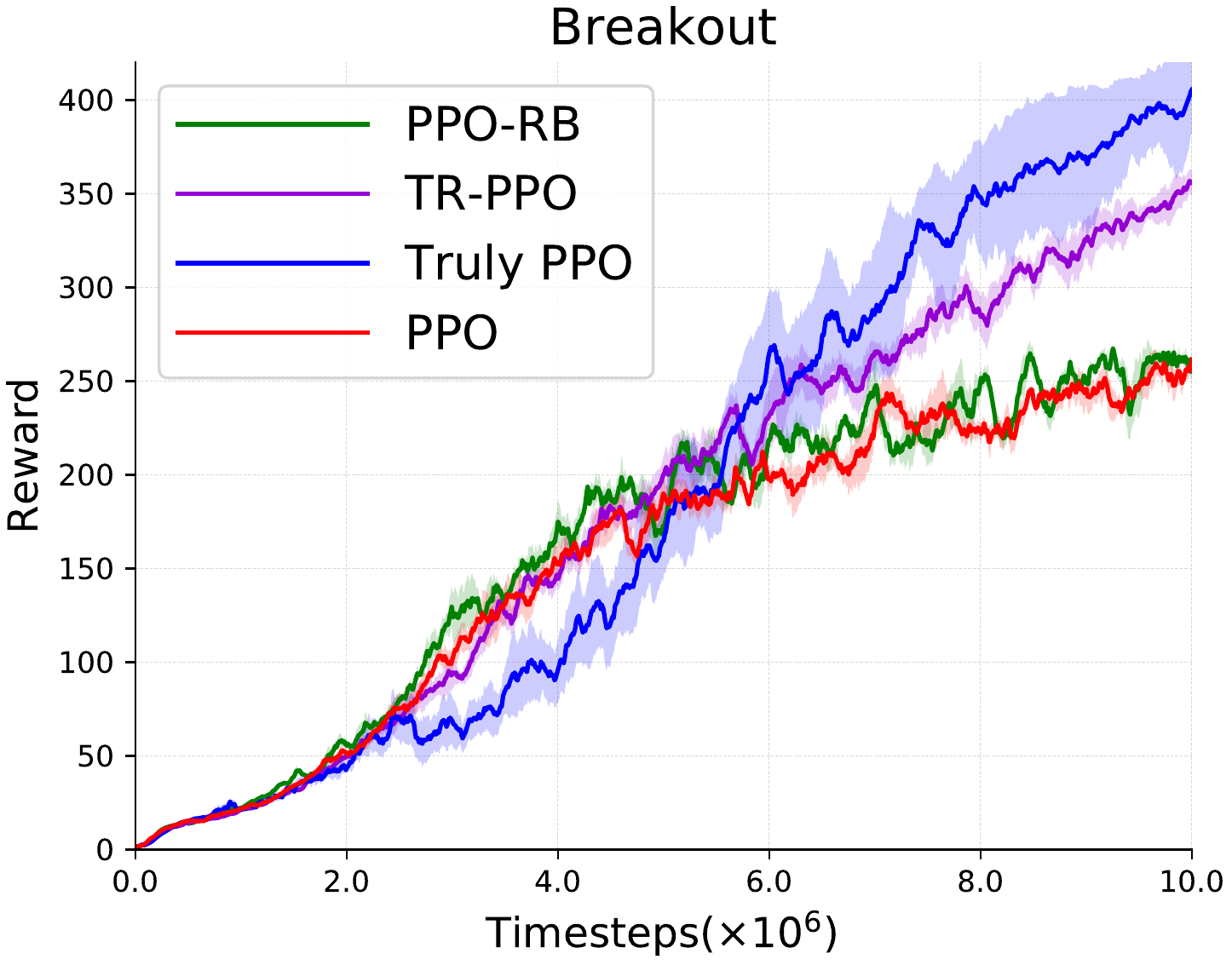}
			}
		 \caption{
		     Episode rewards of the policy during the training process on Atari Tasks. 
		 }\label{fig_rew_atari}
		\end{figure*}

\begin{figure}[!t]
	\centering
  	\centerline{
  		\includegraphics[width=\widthproperty\linewidth]{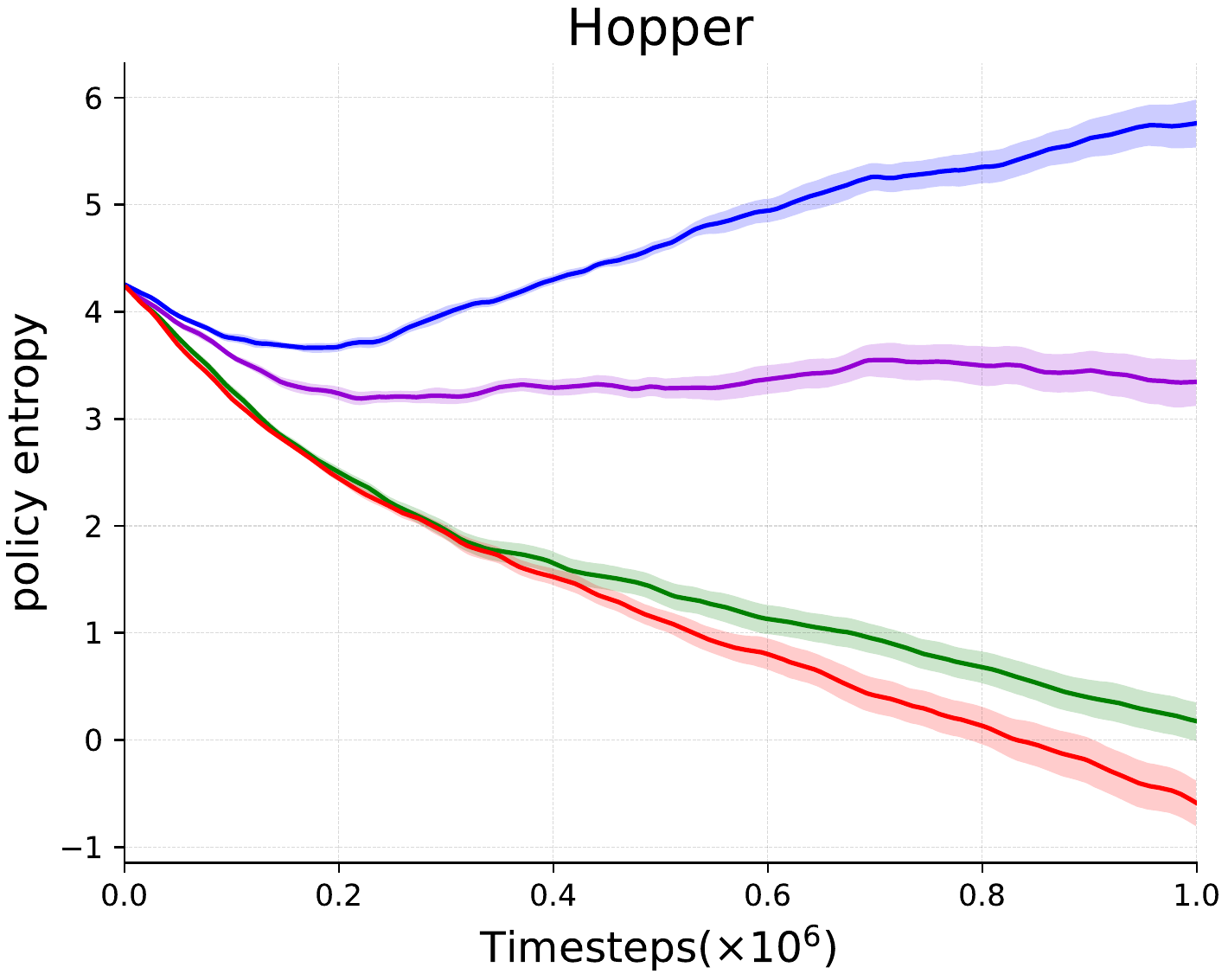}
  		\includegraphics[width=\widthproperty\linewidth]{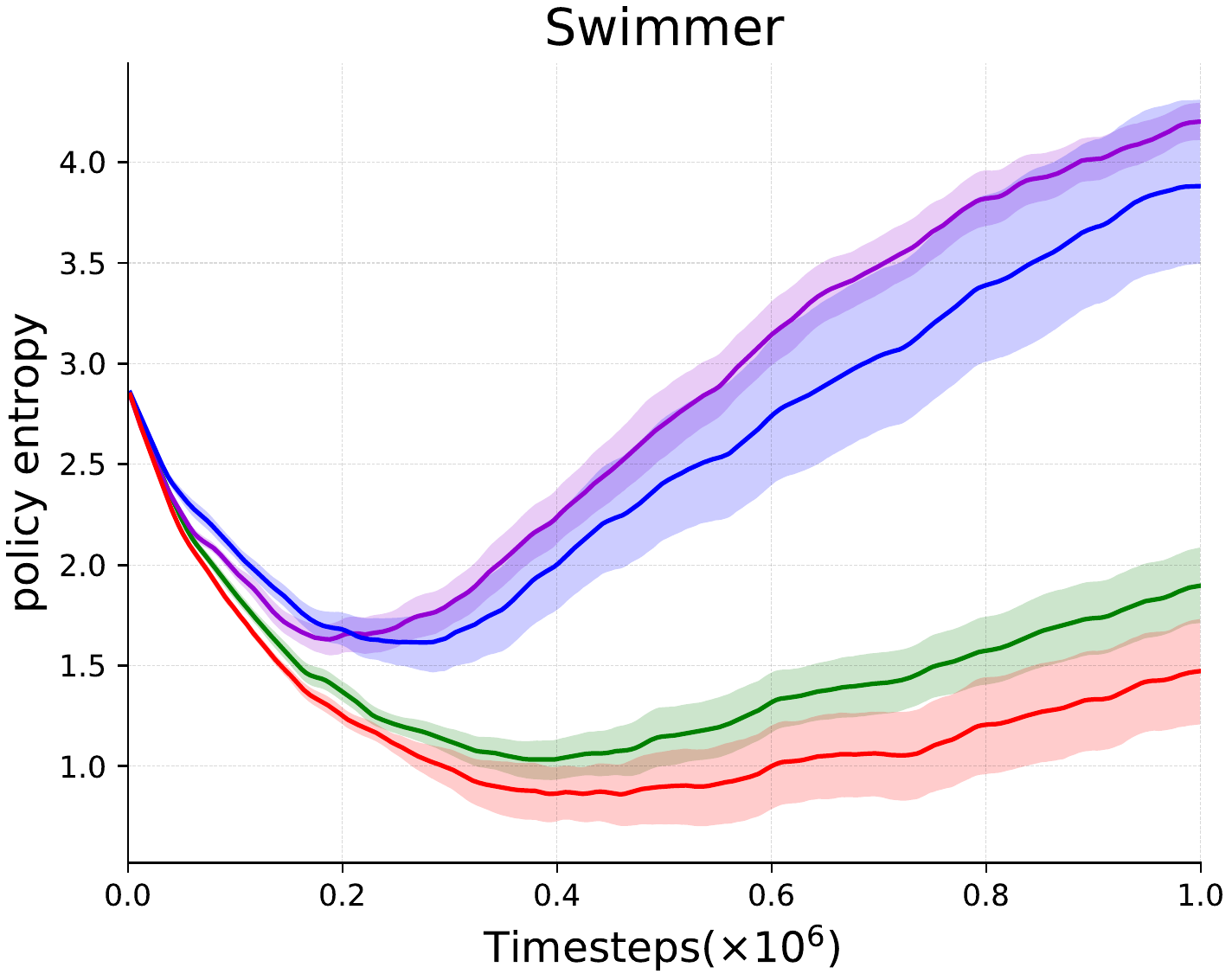}
 		\includegraphics[width=\widthproperty\linewidth]{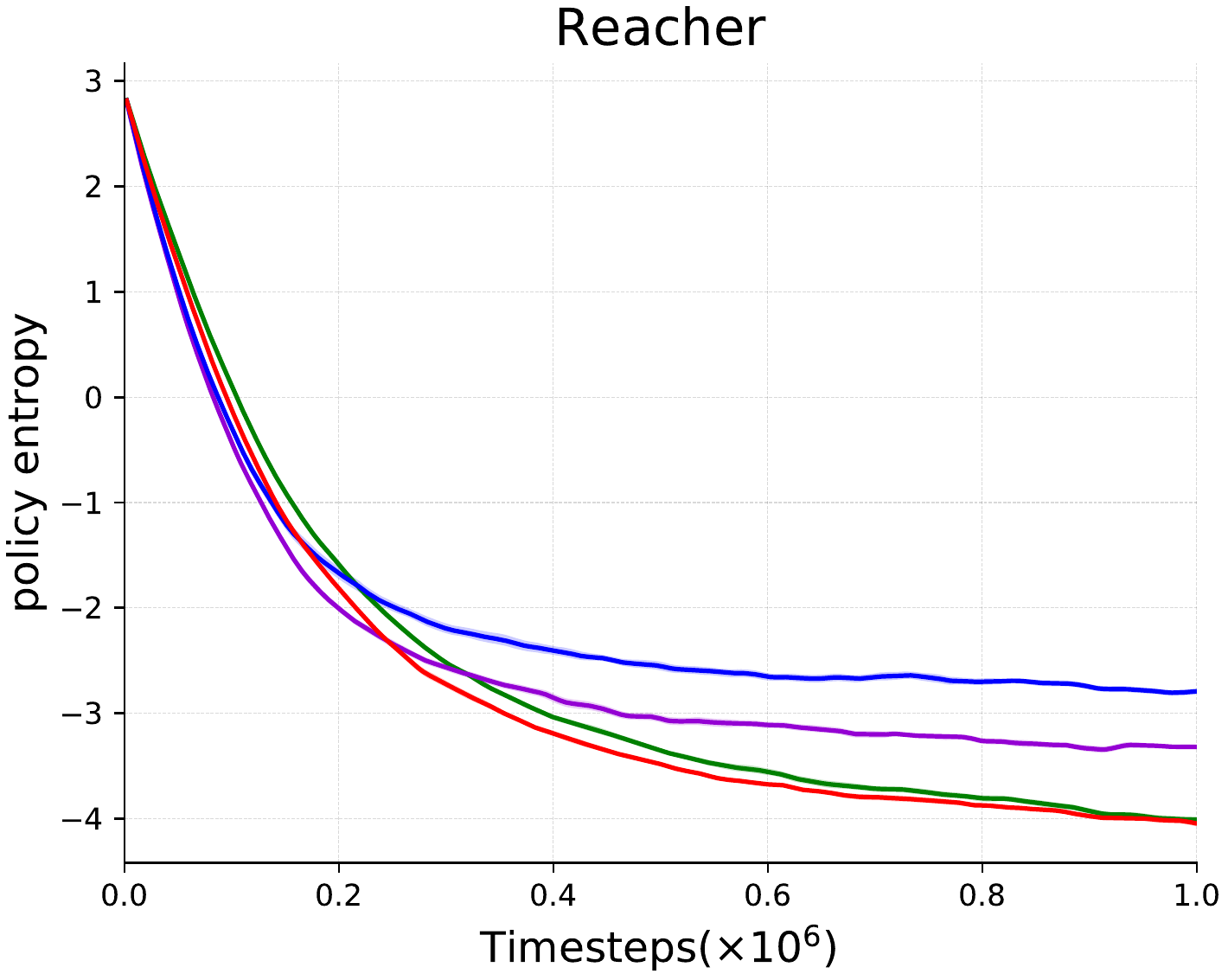}
  		\includegraphics[width=\widthproperty\linewidth]{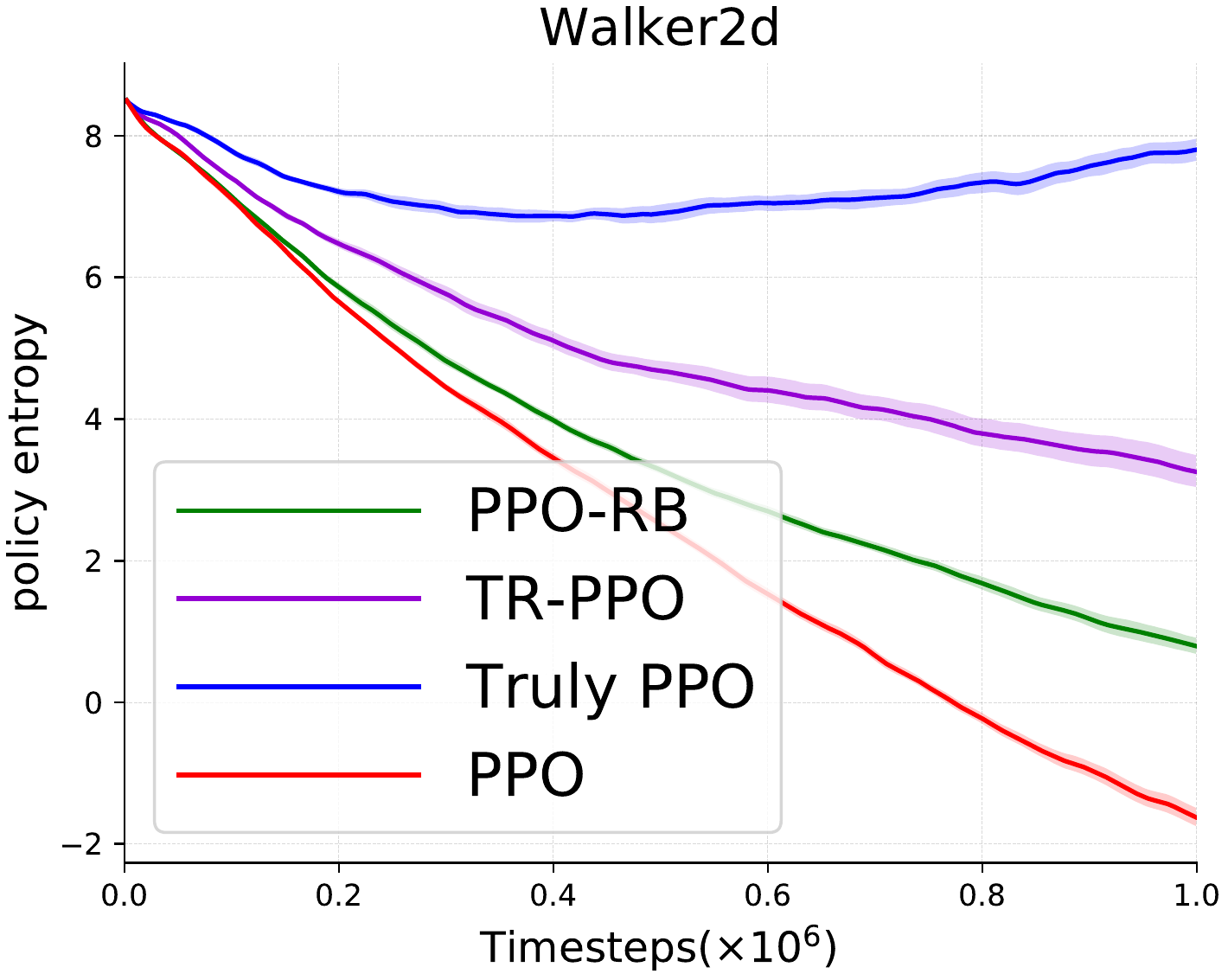}
  	}
	\caption[width=\columnwidth]{
		Policy entropy during the training process. 
	}\label{fig_entropy}
\end{figure}

	\begin{figure}[!t]
  		\centering
		\centerline{
			\includegraphics[width=0.33\linewidth]{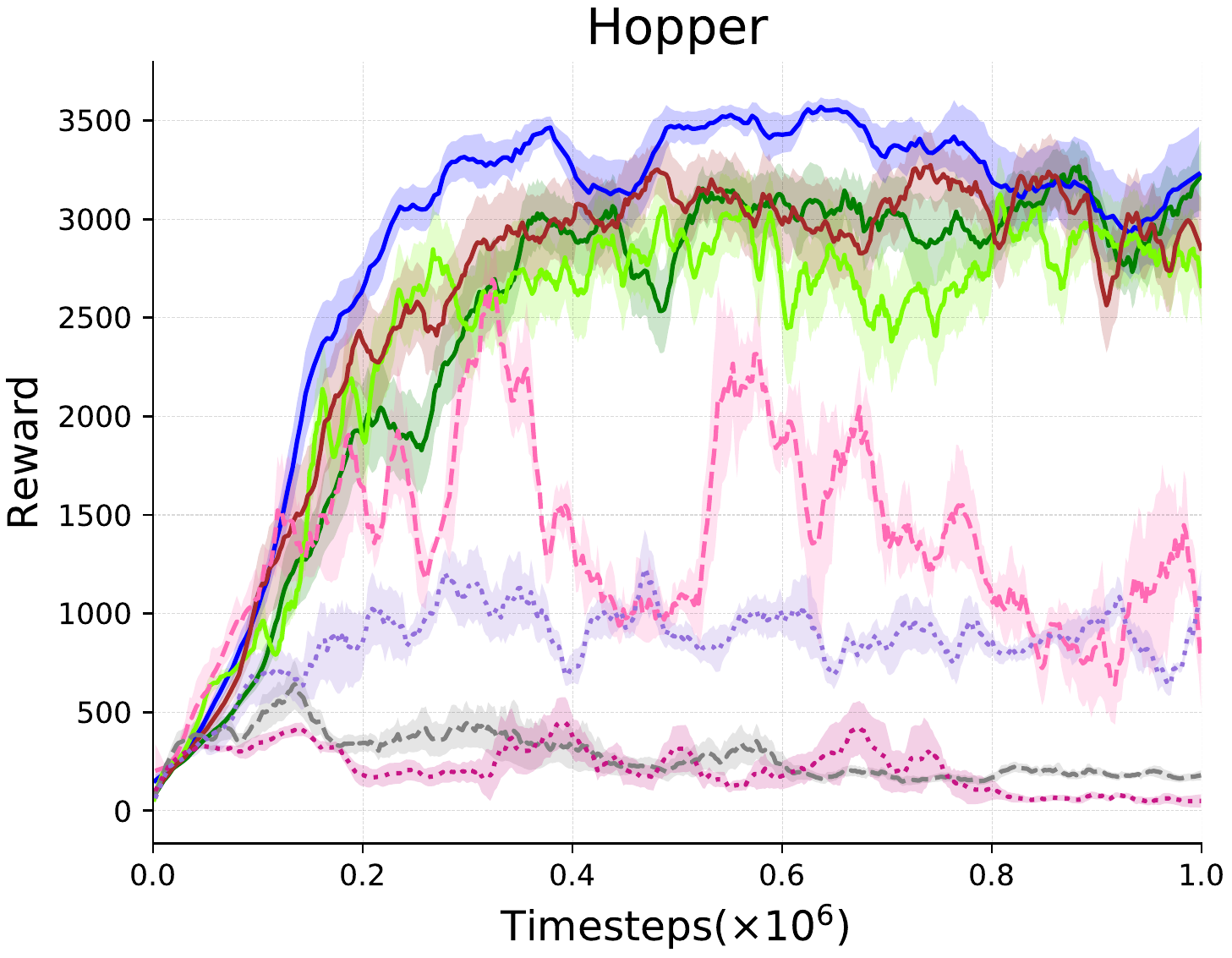}
			\includegraphics[width=0.33\linewidth]{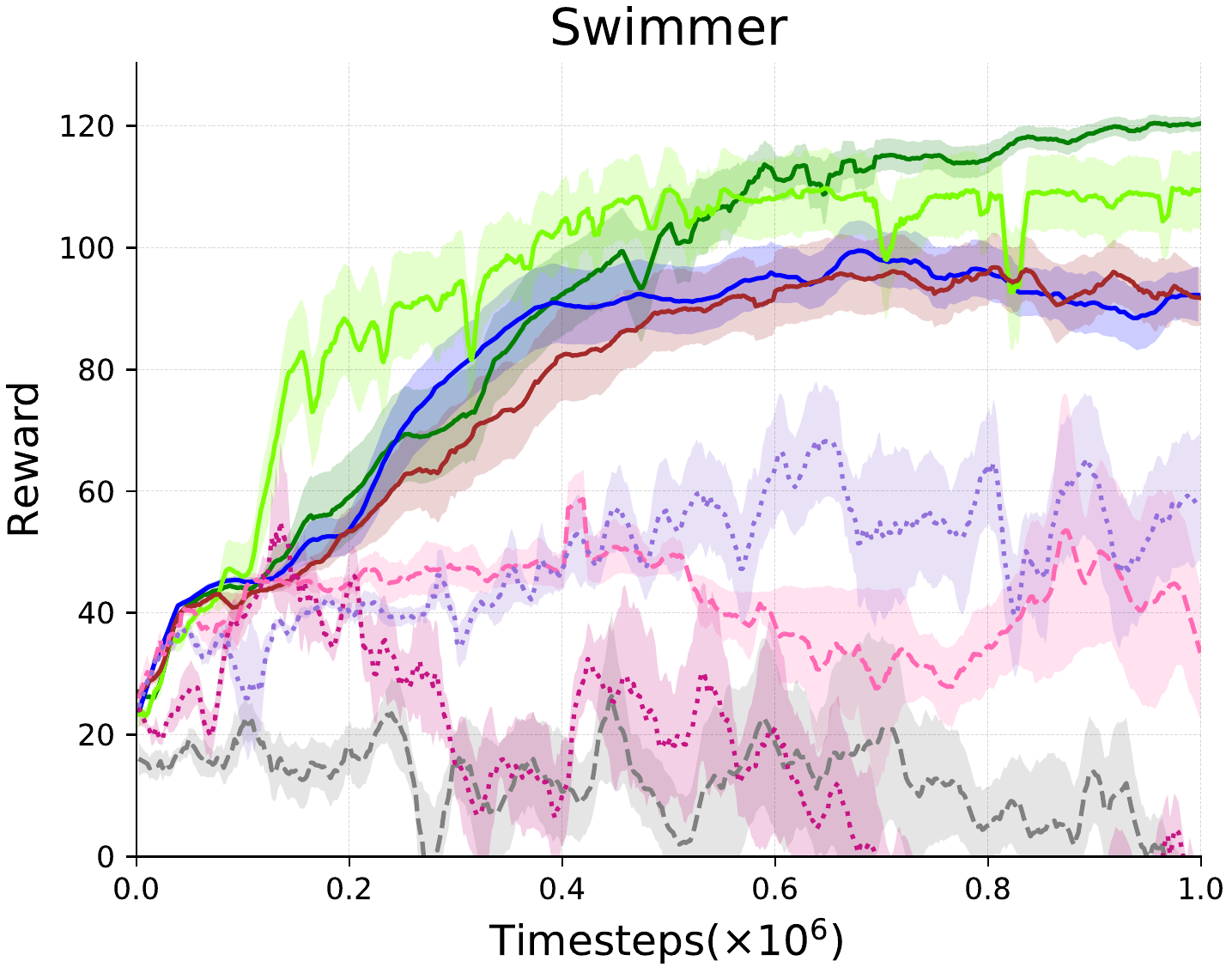}
			\includegraphics[width=0.33\linewidth]{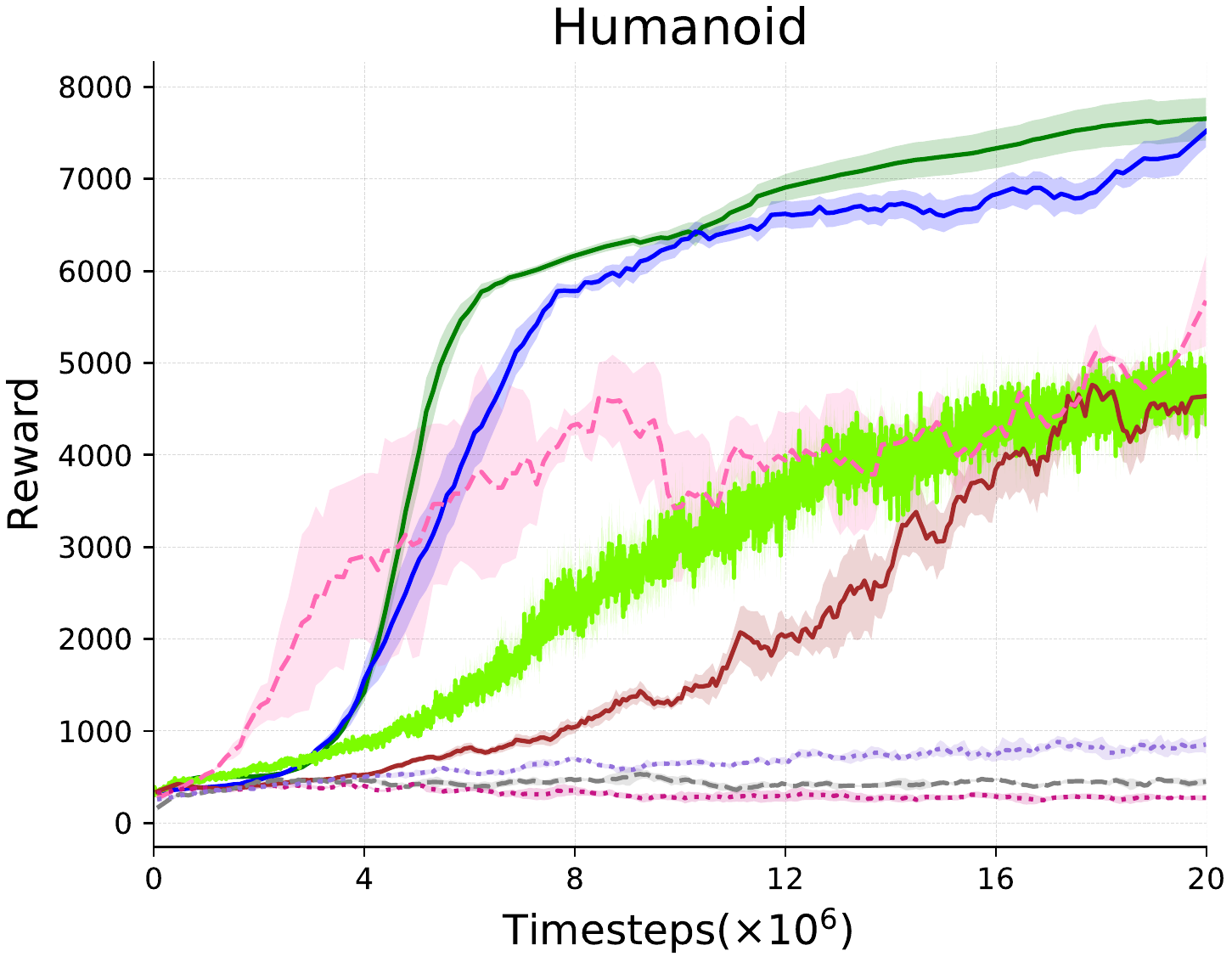}
		}
		\centerline{
			\includegraphics[width=0.33\linewidth]{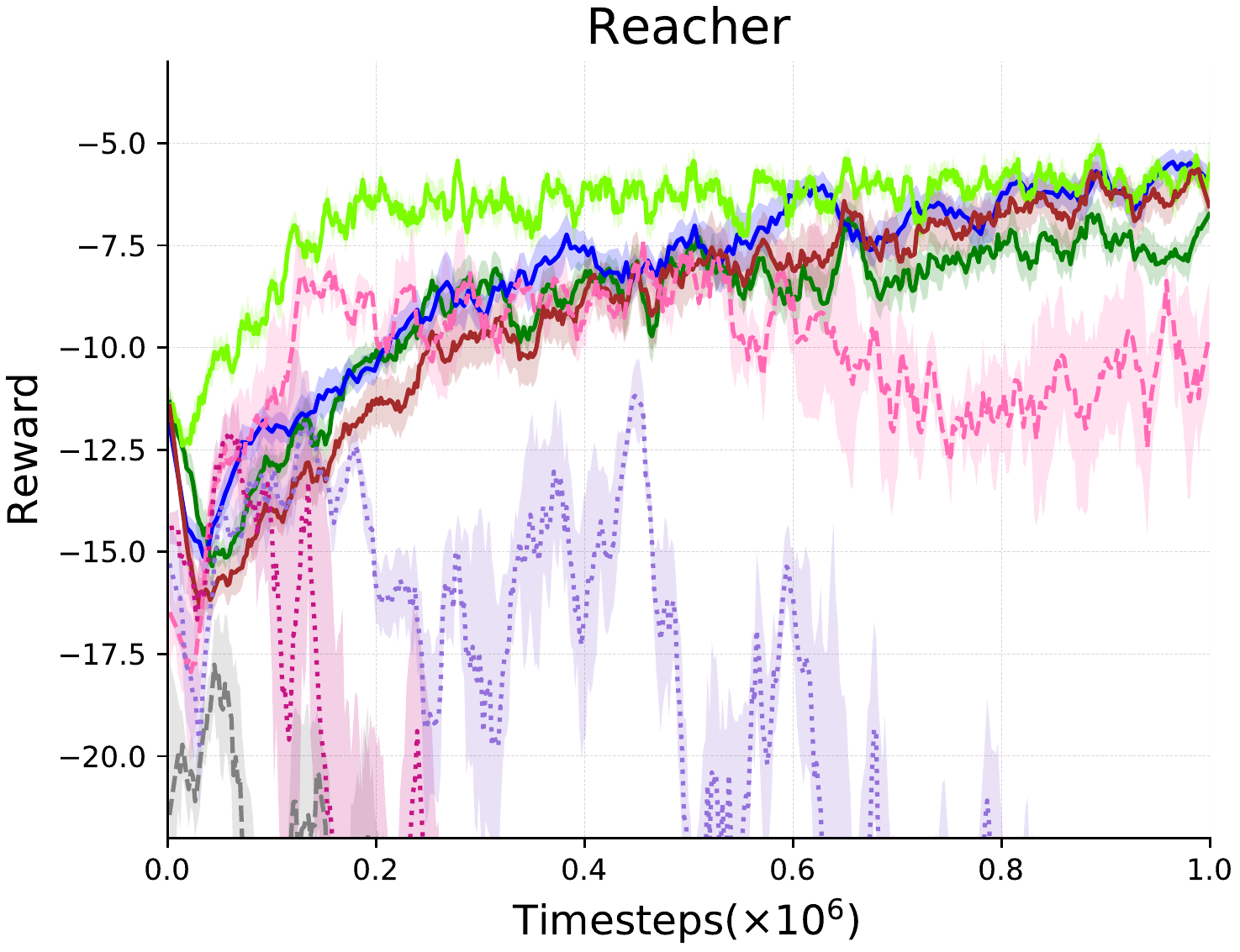}
			\includegraphics[width=0.33\linewidth]{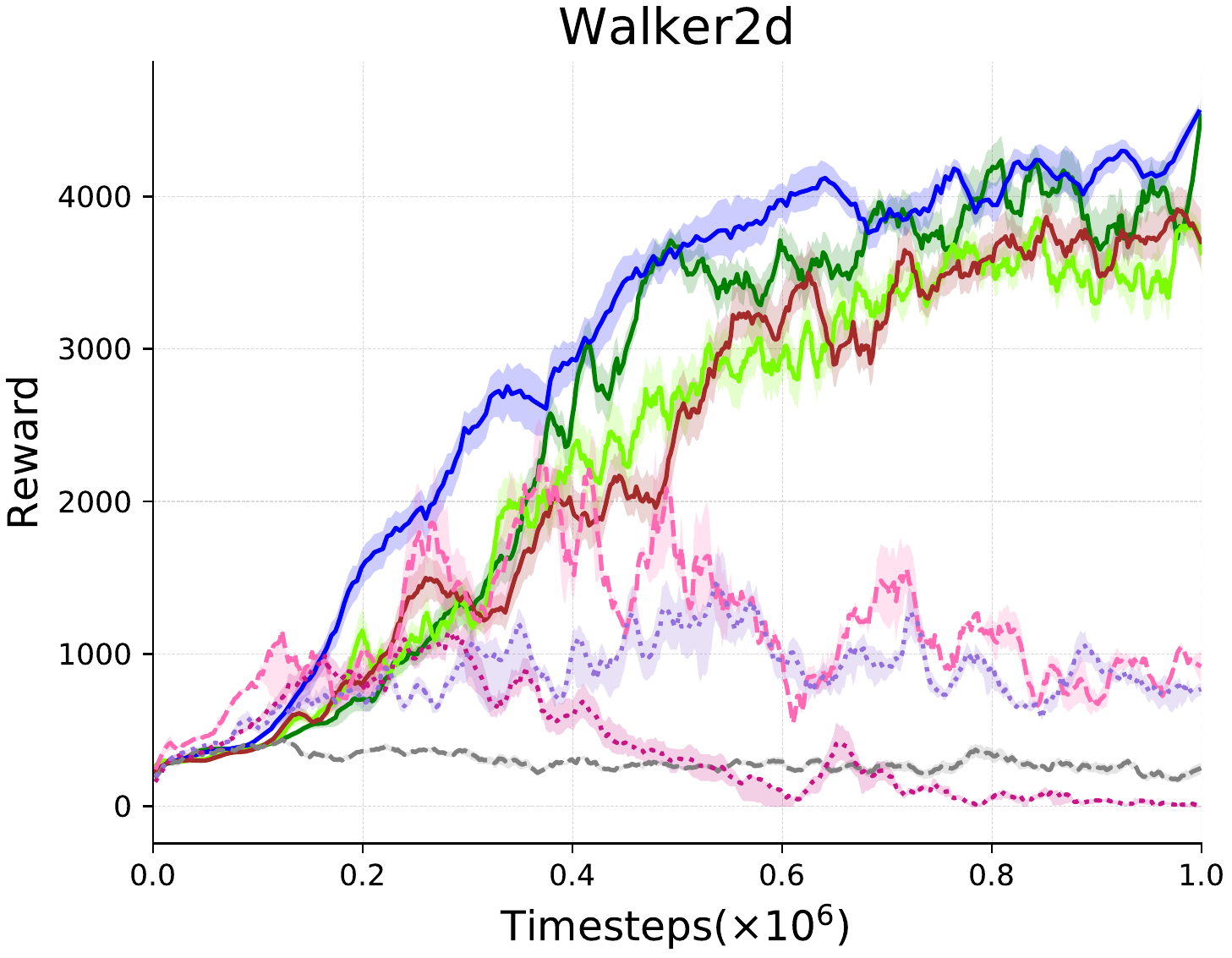}
			\includegraphics[width=0.33\linewidth]{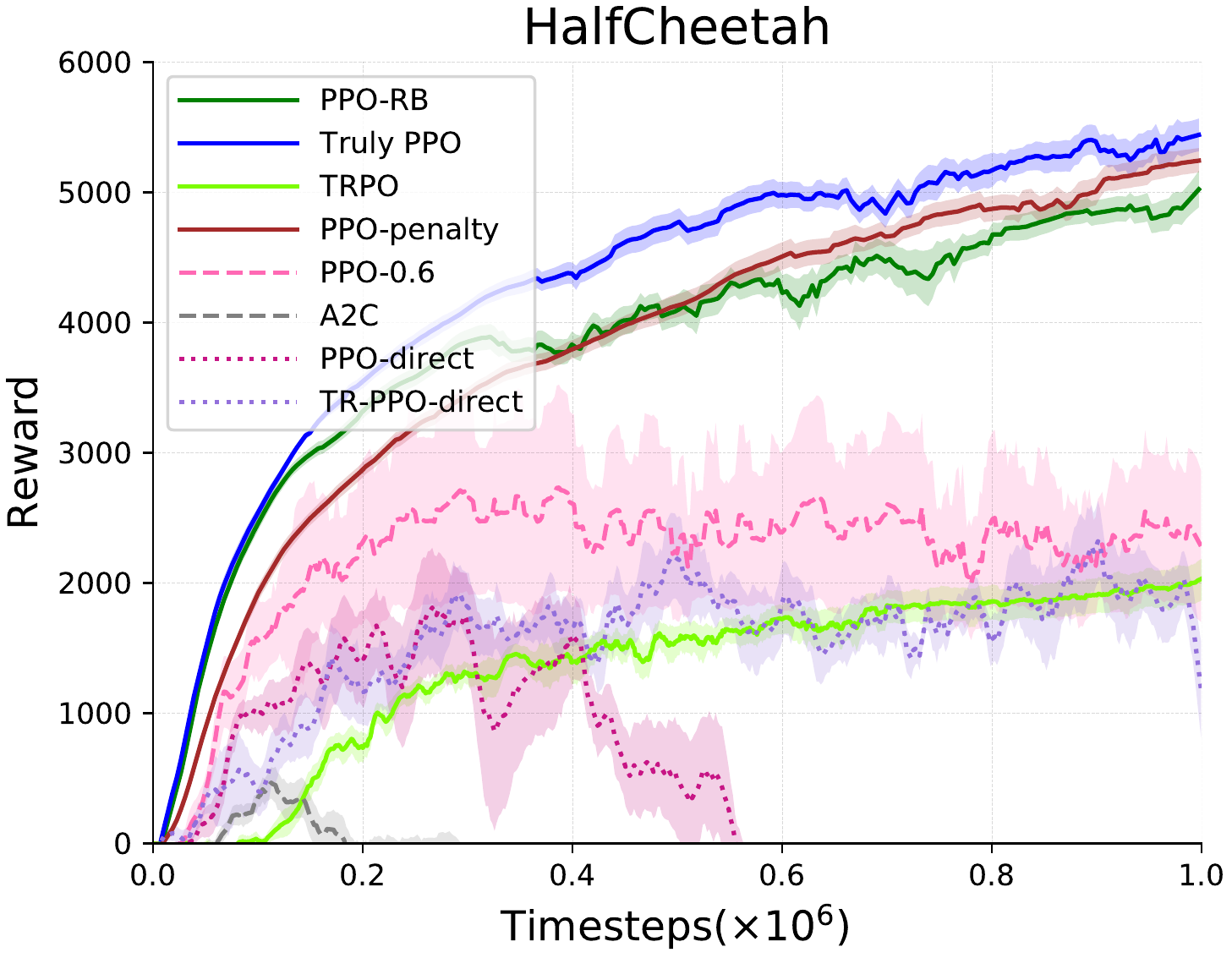}
		}
	  \caption{
		  Comparison with several baseline methods.
	      The learning curves show the episode rewards of the policy during the training process on Mujoco Tasks, 
	  }\label{fig_rew_validate}
	\end{figure}

\subsection{Comparison with Other ``Proximal'' Approaches}
We also compare our variants with TRPO and PPO-penalty which also attempts to enforce proximal constraint.
As shown in \Cref{fig_rew_validate}, our methods achieve higher episode rewards than TRPO on 5 tasks.
TRPO performs better on low dimensional tasks like Reacher and Swimmer (with $dim({\cal A})=2$ ), whereas it performs much worse on high-dimensional tasks, especially on Humanoid (with $dim({\cal A})=17$).
One reason is that the second-order optimization involved in TRPO could be inaccurate, especially on the high-dimensional action space tasks. Such inaccuracy may result in a worse solution.
Compared to PPO-penalty, our methods outperform it on all the tasks, which confirms the superiority of the clipping technique that it is more robust across different tasks and requires relatively less effort on hyperparameter tuning.

\begin{figure}[!t]
	\centering
  	\centerline{
  		\includegraphics[width=\widthproperty\linewidth]{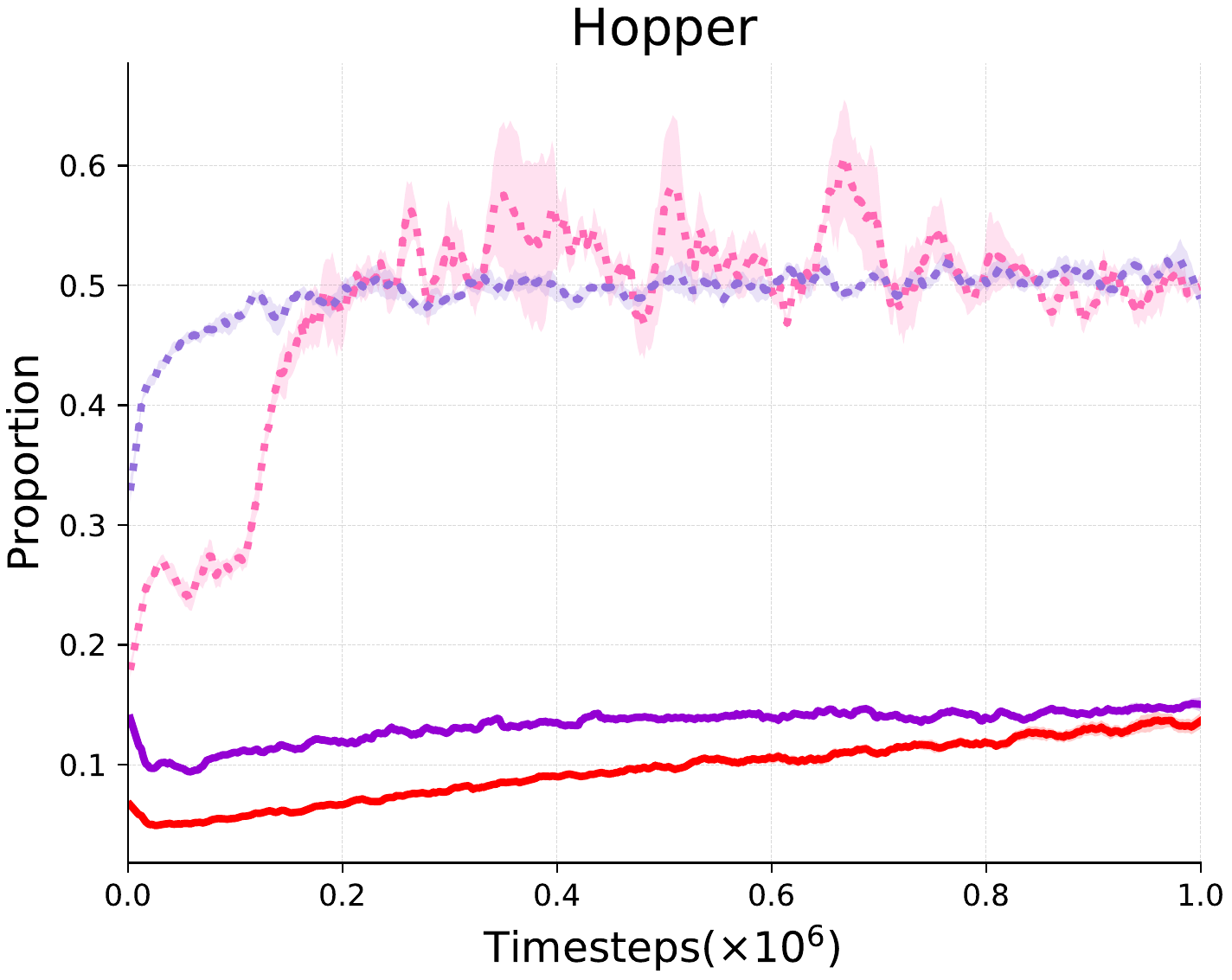}
  		\includegraphics[width=\widthproperty\linewidth]{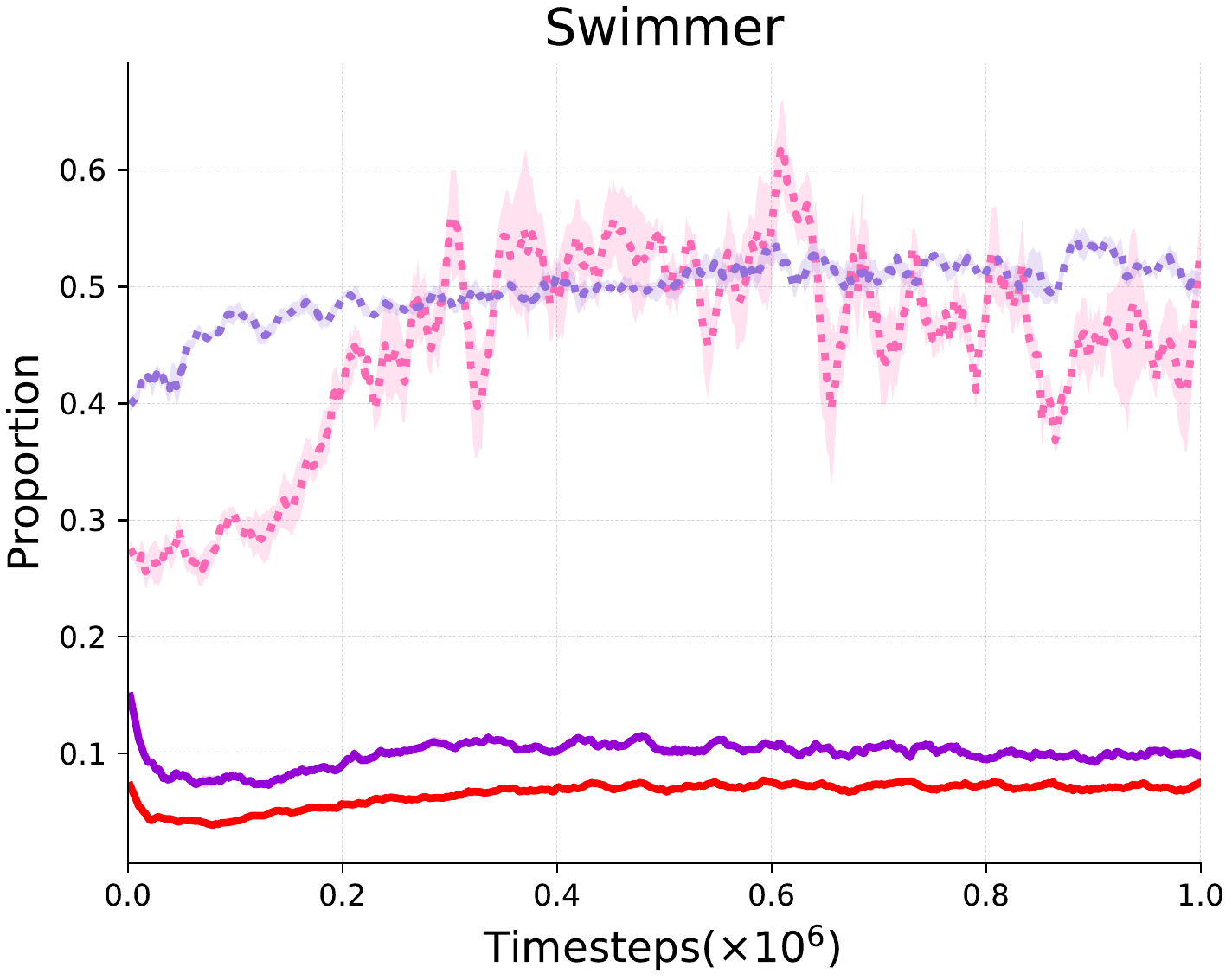}
  		\includegraphics[width=\widthproperty\linewidth]{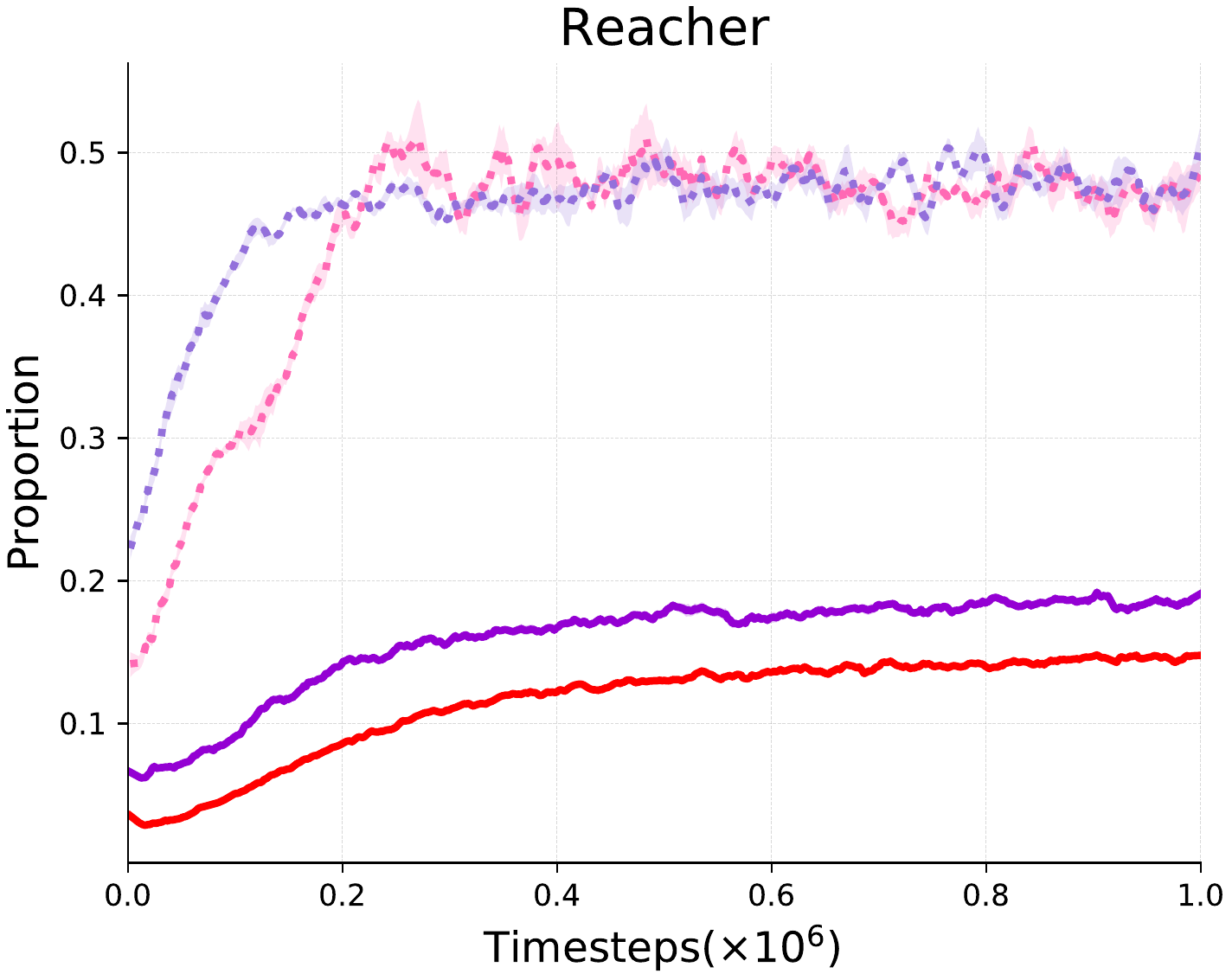}
  		\includegraphics[width=\widthproperty\linewidth]{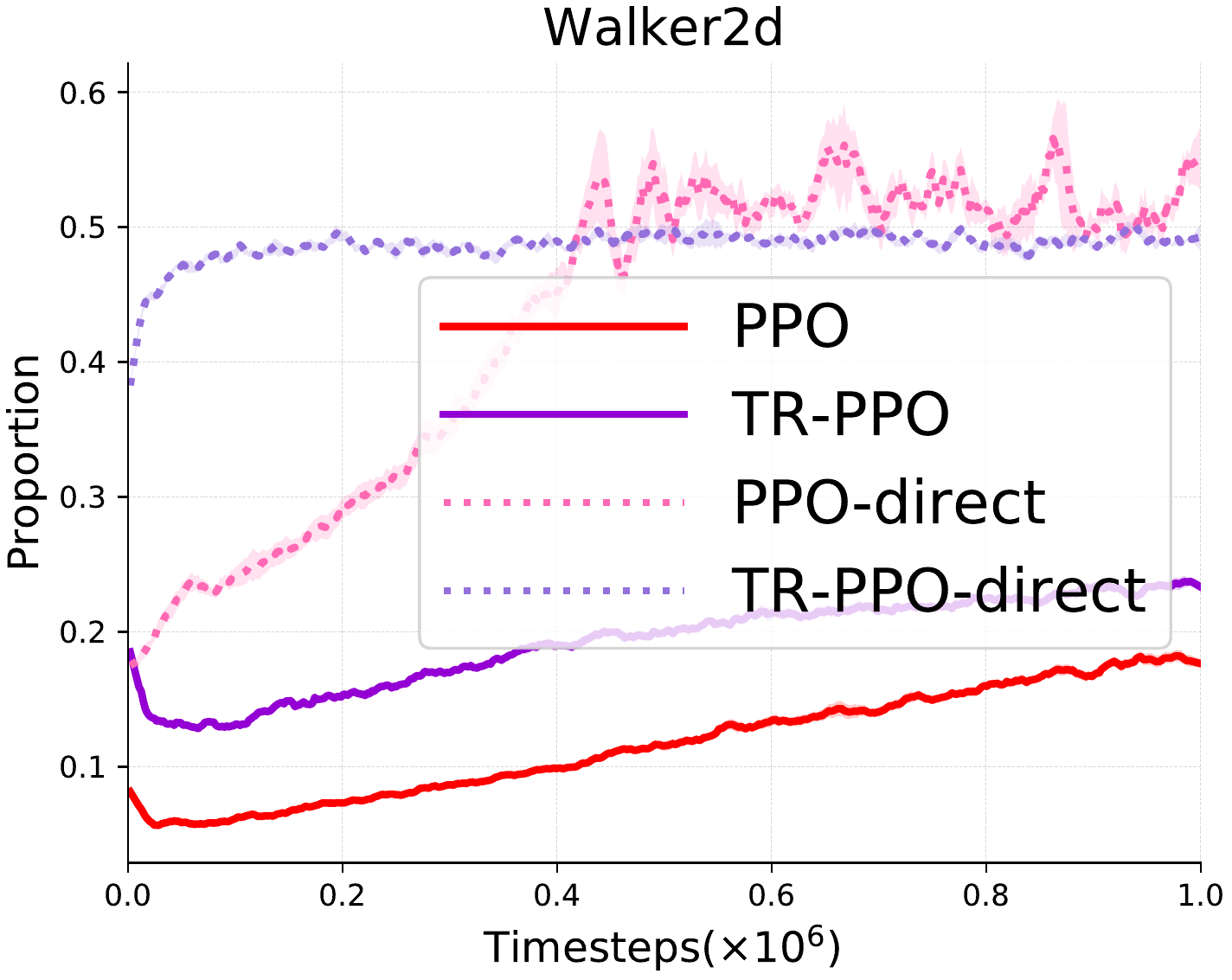}
  	}

	\iffastcompile
		\caption{
		}\label{fig_fraction_bad_solution}
	\else
		\caption[width=\columnwidth]{
		    The proportions of the objectives which are not improved but the policy satisfies the clipping condition. 
		    Formally, the following conditions are considered: 
		    for PPO and \pmethodratiosimple/, $\ratioa[t][\rm new] \adv[t] < \ratioa[t][\rm old] \adv[t]$ and $|\ratioa[t][\rm new]-1|\geq \epsilon$;
		    for \pmethodkl/ and \pmethodklsimple/, $\ratioa[t][\rm new] \adv[t] < \ratioa[t][\rm old] \adv[t]$ and $D_{\rm KL}^{s_t}(\theta\oldsub, \theta\newp)\geq \delta$.
		}\label{fig_fraction_bad_solution}
	\fi
\end{figure}

\subsection{{The Minimization Operation}}\label{sec_experiment_min}
We have discussed the importance of the minimization operation (the improvement condition) of PPO in \Cref{sec_PPO} (see \eqrefnop[eq_L_PPO_expectation] and \eqrefnop[eq_L_PPO_case_general]). 
We have stated that without the minimization operation, the likelihood ratios may be trapped in the ones which make the surrogate objective worse.
To validate our speculation, we evaluate two variants for PPO, which remove minimization operation for PPO and \pmethodkl/, named \emph{\pmethodratiosimple/} and \emph{\pmethodklsimple/}, respectively.
As shown in \Cref{fig_rew_validate}, both these two methods fail on all the tasks (dotted line).
We have stated that the minimization operation provide remedy for the case when the objective is not improved while violating the constraint, e.g.,  $\ratioa[t][\rm new] \adv[t] < \ratioa[t][\rm old] \adv[t]$ and $|\ratioa[t][\rm new]-1|\geq \epsilon$ ( $D_{\rm KL}^{s_t}(\theta\oldsub, \theta\newsub) \geq \delta$ ).
We make statistics of these samples.
As can be seen in \Cref{fig_fraction_bad_solution}, with the minimization operation, the proportions of such un-improved samples are significantly reduced.

\subsection{Comparison with the State-of-the-art Methods }
We compare our methods with two state-of-the-art methods:  Soft Actor Critic (SAC) \citep{haarnoja2018soft} and Twin Delayed Deep Deterministic policy gradient (TD3) \citep{fujimoto18td3}.
We use their implementations by the authors, respectively.
In general, our methods perform stably across different tasks. 
As \Cref{fig_rew} shows and \Cref{tab_reward_hit} lists, our methods outperform SAC on 5 tasks except for HalfCheetah.
Compared to TD3, our methods perform much better on Swimmer, and are slightly better on Hopper, Walker2d and Humanoid, while performing worse on Reacher and HalfCheetah.
On Humanoid task, our methods are not as sample efficient as SAC and TD3 but achieves higher final reward. This may due to that our methods are on-policy algorithms while SAC and TD3 are off-policy ones.
Furthermore, our methods require relatively less effort on hyperparameter tuning, as we use the same hyperparameter across most of the tasks.
In addition, the variants of PPO are much more computationally efficient than SAC. 
{As can be seen in \Cref{tab_training_time}, all the variants of PPO require almost 1/10 and 1/50 time of TD3 with 1 million and 20 million timesteps, respectively.}

\begin{table}[!t]
\centering
	\setlength{\tabcolsep}{3pt}
	
	\begin{tabular}{c|ccc}
	\toprule
	   Task & Variants of PPO  & SAC & TD3 \\
	\midrule
	$1\times10^6$ timesteps & 24 min & 182 min & 238 min  \\ 
	$20\times10^6$ timesteps (Humanoid)& 2 h & 72 h  & 108 h \\
	\bottomrule
	\end{tabular}
	\caption{
	Training wall-clock time of the algorithms.
	}\label{tab_training_time}
\end{table}


%

\section{Conclusion}\label{sec_conclusion}
Despite the effectiveness of the well-known PPO, it somehow lacks theoretical justification, and its actual optimization behaviour is less studied.
To our knowledge, this is the first work to reveal the reason {why} PPO could neither strictly bound the likelihood ratio nor enforce a well-defined trust region constraint.
Based on this observation, we proposed a trust region-based clipping objective function with a rollback operation.
The trust region-based clipping is more theoretically justified while the rollback operation could enhance its ability in restricting policy.
Both these two techniques significantly improve ability in confining policy and maintaining training stability.
Extensive results show the effectiveness of the proposed methods.

\ttt{In conclusion, our results highlight the necessity to confine the policy difference, leading to a stable improvement of the policy.
Besides, different policy metric of the constraint could result in different algorithmic behavior. We found that the KL divergence-based ones outperform the likelihood ratio-based ones. We advocate more studies on how the metric affects the underlying algorithmic behavior.
Moreover, the results show that the clipping technique could be served as an effective approach to address the problems with constraint or the one dragged with an adjusted regularization term. We hope our results may inspire approaches for solving these problems.
We hope it may inspire approaches for solving other problems with a similar structure.
Last but not least, deep RL algorithms have been notorious in its tricky implementations and require much effort to tune the hyperparameters \citep{islam2017reproducibility,henderson2018deep}.
Our proposed methods are equally simple to implement and tune as PPO but perform much better. They may be considered as useful alternatives to PPO.}

\acks{This work is partially supported by National Science Foundation of China (61976115,61672280, 61732006), AI+ Project of NUAA(56XZA18009), research project no. 6140312020413, Postgraduate Research \& Practice Innovation Program of Jiangsu Province (KYCX19\_0195).}


\newpage

\appendix


\section{Implementation Details}\label{sec_implementation_detail}
\Cref{tab_hyperparameters_mujoco} and \Cref{tab_hyperparameters_atari} list the hyperparameters used for Mujoco and Atari tasks respectively.

\begin{table}[h!]
\centering
\setlength{\tabcolsep}{0pt}
\renewcommand{\arraystretch}{1.2}
\begin{tabular}{+c - c}
\toprule
\rowstyle{\bfseries}
Hyperparameter  & Value  
\\ \hline 
%
%
coefficient of PPO & 
	\parbox{.5\linewidth}{
		\vspace{.2\baselineskip}	
		\centering
		$\epsilon=0.2$ 
		\vspace{.2\baselineskip}	
	}
\\ \hline
	coefficient of \pmethodfallback/ & 
	\parbox{.5\linewidth}{
		\vspace{.2\baselineskip}	
		\centering
		$\epsilon=0.2$ \\
		$\alpha=0.02$ (Humanoid), $0.3$ (Other) \\
		\vspace{.2\baselineskip}	
	}

\\ \hline
	coefficient of \pmethodkl/ & 
	\parbox{.5\linewidth}{
		\vspace{.2\baselineskip}	
		\centering
		$\delta=0.05$ (Humanoid), $0.035$ (Other) \\
	}
\\ \hline
	coefficient of \pmethodhybrid/ & 
	\parbox{.5\linewidth}{
		\vspace{.2\baselineskip}	
		\centering
		$\delta=0.05$ (Humanoid), $0.03$ (Other) \\
		$\alpha=5$
		\vspace{.2\baselineskip}
	}
\\ \hline

learning rate & $3\times 10^{-4}$ \\
\hline
number of parallel environments & 
\parbox{2in}{ \centering \quad\quad \\ 64 (Humanoid) \\ 2 (Other tasks) 
	\vspace{.2\baselineskip}  }  \\
\hline
timesteps per epoch & 1024 \\
\hline
initial logstd of policy &
\parbox{2in}{ \centering \quad \\ -1.34 (HalfCheetah,Humanoid) \\ 0 (Other tasks) 
	\vspace{.2\baselineskip} } \\
\hline
policy & Gaussian \\
\hline
$\lambda$ (GAE) & 0.95 \\
\bottomrule
\end{tabular}
\caption{
Hyperparameters for the proposed methods on Mujoco tasks.
}\label{tab_hyperparameters_mujoco}
\end{table}

\begin{table}[h!]
\centering
\setlength{\tabcolsep}{0pt}
\renewcommand{\arraystretch}{1.4}
\begin{tabular}{+c-c}
\toprule
\rowstyle{\bfseries}
	Hyperparameter  & Value
\\ \midrule
%

	coefficient of PPO & 
	\parbox{.5\linewidth}{
		\vspace{.2\baselineskip}	
		\centering
		$\epsilon=\text{LinearAnneal}(0.1,0)$ \\
		\vspace{.2\baselineskip}	
	}

\\ \hline
	coefficient of \pmethodfallback/ & 
	\parbox{.5\linewidth}{
		\vspace{.2\baselineskip}	
		\centering
		$\epsilon=\text{LinearAnneal}(0.1,0)$ \\
		$\alpha=0.01$
		\vspace{.2\baselineskip}	
	}

\\ \hline
	coefficient of \pmethodkl/ & 
	\parbox{.5\linewidth}{
		\vspace{.2\baselineskip}	
		\centering
		$\delta=\text{LinearAnneal}(0.001,0)$ \\
	}

\\ \hline
	coefficient of \pmethodhybrid/ & 
	\parbox{.5\linewidth}{
		\vspace{.2\baselineskip}	
		\centering
		$\delta=\text{LinearAnneal}(0.0008,0)$ \\
		$\alpha=20$
		\vspace{.2\baselineskip}
	}	
\\ \hline
	learning rate & $2.5\times 10^{-4}$ \\
\hline
	number of parallel environments & 
	\parbox{.5\linewidth}{ \centering 8  } 
\\ \hline
	timesteps per epoch & 128 
\\ \hline
	policy & CNN 
\\ \hline
	$\lambda$ (GAE) & 0.95 
\\	\hline	
	coefficient of entropy & 0.01
\\
\bottomrule
\end{tabular}
\caption{
Hyperparameters for the proposed methods on Atari tasks.
}\label{tab_hyperparameters_atari}
\end{table}

\vskip 0.2in
\bibliography{TrulyPPO}

\begin{thebibliography}{27}
\providecommand{\natexlab}[1]{#1}
\providecommand{\url}[1]{\texttt{#1}}
\expandafter\ifx\csname urlstyle\endcsname\relax
  \providecommand{\doi}[1]{doi: #1}\else
  \providecommand{\doi}{doi: \begingroup \urlstyle{rm}\Url}\fi

\bibitem[Bellemare et~al.(2013)Bellemare, Naddaf, Veness, and
  Bowling]{bellemare2013arcade}
Marc~G Bellemare, Yavar Naddaf, Joel Veness, and Michael Bowling.
\newblock The arcade learning environment: An evaluation platform for general
  agents.
\newblock \emph{Journal of Artificial Intelligence Research}, 47:\penalty0
  253--279, 2013.

\bibitem[Brockman et~al.(2016)Brockman, Cheung, Pettersson, Schneider,
  Schulman, Tang, and Zaremba]{Brockman2016OpenAI}
Greg Brockman, Vicki Cheung, Ludwig Pettersson, Jonas Schneider, John Schulman,
  Jie Tang, and Wojciech Zaremba.
\newblock Openai gym, 2016.
\newblock cite arxiv:1606.01540.

\bibitem[Dhariwal et~al.(2017)Dhariwal, Hesse, Klimov, Nichol, Plappert,
  Radford, Schulman, Sidor, and Wu]{baselines}
Prafulla Dhariwal, Christopher Hesse, Oleg Klimov, Alex Nichol, Matthias
  Plappert, Alec Radford, John Schulman, Szymon Sidor, and Yuhuai Wu.
\newblock Openai baselines.
\newblock \url{https://github.com/openai/baselines}, 2017.

\bibitem[Fujimoto et~al.(2018)Fujimoto, van Hoof, and Meger]{fujimoto18td3}
Scott Fujimoto, Herke van Hoof, and David Meger.
\newblock Addressing function approximation error in actor-critic methods.
\newblock In Jennifer Dy and Andreas Krause, editors, \emph{International
  Conference on Machine Learning (ICML)}, volume~80 of \emph{Machine Learning
  Research}, pages 1587--1596, Stockholmsmässan, Stockholm Sweden, 10--15 Jul
  2018. PMLR.
\newblock URL \url{http://proceedings.mlr.press/v80/fujimoto18a.html}.

\bibitem[Haarnoja et~al.(2018)Haarnoja, Zhou, Abbeel, and
  Levine]{haarnoja2018soft}
Tuomas Haarnoja, Aurick Zhou, Pieter Abbeel, and Sergey Levine.
\newblock Soft actor-critic: Off-policy maximum entropy deep reinforcement
  learning with a stochastic actor.
\newblock In \emph{International Conference on Machine Learning (ICML)}, pages
  1856--1865, 2018.

\bibitem[Henderson et~al.(2018)Henderson, Islam, Bachman, Pineau, Precup, and
  Meger]{henderson2018deep}
Peter Henderson, Riashat Islam, Philip Bachman, Joelle Pineau, Doina Precup,
  and David Meger.
\newblock Deep reinforcement learning that matters.
\newblock In \emph{National Conference on Artificial Intelligence (AAAI)},
  2018.

\bibitem[Hu et~al.(2019)Hu, Wang, Liu, and Liu]{hu2019reinforcement}
Yazhou Hu, Wenxue Wang, Hao Liu, and Lianqing Liu.
\newblock Reinforcement learning tracking control for robotic manipulator with
  kernel-based dynamic model.
\newblock \emph{IEEE transactions on neural networks and learning systems},
  2019.

\bibitem[Ilyas et~al.(2018)Ilyas, Engstrom, Santurkar, Tsipras, Janoos,
  Rudolph, and Madry]{ilyas2018deep}
Andrew Ilyas, Logan Engstrom, Shibani Santurkar, Dimitris Tsipras, Firdaus
  Janoos, Larry Rudolph, and Aleksander Madry.
\newblock Are deep policy gradient algorithms truly policy gradient algorithms?
\newblock \emph{arXiv preprint arXiv:1811.02553}, 2018.

\bibitem[Islam et~al.(2017)Islam, Henderson, Gomrokchi, and
  Precup]{islam2017reproducibility}
Riashat Islam, Peter Henderson, Maziar Gomrokchi, and Doina Precup.
\newblock Reproducibility of benchmarked deep reinforcement learning tasks for
  continuous control.
\newblock \emph{arXiv preprint arXiv:1708.04133}, 2017.

\bibitem[Kakade and Langford(2002)]{kakade2002approximately}
Sham Kakade and John Langford.
\newblock Approximately optimal approximate reinforcement learning.
\newblock In \emph{International Conference on Machine Learning (ICML)},
  volume~2, pages 267--274, 2002.

\bibitem[{Kakade}(2001)]{kakade2001natural}
Sham~M {Kakade}.
\newblock A natural policy gradient.
\newblock In \emph{Advances in Neural Information Processing Systems
  (NeurIPS)}, volume~14, pages 1531--1538, 2001.

\bibitem[{Levine} et~al.(2016){Levine}, {Finn}, {Darrell}, and
  {Abbeel}]{levine2016end}
Sergey {Levine}, Chelsea {Finn}, Trevor {Darrell}, and Pieter {Abbeel}.
\newblock End-to-end training of deep visuomotor policies.
\newblock \emph{Journal of Machine Learning Research}, 17\penalty0
  (1):\penalty0 1334--1373, 2016.

\bibitem[Liu et~al.(2018)Liu, Li, Tong, and Chen]{liu2018adaptive}
Yan-Jun Liu, Shu Li, Shaocheng Tong, and CL~Philip Chen.
\newblock Adaptive reinforcement learning control based on neural approximation
  for nonlinear discrete-time systems with unknown nonaffine dead-zone input.
\newblock \emph{IEEE transactions on neural networks and learning systems},
  30\penalty0 (1):\penalty0 295--305, 2018.

\bibitem[Mnih et~al.(2015)Mnih, Kavukcuoglu, Silver, Rusu, Veness, Bellemare,
  Graves, Riedmiller, Fidjeland, Ostrovski, et~al.]{mnih2015human}
Volodymyr Mnih, Koray Kavukcuoglu, David Silver, Andrei~A Rusu, Joel Veness,
  Marc~G Bellemare, Alex Graves, Martin Riedmiller, Andreas~K Fidjeland, Georg
  Ostrovski, et~al.
\newblock Human-level control through deep reinforcement learning.
\newblock \emph{Nature}, 518\penalty0 (7540):\penalty0 529, 2015.

\bibitem[Mnih et~al.(2016)Mnih, Badia, Mirza, Graves, Lillicrap, Harley,
  Silver, and Kavukcuoglu]{mnih2016asynchronous}
Volodymyr Mnih, Adria~Puigdomenech Badia, Mehdi Mirza, Alex Graves, Timothy
  Lillicrap, Tim Harley, David Silver, and Koray Kavukcuoglu.
\newblock Asynchronous methods for deep reinforcement learning.
\newblock In \emph{International Conference on Machine Learning (ICML)}, pages
  1928--1937, 2016.

\bibitem[Peters and Schaal(2008)]{peters2008reinforcement}
Jan Peters and Stefan Schaal.
\newblock Reinforcement learning of motor skills with policy gradients.
\newblock \emph{Neural networks}, 21\penalty0 (4):\penalty0 682--697, 2008.

\bibitem[Rai et~al.(2019)Rai, Antonova, Meier, and Atkeson]{Akshara2019Robot}
Akshara Rai, Rika Antonova, Franziska Meier, and Christopher~G. Atkeson.
\newblock Using simulation to improve sample-efficiency of bayesian
  optimization for bipedal robots.
\newblock \emph{Journal of Machine Learning Research}, 20\penalty0
  (49):\penalty0 1--24, 2019.

\bibitem[Schulman et~al.(2015)Schulman, Levine, Abbeel, Jordan, and
  Moritz]{schulman2015trust}
John Schulman, Sergey Levine, Pieter Abbeel, Michael Jordan, and Philipp
  Moritz.
\newblock Trust region policy optimization.
\newblock In \emph{International Conference on Machine Learning (ICML)}, pages
  1889--1897, 2015.

\bibitem[{Schulman} et~al.(2016){Schulman}, {Moritz}, {Levine}, {Jordan}, and
  {Abbeel}]{schulman2016high}
John {Schulman}, Philipp {Moritz}, Sergey {Levine}, Michael~I. {Jordan}, and
  Pieter {Abbeel}.
\newblock High-dimensional continuous control using generalized advantage
  estimation.
\newblock \emph{{International Conference on Learning Representations}}, 2016.

\bibitem[Schulman et~al.(2017)Schulman, Wolski, Dhariwal, Radford, and
  Klimov]{schulman2017proximal}
John Schulman, Filip Wolski, Prafulla Dhariwal, Alec Radford, and Oleg Klimov.
\newblock Proximal policy optimization algorithms.
\newblock \emph{arXiv preprint arXiv:1707.06347}, 2017.

\bibitem[Silver et~al.(2017)Silver, Hubert, Schrittwieser, Antonoglou, Lai,
  Guez, Lanctot, Sifre, Kumaran, Graepel, et~al.]{silver2017mastering}
David Silver, Thomas Hubert, Julian Schrittwieser, Ioannis Antonoglou, Matthew
  Lai, Arthur Guez, Marc Lanctot, Laurent Sifre, Dharshan Kumaran, Thore
  Graepel, et~al.
\newblock Mastering chess and shogi by self-play with a general reinforcement
  learning algorithm.
\newblock \emph{arXiv preprint arXiv:1712.01815}, 2017.

\bibitem[{Sutton} et~al.(1999){Sutton}, {McAllester}, {Singh}, and
  {Mansour}]{sutton1999policy}
Richard~S. {Sutton}, David~A. {McAllester}, Satinder~P. {Singh}, and Yishay
  {Mansour}.
\newblock Policy gradient methods for reinforcement learning with function
  approximation.
\newblock In \emph{Advances in Neural Information Processing Systems
  (NeurIPS)}, pages 1057--1063, 1999.

\bibitem[Todorov et~al.(2012)Todorov, Erez, and Tassa]{Todorov2012MuJoCo}
Emanuel Todorov, Tom Erez, and Yuval Tassa.
\newblock Mujoco: A physics engine for model-based control.
\newblock In \emph{IEEE/RSJ International Conference on Intelligent Robots and
  Systems}, pages 5026--5033. IEEE, 2012.

\bibitem[Wang et~al.(2019{\natexlab{a}})Wang, He, and Tan]{wang2019truly}
Yuhui Wang, Hao He, and Xiaoyang Tan.
\newblock Truly proximal policy optimization.
\newblock In \emph{Uncertainty in Artificial Intelligence (UAI)},
  2019{\natexlab{a}}.

\bibitem[Wang et~al.(2019{\natexlab{b}})Wang, He, Tan, and Gan]{wang2019trust}
Yuhui Wang, Hao He, Xiaoyang Tan, and Yaozhong Gan.
\newblock Trust region-guided proximal policy optimization.
\newblock In \emph{Advances in Neural Information Processing Systems
  (NeurIPS)}, 2019{\natexlab{b}}.

\bibitem[Williams(1992)]{williams1992simple}
Ronald~J Williams.
\newblock Simple statistical gradient-following algorithms for connectionist
  reinforcement learning.
\newblock \emph{Machine learning}, 8\penalty0 (3-4):\penalty0 229--256, 1992.

\bibitem[Wu et~al.(2017)Wu, Mansimov, Grosse, Liao, and Ba]{wu2017scalable}
Yuhuai Wu, Elman Mansimov, Roger~B Grosse, Shun Liao, and Jimmy Ba.
\newblock Scalable trust-region method for deep reinforcement learning using
  kronecker-factored approximation.
\newblock In \emph{Advances in Neural Information Processing Systems
  (NeurIPS)}, pages 5279--5288, 2017.

\end{thebibliography}

\end{document}